\pdfoutput=1
\documentclass{article}
\usepackage{style/iclr2027_conference}

\usepackage{times}

\iclrfinalcopy

\usepackage[T1]{fontenc}
\usepackage{amsmath, amssymb, amsthm}
\usepackage{graphicx}
\usepackage{subcaption}
\usepackage{float}
\usepackage{booktabs}
\usepackage{hyperref}
\usepackage{url}
\usepackage{cleveref}
\crefname{appsec}{Appendix}{Appendices}
\Crefname{appsec}{Appendix}{Appendices}

\usepackage{listings}
\let\ttdefaultsaved\ttdefault
\usepackage[scaled=0.96,var0,varl,varqu,noupquote]{zi4}
\let\ttdefaultcode\ttdefault
\let\ttdefault\ttdefaultsaved
\newcommand{\codefont}{\let\ttdefault\ttdefaultcode\ttfamily}
\definecolor{codekeyword}{HTML}{1F4E9C}
\definecolor{codecomment}{HTML}{4F7355}
\definecolor{codestring}{HTML}{8F3A32}
\lstdefinestyle{teasercode}{
  language=Python,
  morekeywords={None,True,False,as,with,lambda,yield,assert,nonlocal},
  basicstyle=\codefont\small,
  identifierstyle=\color{black},
  keywordstyle=\color{codekeyword},
  commentstyle=\color{codecomment},
  stringstyle=\color{codestring},
  columns=fixed,
  basewidth=0.5em,
  keepspaces=true,
  showstringspaces=false,
  numbers=none,
  frame=none,
  tabsize=4,
  aboveskip=2pt, belowskip=2pt,
}

\newtheorem{theorem}{Theorem}
\newtheorem{lemma}{Lemma}
\newtheorem{proposition}{Proposition}

\theoremstyle{definition}
\newtheorem{definition}{Definition}
\theoremstyle{remark}
\newtheorem{remark}{Remark}
\theoremstyle{plain}

\newcommand{\sgn}{\operatorname{sgn}}
\newcommand{\ip}[2]{\langle #1, #2\rangle}
\newcommand{\norm}[1]{\lVert #1 \rVert}
\newcommand{\ind}[1]{\mathbf{1}\{#1\}}
\renewcommand{\O}{\mathcal{O}}
\newcommand{\R}{\mathbb{R}}
\newcommand{\N}{\mathbb{N}}
\newcommand{\float}{\mathbb{F}}

\newcommand{\pc}{\mathtt{pc}}
\newcommand{\mem}{\mathtt{mem}}
\newcommand{\bits}{\operatorname{bits}}
\newcommand{\enc}{\operatorname{enc}}
\newcommand{\poly}{\operatorname{poly}}
\newcommand{\voc}{\mathcal{V}}
\newcommand{\dff}{d_{\mathrm{ff}}}
\DeclareMathOperator{\sbop}{sb}
\DeclareMathOperator{\rd}{rd}
\DeclareMathOperator{\relu}{ReLU}
\DeclareMathOperator{\lemaop}{LEMA}
\DeclareMathOperator*{\argmax}{arg\,max}

\title{Latest Exact Match Attention}
\author{Moritz Br\"osamle \\
Department of Mathematics, University of T\"ubingen, Germany \\
\texttt{moritzbroesamle@gmail.com}
}
\date{}

\begin{document}
\maketitle

\begin{abstract}
We introduce \emph{latest exact match attention} (LEMA), an attention variant for transformers where queries and keys are binarized and each query attends only to the latest exactly matching key. We prove that LEMA transformers with chain of thought can simulate word-RAMs, as was recently shown for the less restrictive rightmost hard attention. In contrast to prior hard attention variants, the restriction to exact matches enables an efficient converse direction: word-RAMs can simulate LEMA transformers at a cost per token independent of the context length. Together, these results yield a close correspondence between the two computational models in terms of both compute and memory. Beyond the theory, we propose a training method for LEMA transformers that handles their non-differentiable operations with a straight-through estimator for the binarization and a soft attention surrogate annealed towards LEMA. On a synthetic associative recall task, LEMA models trained this way use their growing state to store and recall a large number of associations, outperforming gated DeltaNet (GDN) with its fixed state size. As a first scaling test, we train LEMA language models with up to $834$ million parameters. They match softmax transformers of around half their size in loss and, on repeated rare phrases and a needle-retrieval task, remain behind softmax transformers but recall across longer distances than GDN models of comparable size. Finally, we implement dictionary-based inference for LEMA transformers and show constant generation speed comparable to GDN despite their growing state, with the dictionaries residing in main memory rather than VRAM. Code is available at \url{https://github.com/moritzbroe/latest_exact_match_attention}.

\end{abstract}

\section{Introduction}
During autoregressive generation, softmax transformers typically face two costs that grow with context length: reading more keys and values slows generation, and storing them consumes limited GPU memory (VRAM). State space models and linear attention variants---collectively referred to as fixed-state models throughout this work---instead compress the context into a state of fixed size, making both memory and compute per token independent of context length. However, this fixed memory capacity limits recall when the amount of information to be retained grows~\citep{arora2023zoology,arora2024based,waleffe2024mamba,afendulev2026hybrid}. We seek to combine a state that can grow but resides in main memory with the constant generation speed of fixed-state models by using direct lookup operations instead of a scan over the state.

Abstractly, softmax attention can be viewed as a differentiable dictionary lookup. A query is compared with all previous keys and returns a mixture of all values weighted by similarity to their keys. This makes the lookup differentiable, but also involves every stored key and value in the computation. We propose to replace it with the following discrete lookup operation during inference:

\begin{definition}[Latest exact match attention]\label{def:lema}
    For binary queries and keys $q_i,k_i\in\{-1,1\}^{d_h}$ and values $v_i\in\R^{d_h}$, latest exact match attention (LEMA) returns
    \[
        \lemaop\bigl((q_i,k_i,v_i)_{i=1}^n\bigr)=(o_i)_{i=1}^n,
        \qquad
        o_i=
        \begin{cases}
            v_{\ell_i}, & \ell_i=\max\{j<i\mid k_j=q_i\}\text{ exists},\\
            0, & \text{otherwise}.
        \end{cases}
    \]
    A LEMA head applies this operation to $q_i=\sgn(W_Qx_i)$, $k_i=\sgn(W_Kx_i)$ and $v_i=W_Vx_i$ for hidden states $x_i\in\R^d$, with $\sgn$ applied coordinatewise and $\sgn(0)=1$. A LEMA transformer replaces softmax attention heads with LEMA heads.
\end{definition}

\begin{figure}[t]
\centering
\begin{minipage}[t]{0.47\textwidth}
\begin{lstlisting}[style=teasercode]
def softmax_step(q, k, v, C):
    # C: kv-cache, two tensors in VRAM
    C.K.append(k)         # grows with t
    C.V.append(v)
    a = softmax(C.K @ q)  # scores t keys
    return a @ C.V        # mixes t values
\end{lstlisting}
\end{minipage}\hfill
\begin{minipage}[t]{0.47\textwidth}
\begin{lstlisting}[style=teasercode]
def lema_step(q, k, v, C):
    # C: kv-cache, dict in main memory
    q, k = binarize(q), binarize(k)
    o = C.get(q, 0)       # one lookup
    C[k] = v              # one insert
    return o
\end{lstlisting}
\end{minipage}
\vspace{-4pt}
\caption{Simplified code for the kv-cached generation step of a softmax head and a LEMA head.}
\label{fig:teaser}
\vspace{-8pt}
\end{figure}

In words, a LEMA transformer binarizes keys and queries and each query attends only to the latest exactly matching key, retrieving zeros if no key matches. During autoregressive inference of a LEMA transformer, the kv-cache of each attention head can be implemented as a dictionary as illustrated in \Cref{fig:teaser}. Processing one token consists of one lookup and one insert to this dictionary, and updating an existing key overwrites its old value, allowing for a content-dependent state size. The dictionary for a LEMA head with head dimension $d_h$ contains at most $2^{d_h}$ entries, as only that many distinct keys exist, so the state size is bounded in theory but practically unbounded for large $d_h$. Prior works have used attention to the nearest or latest token satisfying some condition~\citep{csordas2022ndr, yang2024masked, friedman2023learning} or even to the tokens with exactly matching keys~\citep{liu2025tale,yang2025pencil} as a formal device or in small trained models, but the form of \Cref{def:lema} with binarized queries and keys has, to our knowledge, not been proposed.

In \Cref{sec:ramsimulation}, we establish a correspondence between LEMA transformers with chain of thought and word-RAMs, an abstraction of modern computers, which, to our knowledge, has not been established for other attention variants. Recently,
\citet{li2026wordram} showed that $t$ steps of a word-RAM with word size $w$ can be simulated by a transformer with rightmost hard attention using $\O(t\poly(w))$ chain of thought steps. In \Cref{thm:ramsim} we show that an analogous statement holds for the more restrictive LEMA and bound the number of distinct keys used by the LEMA heads in terms of the word-RAM's space usage. Using LEMA instead of rightmost hard attention enables an efficient converse of this statement: \Cref{thm:lemasim} shows that a word-RAM can perform autoregressive generation of an $N$-parameter LEMA transformer using $\O(N)$ steps per processed token with space usage corresponding to the number of distinct keys. The kv-caches are implemented as trie-based dictionaries on the word-RAM with lookup and insert operations taking time independent of the number of kv-entries.
Apart from input/output costs, these results yield a round trip with space overhead polynomial in the word size and time overhead polynomial in the word size and program length.

The hard attention variants of prior expressivity results serve as theoretical abstractions of softmax attention~\citep{hahn2020theoretical, merrill2024cot}, whose faithfulness is debated~\citep{merrill2022saturated, velickovic2025softmax}, and those works do not address training models with the analyzed attention rules. In contrast, we propose in \Cref{sec:training} a way to train the very attention rule we analyze despite its non-differentiable operations. Gradients through the binarization of queries and keys come from a straight-through estimator~\citep{bengio2013estimating, courbariaux2016binarized}, and for the latest-match operation we use a surrogate based on stick-breaking attention~\citep{tan2025stickbreaking, raffel2017monotonic} during training, slowly annealing it towards LEMA. Training with this surrogate still requires computation quadratic in sequence length.

Using these techniques, we train LEMA transformers on a synthetic associative recall task where a growing number of associations between tokens is presented and then queried. We find that tiny LEMA transformers are able to store a large number of associations, just as softmax transformers, while the linear attention variant gated DeltaNet (GDN)~\citep{yang2025gateddelta} fails once the number of associations becomes too large. Next, we test whether the method scales by training LEMA transformers with up to $834$ million parameters on FineWeb-Edu~\citep{penedo2024fineweb}. At the larger sizes, they match softmax transformers with slightly more than half their parameters in language modeling loss. On two proxies for long-range recall, namely the loss on repeated rare phrases and the single-needle task of the RULER benchmark~\citep{hsieh2024ruler} with repetitive filler, they remain clearly behind softmax transformers but outperform GDN at long distances. On both the synthetic recall task and in language modeling, we observe shortcomings of our training method and hence consider it merely a first attempt at training LEMA transformers.

In \Cref{sec:inference}, we present and benchmark an implementation of LEMA transformers using hash tables as dictionaries for the kv-caches. This keeps the generation speed constant, close to that of GDN, until the hash table nears its capacity. Furthermore, these kv-caches are stored in main memory rather than VRAM, so that the VRAM holds only the model's parameters and activations.

\section{Related work}
In most sequence models, every generated token is computed from the whole context-dependent state, so that the compute per token grows with the state, and models differ in how the state grows with the context. Many models use a fixed state size, most prominently recurrent networks~\citep{elman1990finding, hochreiter1997long}, state space models~\citep{gu2022s4, gu2024mamba, dao2024mamba2}, linear attention variants~\citep{katharopoulos2020transformers, schlag2021linear, liu2022ecoformer, yang2024deltanet, yang2025gateddelta} and recent test-time training methods~\citep{sun2024ttt, behrouz2025titans}. Other methods bound the state of transformers: \citet{lingle2024transformervq} quantizes keys to a small fixed codebook and accumulates the values per code, while \citet{zhang2023h2o} and \citet{xiao2024streamingllm} evict entries from the kv-cache beyond a budget, and \citet{cui2026hola} adds a small fixed cache of exact key-value pairs to GDN. Softmax transformers~\citep{vaswani2017attention} instead grow their state linearly in the context length, while hybrid architectures~\citep{poli2024mad}, grouped-query and latent attention~\citep{shazeer2019fast, ainslie2023gqa, deepseek2024v2} and cache quantization~\citep{liu2024kivi, hooper2024kvquant} reduce the slope of this growth, and binarized queries and keys~\citep{horton2025hamming, xiao2026binaryattention} reduce attention computation. Log-linear attention~\citep{guo2025loglinear} grows the state logarithmically. \citet{pal2026distinct} cache keys far from all stored keys, thus sharing LEMA's content-dependent state size, but read the full cache with softmax attention. Some methods also merge cache entries at a learned, content-dependent rate~\citep{nawrot2024dmc} or let transformers summarize or erase their chain of thought~\citep{yan2026inftythink, yang2025pencil, aghajohari2025markovian}, which allows simulating Turing machines space-efficiently~\citep{yang2025pencil, broesamle2026expressive}.

Other methods use only part of the state to compute the next token. Several content-based sparse-attention methods~\citep{tang2024quest, desai2025hashattention, gong2025hata, yuan2025nsa, lu2025moba, deepseek2025v32} use a cheap scan over keys or block summaries to select a small subset for attention, so their per-token compute still grows linearly with context length. Reformer~\citep{kitaev2020reformer} and Routing Transformers~\citep{roy2021routing} select keys by locality-sensitive hashing or learned clustering, then mix their values with softmax attention. Other methods use indexed memory access. Neural Random-Access Machines~\citep{kurach2016neural} train neural controllers over differentiable arithmetic and memory-access primitives, with constant-time memory access after discretization. Neural Episodic Control~\citep{pritzel2017neural} grows a dictionary of key-value pairs and updates the value of an existing key in place, Sparse Access Memory~\citep{rae2016sam} reads and writes a few slots of a fixed-size memory per step, Memorizing Transformers~\citep{wu2022memorizing} attend to the nearest neighbors of the query among past keys, and RetrievalAttention~\citep{liu2024retrievalattention} and MagicPIG~\citep{chen2025magicpig} keep the kv-cache of trained softmax transformers in main memory and read a small fraction of the keys per query, found by a graph index or sampled through locality-sensitive hashing. FwPKM~\citep{zhao2026fwpkm} and Sparse Delta Memory~\citep{cabannes2026sdm} maintain a fixed-size state of $m$ value vectors per head in VRAM, selecting entries using two vectors of $\sqrt{m}$ scores. A LEMA head instead addresses up to $2^{d_h}$ entries with a $d_h$-bit key, storing only keys that occur.

Finally, Engram~\citep{cheng2026engram} and memory layers~\citep{lample2019pkm, berges2024memorylayers} access a large table of parameters through a hashed $n$-gram or a product key, at a cost independent of or sublinear in its size. As this table contains parameters rather than context-dependent state, these methods are conceptually closer to mixture-of-experts layers~\citep{shazeer2017moe} than to LEMA.

\section{Computational equivalence with word-RAMs}\label{sec:ramsimulation}
A \emph{word-RAM}~\citep{fredman1993fusion, hagerup1998wordram} $M=(P,r,w)$ consists of a program $P$, a register count $r$ and a word size $w$. Its $r$ registers and $2^w$ memory cells hold $w$-bit \emph{words}, and $P$ is a sequence of arithmetic, branching and memory access instructions reading from and writing to these cells. Time $t_M(x)$ is the number of executed instructions on an input $x$ before halting, while space $s_M(x)$ counts the registers and memory cells used by that computation. Precise definitions are given in \Cref{app:ram-model}.

For this section, a \emph{LEMA transformer} is a transformer decoder with LEMA heads as defined by \Cref{def:lema} using ReLU MLPs and no normalization or positional encodings. All operations in its forward pass use standard IEEE-style floating-point arithmetic, and we write $p$ for the total number of bits of the format, its \emph{precision}. Generation is greedy. We say that $T$ generates $y$ from $x$ if its autoregressive generation from $x$ ends in $\texttt{<out>}\;y\;\texttt{<eos>}$. The tokens before $\texttt{<out>}$ are the chain of thought. $t_T(x)$ then denotes the total number of generated tokens and $s_T(x)$ the sum over all heads of the number of distinct keys used by each head. See \Cref{app:lema-model} for the formal model. 

\subsection{LEMA transformers simulate word-RAMs}
The following result parallels Theorem~1 of \citet{li2026wordram}, which establishes word-RAM simulation capability for the less restrictive rightmost hard attention, where queries and keys are not binarized and each query selects the key with the largest inner product, using rightmost tie-breaking. Notably, many parts of their construction and of constructions in other transformer expressivity works~\citep{merrill2024cot, li2024cot, liu2025tale, yang2025pencil, broesamle2026expressive} attend through exact matches of encoded addresses or indices, which partly motivated this work.
The encoding $\enc$ in the statement writes out words with corresponding memory addresses bit by bit (\Cref{def:enc}). Every $\O$-expression hides only a universal constant.
\begin{theorem}\label{thm:ramsim}
    Let $M=(P,r,w)$ be a word-RAM. Then there exists a LEMA transformer $T$ with an $\O(1)$-size vocabulary, $\O(1)$ heads per layer, depth $\O(w)$, model dimension $d=\O(w)$, head dimension $d_h=\O(w)$, MLP width $\dff=\O(w+|P|)$ and precision $p=\O(\log w)$ such that the following holds. If $M$ on an input word sequence $x$ halts with output sequence $y$, then $T$ generates $\enc(y)$ from $\enc(x)$ and
    \[
        t_T(\enc(x)) = \O\bigl((t_M(x)+|y|)w\bigr)
        \qquad\text{and}\qquad
        s_T(\enc(x)) = \O(s_M(x)+w)\;.
    \]
\end{theorem}
The proof is given in \Cref{app:ramsim-proof}. During generation, the constructed transformer produces tokens encoding the changes of each word-RAM step to the word-RAM's registers, memory cells and program counter.

\subsection{Word-RAMs simulate LEMA transformers}
The next result shows that a word-RAM can evaluate a LEMA transformer with per-token compute independent of the context length and memory growing with the number of distinct keys. Again, every $\O$-expression hides only a universal constant.

\begin{theorem}\label{thm:lemasim}
    There is a universal constant $C>0$ such that, for every LEMA transformer $T$ with $N$ parameters, model dimension $d$, head dimension $d_h$ and precision $p$, there exist a program $P$, a register count $r$ and a word size threshold $w_0$ satisfying
    \[
        |P|=\O(N),
        \qquad r=\O(1),
        \qquad w_0=\O(p+\log N)\;
    \]
    such that the following holds for every $w\ge w_0$. If $T$ on input $x$ generates output $y$ and
    \[
        \max\{|x|,|y|\}+C\bigl(d+s_T(x)d_h\bigr)<2^w\;,
    \]
    then the word-RAM $M=(P,r,w)$ outputs the token indices of $y$ from the token indices of $x$ and
    \[
        t_M(x) = \O\bigl((|x|+t_T(x))N\bigr) 
        \qquad\text{and}\qquad
        s_M(x) = \O\bigl(|x|+|y|+d+s_T(x)d_h\bigr)\;.
    \]
\end{theorem}
The proof in \Cref{app:lemasim-proof} constructs a word-RAM that performs standard autoregressive inference, with the kv-cache of each head stored as a binary trie mapping keys to their latest values. A lookup or update costs $\O(d_h)$ time, and the dictionary operations are dominated by the $\O(N)$ operations per token needed for matrix-vector multiplications. A trie with $s$ distinct keys uses $\O(s d_h)$ words, which gives the stated space bound after accounting for reusable buffers and the input and output. The capacity condition ensures that these fit in memory, while the threshold $w_0$ provides enough bits for scalar arithmetic and instruction indices.

To see the correspondence between LEMA transformers and word-RAMs, the two results can be combined into a round trip: for a word-RAM $M=(P,r,w)$, applying \Cref{thm:ramsim} and then \Cref{thm:lemasim} yields a single word-RAM $M'$ with word size $w'=\O(w)$ that maps $\enc(x)$ to $\enc(y)$ whenever $M$ halts on $x$ with output $y$, identifying tokens with their vocabulary indices. Its time and space satisfy $t_{M'}(\enc(x))=\O((t_M(x)+|x|+|y|)\,|P|\poly(w))$ and $s_{M'}(\enc(x))=\O(s_M(x)\poly(w))$. See \Cref{app:roundtrip} for details.

The head dimension $d_h$ loosely corresponds to the word size $w$: a word-RAM can access $2^w$ memory cells through $w$-bit addresses, while a LEMA head can access up to $2^{d_h}$ key-value pairs through $d_h$-bit keys. The correspondence appears to be specific to LEMA transformers and word-RAMs. Using rightmost hard attention instead of LEMA would replace the exact-match lookups with maximum inner product searches for which no comparably efficient exact method is known. Likewise, Turing machines, the model of many prior chain of thought expressivity results~\citep{perez2021attention, merrill2024cot}, lack the random memory access that makes the dictionary operations efficient.

\section{Training LEMA transformers}\label{sec:training}
\subsection{Training method}
Here we introduce the training method we will use to train LEMA transformers. The two non-differentiable operations, namely the binarization of queries and keys and the latest exact match operation, are treated separately.

\paragraph{Binarizing queries and keys.}
The binarization of queries and keys in LEMA transformers is analogous to how activations are binarized in binarized neural networks. We adopt the most common training technique used in that literature, the straight-through estimator, which uses the gradient of a similarly shaped differentiable function in the backward pass of the sign function~\citep{courbariaux2016binarized, yin2019understanding}. In particular, for the forward pass we obtain queries and keys of a LEMA head for hidden state $x \in \R^d$ as
$$q = \sgn(\beta \operatorname{RMSNorm}(W_Q x))$$
where in the backward pass the sign function is treated as having derivative $\tanh'$. Neither the normalization $\operatorname{RMSNorm}(y) = y \big/ \sqrt{\textstyle\sum_i y_i^2/d_h}$~\citep{zhang2019rmsnorm} nor the multiplication by the hyperparameter $\beta > 0$ changes the forward pass, as $\sgn$ is invariant under positive rescaling, but helps keep activations in a region with gradient magnitudes controlled by $\beta$ instead of drifting.

\paragraph{Annealing towards LEMA.}
Just as the sign function is approximated by $\tanh$ as a differentiable surrogate, we approximate the latest exact match operation with a smooth surrogate. As this surrogate we choose stick-breaking attention~\citep{tan2025stickbreaking}, where a high attention score to one position suppresses the attention weights on earlier positions, and extend it with a threshold parameter $c$:
$$\sbop_{\alpha, c}((q_i,k_i,v_i)_{i=1}^n) = (o_i)_{i=1}^n$$
where the output at position $i$ is defined as
\begin{equation}
o_i = \sum_{j < i} w_{ij} v_j,\; w_{ij} = \sigma(\alpha (\langle q_i, k_j\rangle -c)) \prod_{l=j+1}^{i-1} (1-\sigma(\alpha(\langle q_i, k_l\rangle - c)))
\end{equation}
with $\sigma$ the sigmoid function. With queries and keys binarized to $\{-1,1\}^{d_h}$, the score $\ip{q_i}{k_j}$ lies in $\{d_h, d_h-2, d_h-4, \dots\}$, so setting $c = d_h - 1$ makes the sigmoid argument $\alpha$ for an exact match and at most $-\alpha$ otherwise. Increasing $\alpha$ hence sharpens the surrogate towards LEMA:
\begin{lemma}\label{lem:surrogate}
    Let $q_i, k_i \in \{-1,1\}^{d_h}$ and $v_i \in \R^{d_h}$ for $i \le n$, and let $\alpha > 0$. Then for every position $i$,
    \[
        \bigl\|\lemaop\bigl((q_j,k_j,v_j)_{j=1}^n\bigr)_i - \sbop_{\alpha, d_h-1}\bigl((q_j,k_j,v_j)_{j=1}^n\bigr)_i\bigr\|
        \le 2n e^{-\alpha} \max_{j<i}\norm{v_j}.
    \]
\end{lemma}
The proof is given in \Cref{app:surrogate}.
With this differentiable approximation of the latest exact match operation, one could try to use the same technique as for the binarization, i.e.\ using hard LEMA in the forward pass with the gradients from $\sbop_{\alpha, d_h-1}$ in the backward pass. Unless $d_h$ is very small, however, at initialization most queries then match no key exactly, leaving attention outputs at zero and degrading learning as discussed in \Cref{app:hardening}. Instead, training starts with the stick-breaking surrogate at $c=0$ and hardens it in two phases. After learning rate warmup, $c$ increases linearly from $0$ to $d_h-1$ while $\alpha$ stays at its initial value $\frac{1}{\sqrt{d_h}}$. Then $\alpha$ increases linearly to $10$ at the end of training, gradually aligning the surrogate with LEMA. The timing of both ramps is given per experiment in \Cref{app:recall-details,app:lm-details}. For the backward pass, however, we cap $\alpha$ at $2$ to retain gradient flow, so forward and backward pass differ again once $\alpha$ exceeds the cap. We use LEMA with head dimension $d_h=64$; smaller head dimensions are discussed in \Cref{app:head-sizes}.

\subsection{Associative recall}\label{sec:recall}
We train LEMA transformers and two baselines used throughout the paper, softmax transformers with rotary positional encodings (RoPE)~\citep{su2024roformer} and GDN, on a synthetic associative recall task. Similar tasks have been used previously to contrast the memory capacity of transformers and fixed-state models~\citep{arora2023zoology, arora2024based, jelassi2024repeat, okpekpe2025revisiting}.

\paragraph{Task and models.}
We use a variation of the multi-query associative recall task introduced in \citet{arora2023zoology}. Each sample is of the form
$(a_1,b_1,a_2,b_2,\dots,a_n,b_n,\mathrm{SEP},a_{\pi(1)},\dots,a_{\pi(n)})$
where $a_1,\dots,a_n$ are sampled without repeats from a vocabulary of size $4096$, $b_1,\dots,b_n$ are sampled uniformly from a disjoint vocabulary of the same size, $\pi$ is a random permutation of $\{1,\dots,n\}$ and $\mathrm{SEP}$ is a separator token. Hence, the total vocabulary size is $2\cdot 4096+1$ and a sequence with $n$ associations has length $3n+1$. After each $a_{\pi(i)}$ following the separator, the model must predict $b_{\pi(i)}$. Consistent with prior work, we train tiny models: each model has $2$ layers, model dimension $d=64$ and a single head of dimension $d_h=64$.

\begin{figure}[t]
\centering
\includegraphics[width=\linewidth]{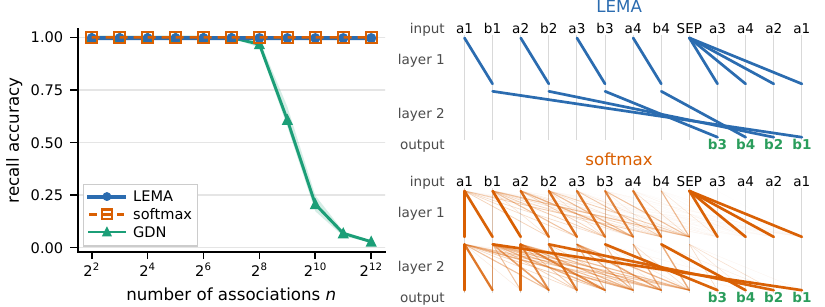}
\caption{\emph{Left:} mean recall accuracy at $n$ associations over three seeds, with shaded minimum-to-maximum ranges. LEMA is trained only at $n=8$, the others on a curriculum. \emph{Right:} attention weights of a LEMA transformer and a softmax transformer on one $n=4$ sample input.}
\label{fig:recall}
\end{figure}

\paragraph{Softmax transformers and GDN.}
Trained directly at larger $n$, softmax transformers do not discover a solution in our experiments, so we train them and GDN on a curriculum of increasing $n$, at each stage using the best checkpoint for evaluation at that $n$ and to initialize the next stage. \Cref{fig:recall} shows the mean accuracy over three seeds, with individual runs in \Cref{app:recall-details}. Solving the task requires holding all $n$ pairs in the state once they are read. Softmax transformers do so with their linearly growing state, while GDN's fixed-size state holds a limited number of associations and degrades beyond it.

\paragraph{LEMA.}
LEMA transformers can grow their state like softmax transformers if they assign distinct key codes to tokens, and the task tests in a controlled setting whether our training method realizes this ability. On this task, we find hardening sensitive to high learning rates and therefore lower the learning rate during hardening (\Cref{app:recall-details}). Training then succeeds and we can avoid combining a curriculum with hardening, as training only at $n=8$ extrapolates almost perfectly to $n=4096$ (\Cref{fig:recall}). The underlying stick-breaking attention already extrapolates to a few hundred pairs when trained at $n=8$. Zeroing small stick-breaking gates at evaluation largely restores extrapolation to $n=4096$, showing that leakage through these gates limits recall (\Cref{app:recall-sb}). LEMA's discrete rule eliminates this leakage.

\paragraph{Mechanistic interpretability.}
As a LEMA head allows each position to attend to at most one prior position, the value routing between positions is explicit. The attention weights in \Cref{fig:recall} (right) reveal an induction-head circuit, a mechanism hypothesized to underlie much of in-context learning~\citep{olsson2022induction}. For each presented pair $a_i b_i$, the first layer's head copies $a_i$ into the representation at $b_i$. When $a_i$ reappears after $\mathrm{SEP}$, the second layer's head can then attend to $b_i$ to retrieve it. The softmax transformer implements the same mechanism but also puts attention mass on other tokens, which makes the value routing harder to trace (\Cref{app:recall-mechanism}). The mechanism also predicts behavior not seen in training: when the same $a_i$ is paired with different following tokens, LEMA transformers predict the most recently paired token, while softmax transformers and GDN spread their predictions over the alternatives (\Cref{tab:repeated-keys}). We leave it to future work to investigate whether LEMA offers interpretability benefits on more complex tasks including natural language modeling.

\subsection{Language modeling}\label{sec:lm}
Next, we train LEMA transformers of $29$ to $834$ million parameters on the FineWeb-Edu dataset of high-quality text documents~\citep{penedo2024fineweb} and compare them to softmax transformers and GDN with the same model dimensions and depths. All models are trained on around $20$ tokens per parameter~\citep{hoffmann2022training} at context length $2048$. Details can be found in \Cref{app:lm-details}, ablations and further analyses in \Cref{app:lm-ablations,app:lm-analysis}. The cross-entropy losses of each model on the validation set are shown in \Cref{fig:lm} (left) with the precise numbers in \Cref{tab:lm-ce}. LEMA transformers stay clearly behind softmax transformers, matching those with $55$ to $57\%$ of their parameters from $77$M on.

\begin{figure}[t]
\centering
\includegraphics[width=\linewidth]{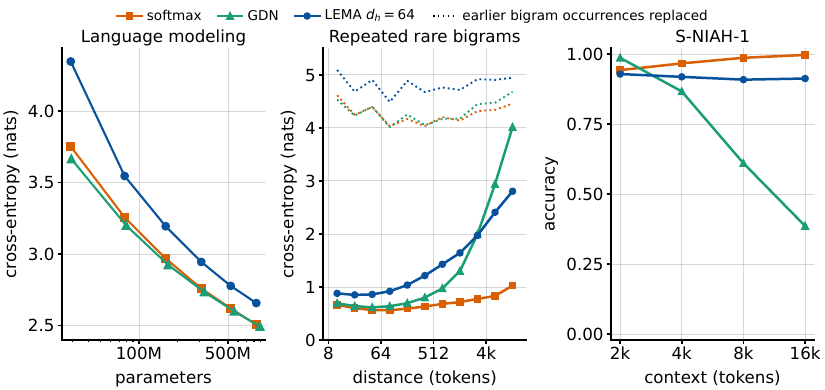}
\caption{Results for language models trained on FineWeb-Edu. \emph{Left:} validation loss against parameter count. \emph{Middle:} loss on the second token of repeated bigrams occurring rarely in training against the distance to the earlier occurrence of the bigram for the largest models, softmax transformers and GDN after context extension to $16$k and LEMA without. \emph{Right:} accuracy on the single-needle task S-NIAH-1 of RULER against the context length for the same models. }
\label{fig:lm}
\end{figure}

\paragraph{Long-range recall.}
Production-scale models combine fixed-state layers like GDN with attention layers~\citep{minimax2025minimax01, kimi2025linear}, as the poor recall of pure fixed-state models isolated in the synthetic recall task also shows in practice despite their competitive loss~\citep{waleffe2024mamba}. To be a viable architecture, LEMA transformers need better long-range recall than fixed-state models. We assess this with two proxies for long-range recall in which softmax transformers outperform fixed-state architectures. In order to measure recall over longer distances, we extend the context length of softmax transformers and GDN by retraining for $1$B tokens at context length $16\,384$, which is crucial for softmax transformers and clearly improves GDN. Context extension worsened LEMA's recall, so we use its original checkpoints. See \Cref{app:lm-extension} for details.

\paragraph{Repeated rare bigrams.}
A simple measure of recall that needs no additional data is the loss on tokens that can be copied from context~\citep{arora2023zoology}. In particular, we average the loss of each model over the second token $y$ of every repeated occurrence of a bigram $xy$ in a document of the validation set, i.e.\ in sequences of the form $(\dots, x, y, \dots, x, y, \dots)$, where $x$ does not reappear with another continuation in between and the bigram $xy$ occurs at most $100$ times per $10$B training tokens.
This tests a simple form of recall: copying the latest observed continuation of $x$, as in the induction-head circuit of \Cref{sec:recall}. These tokens are then mostly parts of rare phrases like names or technical terms. \Cref{fig:lm-recall-examples} shows examples. The measured loss resolved by the distance between the two occurrences of the bigram is shown in \Cref{fig:lm} (middle), together with each model's baseline (dotted), its loss after replacing earlier occurrences of the bigram by random tokens.
The softmax transformer excels on this task even at large distances, while GDN declines sharply with distance. LEMA transformers are behind softmax transformers and GDN at short distances but degrade less with distance than GDN: in the farthest bucket their loss is still $2.1$ nats below their baseline, while GDN is only $0.7$ nats below its baseline and thus makes less use of earlier occurrences to improve its prediction. With intervening competing continuations, LEMA's long-range advantage over GDN mainly holds relative to each model's baseline. Details on the task and further numbers, including for the smaller model sizes where results are mostly similar, can be found in \Cref{app:lm-recall}.

\paragraph{Single-needle retrieval.}
We consider the easiest needle-in-a-haystack task from RULER, S-NIAH-1, which places a seven-digit number in a repeated filler sentence and asks for it at the end. The accuracy of generating the number, averaged over random placements of it between the filler sentences, is shown in \Cref{fig:lm} (right) for context lengths up to $16$k. The exact protocol and further numbers are given in \Cref{app:sniah1}. Consistent with the literature, our softmax transformer performs almost perfectly up to its (extended) context length. Unlike the synthetic task, this task is solvable with a small state size in principle, but GDN fails to solve it at long contexts anyway. For our LEMA transformers, the dictionaries of an $L$-layer model stop changing after at most $L$ repetitions of the filler sentence, so further repetitions leave the generated answer unchanged (\Cref{app:sniah1}). Thus, successful retrieval persists indefinitely through this filler without further state growth.
On further RULER tasks (\Cref{app:ruler-more}), GDN and LEMA rarely generate the answer at this scale, GDN more often than LEMA, while the cross-entropy of the correct answer mostly has GDN ahead at short and LEMA at long contexts.

\paragraph{Training limitations.}
Lower language-modeling loss does not consistently improve recall: training the $77$M model on five times the tokens improves its loss yet worsens its bigram recall at every distance (\Cref{app:lm-tokens}). Further, around $10\%$ of the heads of the $834$M model never find a match (\Cref{app:lm-analysis}). Lastly, the gap to softmax transformers does not stem from stick-breaking attention: training the $77$M model with continuous queries and keys and regular stick-breaking attention gives slightly lower loss than softmax and comparable bigram recall (\Cref{app:lm-ablations}).

\section{Inference with LEMA transformers}\label{sec:inference}
\subsection{Autoregressive generation}
We benchmark the generation speed of GDN, LEMA and softmax transformers of different sizes across context lengths.
The time per generated token is shown in \Cref{fig:inference} (top) with some exact numbers in \Cref{app:inf-benchmark}. As we do not train models at the larger sizes, all models here are untrained, which is irrelevant for softmax transformers and GDN. For LEMA transformers, we replace queries and keys with random ones during generation, which is close to the worst case for the table's occupancy and cache locality (\Cref{app:inf-lema}). Grouped-query attention (GQA)~\citep{ainslie2023gqa}, common in modern language models but not used in \Cref{fig:inference}, reduces the kv-cache by the group size and flattens the softmax decode curves (\Cref{app:inf-state}). \Cref{app:inf-batched} uses $8$-fold GQA for softmax and shows LEMA throughput comparable to GDN at large batch sizes.

\begin{figure}[t]
\centering
\includegraphics[width=0.9\linewidth]{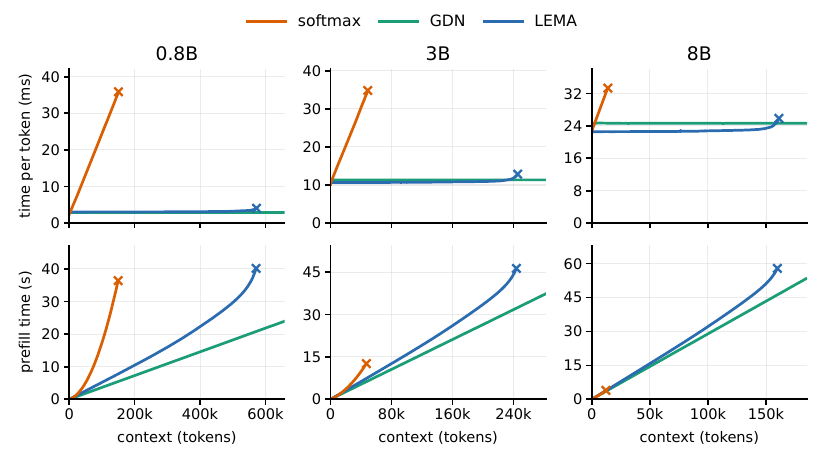}
\caption{Time per generated token (top) and prefill time (bottom) against context length at batch size $1$ on an RTX~3090. LEMA uses random queries and keys. Crosses mark the VRAM limit for softmax and $95\%$ occupancy of LEMA's $50$\,GB hash table in main memory.}
\label{fig:inference}
\end{figure}
\paragraph{Softmax transformer and GDN implementation.} We use vLLM~\citep{kwon2023efficient} as a performant baseline (\Cref{app:inf-softmax}). When generating a single sequence, generation speed is bottlenecked by the streaming of model weights and kv-cache from VRAM to the GPU cores~\citep{shazeer2019fast, pope2023efficiently}. The linearly growing kv-cache in VRAM then explains the observed linear growth of the time per generated token with context length, while the fixed-size state of GDN keeps it constant.

\paragraph{LEMA implementation.} We implement the kv-caches of all heads as a single open-addressing hash table with linear probing in main memory, in place of the trie-based dictionaries used for \Cref{thm:lemasim}. For each generated token, only the model's weights then need to be streamed from VRAM to the GPU cores. Queries, keys and values for each layer are moved to the CPU and used to query and update the kv-cache there as shown in the code in \Cref{fig:teaser}. As long as the hash table is not near its capacity limit, LEMA transformers generate at a constant speed comparable to that of GDN despite their growing state. For the trained $834$M model, the dictionaries hold on average $1.7$k entries per head after $16$k tokens and $18$k after $256$k tokens (\Cref{tab:state-sizes}). Trained models therefore reach the table's capacity at larger context lengths than the random codes used here (\Cref{app:inf-lema}).

\subsection{Prefill}
Prefill processes many prompt tokens together, amortizing the GPU's weight reads across them. For softmax transformers, the attention compute for processing a prompt is quadratic in its length, while LEMA transformers achieve linear-time prefill away from table capacity with alternating lookups and inserts, as does GDN with its chunked kernel. \Cref{fig:inference} (bottom) shows the measured prefill times, with exact numbers in \Cref{app:inf-benchmark}.

\section{Discussion and future work}\label{sec:discussion}
LEMA transformers simulate word-RAMs with chain of thought and are in turn simulated by word-RAMs at a cost per token independent of the context length, closely connecting the two computational models. In practice, generation speed is comparable to that of GDN, and the cache can live in main memory rather than VRAM.

The proposed training method successfully exploits the growing state of LEMA on synthetic recall. The general language models trained with it, while showing some promising results relative to GDN on our long-range recall proxies, remain far behind softmax transformers. We believe the training method contributes to this gap: hardening can be sensitive to the learning rate, some attention heads never find matches, and longer training can worsen recall. Improving the training method is therefore a priority for future work. Finally, training with the surrogate still requires quadratic compute. Unlocking extreme context lengths therefore calls for robust length generalization, as observed on synthetic recall and S-NIAH-1, or a subquadratic training method.

\subsubsection*{Acknowledgments}
The author is grateful for support by the German Research Foundation through Project 553088969 as well as the Cluster of Excellence ``Machine Learning---New Perspectives for Science'' (EXC 2064/1 number 390727645).

\bibliographystyle{style/iclr2027_conference}
\bibliography{main}

\appendix
\crefalias{section}{appsec}
\crefalias{subsection}{appsec}
\crefalias{subsubsection}{appsec}
\section{Additional material for \texorpdfstring{\Cref{sec:ramsimulation}}{Section 3}}\label{app:ram}

\subsection{Word-RAMs}\label{app:ram-model}
There are many definitions of word-RAMs, differing e.g.\ in the instruction set or how input and output are handled. Our definitions choose these aspects in ways that make \Cref{thm:ramsim} and \Cref{thm:lemasim} particularly clean, but changing the word-RAM definition to another common one would mainly change the results by moving factors of $w$ and changing where the input and output lengths enter the scalings.

For $k\in\N$, write $[k]:=\{0,\dots,k-1\}$ and identify $[2^w]$ with the set of $w$-bit words $\{0,1\}^w$.

\begin{definition}[Word-RAM]\label{def:wordram}
    A \emph{word-RAM} is a triple $M=(P,r,w)$ consisting of integers $w\ge2$ and $1\le r\le 2^w$, and a nonempty program $P=(I_0,\dots,I_{|P|-1})$ with $|P|\le 2^w$. Every instruction has one of the forms
    \[
      \begin{array}{lll}
        R_i \leftarrow c, &
        R_i \leftarrow R_j \mathbin{\odot} R_{j'}, &
        R_i \leftarrow \operatorname{msb}(R_j), \\[2pt]
        R_i \leftarrow \mem[R_j], &
        \mem[R_i] \leftarrow R_j, &
        \texttt{if }R_i\neq 0\texttt{ goto }R_j, \\[2pt]
        \texttt{halt}, &&
      \end{array}
    \]
    where $i,j,j'\in[r]$, $c\in[2^w]$, and
    \[
        \odot\ \in\ \{+,-,\times,\ll,\gg,\wedge,\vee,\mathbin{\mathrm{xor}},<,\le,=,\neq\}.
    \]
    Arithmetic operations are unsigned modulo $2^w$. Shift operations are defined as $a\ll b=(a\cdot2^b)\bmod 2^w$ and $a\gg b=\lfloor a/2^b\rfloor$. Boolean operations $\wedge, \vee, \mathbin{\mathrm{xor}}$ act bitwise, and comparisons return $0$ or $1$. Finally,
    \[
        \operatorname{msb}(a):=
        \begin{cases}
            \lfloor\log_2 a\rfloor,&a>0,\\
            0,&a=0.
        \end{cases}
    \]
\end{definition}

\begin{definition}[Execution]\label{def:ramexec}
    A configuration is a triple $\gamma=(\pc,\rho,\mu)$ with $\pc\in[|P|]\cup\{\bot\}$, register state $\rho:[r]\to[2^w]$, and memory state $\mu:[2^w]\to[2^w]$. It is halting exactly when $\pc=\bot$.

    For $g\in[|P|]$, let $\operatorname{next}(g)=g+1$ if $g+1<|P|$, and $\operatorname{next}(g)=\bot$ otherwise. From a non-halting configuration, first set $\pc'=\operatorname{next}(\pc)$, $\rho'=\rho$ and $\mu'=\mu$, and then apply the relevant update:
    \[
    \begin{array}{ll}
      R_i\leftarrow c & \rho'(i)=c,\\
      R_i\leftarrow R_j\mathbin{\odot}R_{j'}&
        \rho'(i)=\rho(j)\mathbin{\odot}\rho(j'),\\
      R_i\leftarrow\operatorname{msb}(R_j)&
        \rho'(i)=\operatorname{msb}(\rho(j)),\\
      R_i\leftarrow\mem[R_j]&\rho'(i)=\mu(\rho(j)),\\
      \mem[R_i]\leftarrow R_j&\mu'(\rho(i))=\rho(j),\\
      \texttt{if }R_i\neq0\texttt{ goto }R_j&
        \pc'=\rho(j)\text{ if }\rho(i)\neq0\text{ and }\rho(j)<|P|,\\
      &\pc'=\bot\text{ if }\rho(i)\neq0\text{ and }\rho(j)\ge|P|,\\
      \texttt{halt}&\pc'=\bot.
    \end{array}
    \]
    An untaken conditional jump keeps the default value $\operatorname{next}(\pc)$.
\end{definition}

\begin{definition}[Computation, time and space]\label{def:ramcomp}
    Let $x=(x_1,\dots,x_n)\in[2^w]^n$ with $n<2^w$. The initial configuration is
    \[
        \gamma_0(x)=(0,\rho_0,\mu_0),\qquad
        \rho_0\equiv0,\qquad
        \mu_0=(n,x_1,\dots,x_n,0,0,\dots).
    \]
    If $t$ is the first time at which the induced execution reaches a halting configuration, then $M$ \emph{halts in $t$ steps}. Its output is $y=(\mu_t(1),\dots,\mu_t(m))$, where $m=\mu_t(0)$, and
    \begin{align*}
        t_M(x)&:=t,\\
        s_M(x)&:=r+\left|\left\{a\in[2^w]:
        \begin{array}{l}
            a\le\max(n,m),\text{ or cell }a\text{ is accessed}\\[-1pt]
            \text{during the execution}
        \end{array}\right\}\right|,
    \end{align*}
    where an access is either a load from or a store to the cell.
\end{definition}

\subsection{LEMA transformers}\label{app:lema-model}
In order to simulate LEMA transformers with word-RAMs, all computational intermediates are rounded to finite precision and we hence define floating-point formats in a standard way. Furthermore, addition with rounding is not associative anymore and hence the order of all operations needs to be fixed in order for a LEMA transformer's forward pass to be well-defined. 

\begin{definition}[Floating-point formats]\label{def:float}
    Consider integer \emph{mantissa} and \emph{exponent} precisions $p_m\ge1$ and $p_e\ge2$, and let $e_{\max}=2^{p_e-1}-1$. The corresponding floating-point format is
    \begin{align*}
      \float(p_m,p_e):={}&\{0\}\\
      &{}\cup\left\{(-1)^s f\,2^{1-e_{\max}-p_m}:s\in\{0,1\},\ 1\le f<2^{p_m}\right\}\\
      &{}\cup\left\{(-1)^s(1+f2^{-p_m})2^{E-e_{\max}}:
          s\in\{0,1\},\ 1\le E\le2^{p_e}-2,\ 0\le f<2^{p_m}\right\}.
    \end{align*}
    Here $f,E\in\mathbb{Z}$. The second line consists of the subnormal values and the third of the normal values. The format's precision is $p=1+p_e+p_m$ bits. 

    Let $\Omega=(2-2^{-p_m-1})2^{e_{\max}}$. For a real number $z$ with $|z|<\Omega$, let $\rd(z)$ be a nearest element of $\float$, with ties broken toward the value whose least significant stored fraction bit is zero, where $0$ is stored with all fields zero. For $|z|\ge\Omega$ we say that rounding $z$ \emph{overflows} and leave $\rd(z)$ undefined. On its domain, $\rd$ coincides with IEEE round-to-nearest ties-to-even, including gradual underflow, since $\Omega$ is the midpoint between the largest element of $\float$ and $2^{e_{\max}+1}$ and IEEE rounding hence returns a finite value exactly when $|z|<\Omega$. We write
    \[
        a\oplus b:=\rd(a+b),\qquad a\otimes b:=\rd(ab).
    \]
    A floating-point computation is \emph{finite} if none of its operations overflows.

    Every scalar sum required in the matrix-matrix and matrix-vector multiplications below is evaluated from left to right, starting at $0$, and every product is rounded before it is added. Coordinates in a dot product are processed in increasing order, and attention heads are accumulated in increasing head order. Bias and residual additions are performed in the order in which they are displayed.
\end{definition}

For the simulations, we use ReLU MLPs and omit normalization and positional encodings, as specified in \Cref{sec:ramsimulation}. The transformers constructed in \Cref{thm:ramsim} have residual states in $\{-1,1\}^d$ at every sublayer boundary (\Cref{app:construction}), so normalization by their root mean square would be the identity in exact arithmetic. The architecture used in the experiments is described in \Cref{app:model-details}. Throughout, $\sgn(u)=1$ for $u\ge0$ and $\sgn(u)=-1$ otherwise, applied coordinatewise.

\begin{definition}[LEMA transformer]\label{def:lematransformer}
    A LEMA transformer $T$ consists of a finite vocabulary $\voc$ containing distinct tokens $\texttt{<out>}$ and $\texttt{<eos>}$, positive integers $L,H,d,d_h,\dff$, a format $\float$ as in \Cref{def:float}, and parameters
    \[
      \begin{array}{ll}
        \operatorname{emb},\operatorname{unemb}\in\float^{|\voc|\times d},&\\
        W_Q^{\ell,h},W_K^{\ell,h},W_V^{\ell,h}\in\float^{d_h\times d},&
        W_O^{\ell,h}\in\float^{d\times d_h},\\
        W_1^\ell\in\float^{\dff\times d},\ b^\ell\in\float^{\dff},&
        W_2^\ell\in\float^{d\times\dff}
      \end{array}
    \]
    for $\ell\in\{1,\dots,L\}$ and $h\in\{1,\dots,H\}$. The number $N$ of \emph{parameters} is the total number of entries in these matrices and vectors.
\end{definition}

\begin{definition}[Forward pass]\label{def:lemaforward}
    Let $\tau=(\tau_1,\dots,\tau_K)\in\voc^K$ with $K\ge1$, and set $x_i^{(0)}=\operatorname{emb}_{\tau_i}$. For $\ell=1,\dots,L$, $h=1,\dots,H$ and $i=1,\dots,K$, compute
    \begin{align*}
      q_i^{\ell,h}&=\sgn(W_Q^{\ell,h}x_i^{(\ell-1)}),&
      k_i^{\ell,h}&=\sgn(W_K^{\ell,h}x_i^{(\ell-1)}),&
      v_i^{\ell,h}&=W_V^{\ell,h}x_i^{(\ell-1)},
    \end{align*}
    and
    \[
      \bigl(o_i^{\ell,h}\bigr)_{i=1}^K=\lemaop\bigl((q_i^{\ell,h},k_i^{\ell,h},v_i^{\ell,h})_{i=1}^K\bigr)
    \]
    as in \Cref{def:lema}, followed by
    \begin{align*}
      x_i^{(\ell-\frac12)}
        &=x_i^{(\ell-1)}+\sum_{h=1}^H W_O^{\ell,h}o_i^{\ell,h},\\
      z_i^\ell&=\relu(W_1^\ell x_i^{(\ell-\frac12)}+b^\ell),\\
      x_i^\ell&=x_i^{(\ell-\frac12)}+W_2^\ell z_i^\ell.
    \end{align*}
    Here $\relu(u)=\max(u,0)$ coordinatewise. All scalar operations and sums use the order fixed in \Cref{def:float}. The forward pass is defined only if it is finite. Its prediction is
    \[
        T(\tau):=\argmax_{\sigma\in\voc}
        \left\langle\operatorname{unemb}_\sigma,x_K^{(L)}\right\rangle,
    \]
    with ties broken by a fixed total order on $\voc$. This order also identifies $\voc$ with $[|\voc|]$ when tokens are supplied to a word-RAM.
\end{definition}

\begin{definition}[Generation, time and space]\label{def:lemagen}
    Let $\voc_0=\voc\setminus\{\texttt{<out>},\texttt{<eos>}\}$ and let the prompt $\tau_{1:n}\in\voc^n$ be nonempty. Autoregressive generation appends
    \[
        \tau_{k+1}=T(\tau_{1:k}),\qquad k\ge n.
    \]
    We say that $T$ \emph{generates} $y\in\voc_0^m$ from $\tau_{1:n}$ in $t$ steps if $t$ is minimal with $\tau_{n+t}=\texttt{<eos>}$ and, for some $z\in\voc_0^*$,
    \[
        (\tau_{n+1},\dots,\tau_{n+t})
        =(z,\texttt{<out>},y,\texttt{<eos>}).
    \]
    This factorization is unique because neither $z$ nor $y$ contains \texttt{<out>} or \texttt{<eos>}. We set $t_T(\tau):=t$ and
    \[
      s_T(\tau):=\sum_{\ell=1}^L\sum_{h=1}^H
      \left|\{k_i^{\ell,h}:1\le i\le n+t-1\}\right|.
    \]
    The final \texttt{<eos>} token is generated but not processed and therefore contributes no key.
\end{definition}

Finally, we need to define the encoding of word-RAM inputs and outputs into transformer tokens, which is used in \Cref{thm:ramsim} but not specified there.
\begin{definition}[Encoding]\label{def:enc}
    Write $\bits_w(a)\in\{\texttt{0},\texttt{1}\}^w$ for the binary representation of $a\in[2^w]$, most significant bit first, and let
    \[
        C_j(v):=(\texttt{mem},\bits_w(j),\texttt{:},\bits_w(v))
    \]
    be the block encoding memory cell $j$ with content $v$. A word sequence $a=(a_1,\dots,a_k)\in[2^w]^k$ with $k<2^w$ is encoded as its length followed by its words, one block per memory cell and separated by \texttt{\#} tokens,
    \[
        \enc(a):=C_0(k)\ \texttt{\#}\ C_1(a_1)\ \texttt{\#}\ \cdots\ \texttt{\#}\ C_k(a_k).
    \]
    Note that prepending the length of input $x$ and output $y$ is precisely how word-RAMs in \Cref{def:ramcomp} handle inputs and outputs. The tokens of $\enc$ belong to the constant vocabulary
    \[
      \voc=\{\texttt{0},\texttt{1},\texttt{\#},\texttt{:},
        \texttt{mem},\texttt{reg},\texttt{pc},
        \texttt{<out>},\texttt{<eos>},\texttt{step}\}
    \]
    of the transformer in \Cref{thm:ramsim}.
\end{definition}

\subsection{Proof of \texorpdfstring{\Cref{thm:ramsim}}{Theorem 1}}\label{app:ramsim-proof}
In order to prove \Cref{thm:ramsim}, we start by defining the token sequence that the LEMA transformer simulating a given word-RAM will produce for each input and show that it has the right length. The theorem is then restated, asserting the existence of a LEMA transformer with precisely specified dimensions that predicts this token sequence (\Cref{prop:predictor}). The proof consists of constructing this LEMA transformer, which is done in \Cref{app:construction}.

\subsubsection{Token sequence and exact statement}\label{app:transcript}
Throughout, fix a word-RAM $M=(P,r,w)$ and an input $x=(x_1,\dots,x_n)$ on which $M$ halts, with execution $\gamma_0,\dots,\gamma_t$ and output $y$ of length $m$. Write $\pc_k$ for the program counter of $\gamma_k$, so that $\pc_0=0$ and $\pc_k=\bot$ only for $k=t$.

The transformer keeps the state of $M$ in the token sequence as a log of all writes. A write of the value $v$ to the register or memory cell $a$ is spelled out as the \emph{record}
\[
    D\ \bits_w(a)\ \texttt{:}\ \bits_w(v)\ \texttt{\#}
    \qquad\text{with } D\in\{\texttt{reg},\texttt{mem}\},
\]
which we read as ``$D$-address $a$ now holds $v$''. The prompt $\enc(x)$ followed by \texttt{\#} is a sequence of such records, one for every memory cell $0,\dots,n$ of the initial memory $\mu_0$. Every step of $M$ is then represented by the token \texttt{step}, followed by the record of the write the step performs, if it performs one, and by \texttt{pc} and the bits of the new program counter. At any point, the current content of a register or memory cell is the value of the latest record with that address, or $0$ if there is none, since $M$ starts from the all-zero state except for the input. Finding the latest record with a given address is exactly what a LEMA head does when the address is its key.

\begin{definition}[Transcript]\label{def:transcript}
For $k\in\{1,\dots,t\}$ let $W_k$ be the empty sequence if step $k$ does not write, and otherwise
\[
    W_k=(D_k,\bits_w(a_k),\texttt{:},\bits_w(v_k),\texttt{\#})
\]
if step $k$ writes the value $v_k$ to the address $a_k$ of type $D_k\in\{\texttt{reg},\texttt{mem}\}$. The \emph{transcript} of $M$ on $x$ is the token sequence
\begin{equation}\label{eq:ram-transcript}
\begin{split}
 \tau(x)={}&\enc(x)\ \texttt{\#}\ \texttt{pc}\,\bits_w(0)\\
 &\bigl(\texttt{step}\ W_1\ \texttt{pc}\,\bits_w(\pc_1)\bigr)\cdots
   \bigl(\texttt{step}\ W_{t-1}\ \texttt{pc}\,\bits_w(\pc_{t-1})\bigr)\\
 &\texttt{step}\ W_t\ \texttt{<out>}\ \enc(y)\ \texttt{<eos>}
\end{split}
\end{equation}
over the vocabulary of \Cref{def:enc}.
\end{definition}

\begin{figure}
\centering\small
\begin{tabular}{@{}ll@{}}
\toprule
part & tokens\\
\midrule
$\enc(x)$ & \texttt{mem 0 0 0 : 0 0 1 \# mem 0 0 1 : 0 1 0}\\
initial program counter & \texttt{\# pc 0 0 0}\\
step 1, $R_1\leftarrow1$ & \texttt{step reg 0 0 1 : 0 0 1 \# pc 0 0 1}\\
step 2, $R_0\leftarrow\mem[R_1]$ & \texttt{step reg 0 0 0 : 0 1 0 \# pc 0 1 0}\\
step 3, $R_0\leftarrow R_0+R_0$ & \texttt{step reg 0 0 0 : 1 0 0 \# pc 0 1 1}\\
step 4, $\mem[R_1]\leftarrow R_0$ & \texttt{step mem 0 0 1 : 1 0 0 \# pc 1 0 0}\\
step 5, \texttt{halt} & \texttt{step <out>}\\
$\enc(y)$ & \texttt{mem 0 0 0 : 0 0 1 \# mem 0 0 1 : 1 0 0 <eos>}\\
\bottomrule
\end{tabular}
\caption{The transcript for $w=3$, the program $P=(R_1\leftarrow1,\ R_0\leftarrow\mem[R_1],\ R_0\leftarrow R_0+R_0,\ \mem[R_1]\leftarrow R_0,\ \texttt{halt})$ and the input $x=(2)$, which produces the output $y=(4)$. Consecutive tokens are separated by one space, so every word consists of three bit tokens, and the rows are to be read as one sequence. Everything after $\enc(x)$ is generated.}
\label{fig:transcript}
\end{figure}

\Cref{fig:transcript} shows an example. The initial program counter block has $w+2$ tokens, every step except the last contributes at most $3w+5$ tokens, the last step contributes at most $2w+6$ tokens plus $\enc(y)$, and $|\enc(y)|=(m+1)(2w+3)-1$. Hence
\begin{equation}\label{eq:transcript-length}
    |\tau(x)|-|\enc(x)|\le t(3w+5)+(m+1)(2w+3)+2.
\end{equation}

The transcript is designed so that the current state of $M$ can be read off with latest-match lookups. Before \texttt{<out>}, the records of the transcript are the initial memory cells $0,\dots,n$ and then the writes of the steps in the order in which they are executed, and every record is complete before the next \texttt{step} token. This gives the following lemma, which is the heart of the construction.

\begin{lemma}[Latest write]\label{lem:latestwrite}
Let $i$ be the position of the \texttt{step} token of step $k$, or any position after \texttt{<out>}, in which case put $k=t+1$. Fix $\theta\in\{\texttt{reg},\texttt{mem}\}$ and an address $a$, and consider the records of $\tau(x)$ that lie before $i$ and before \texttt{<out>}, and have type $\theta$ and address $a$. If there is such a record, the value of the latest one is the content of $\theta$-address $a$ in $\gamma_{k-1}$. If there is none, this content is $0$.
\end{lemma}
\begin{proof}
The records before $i$ and before \texttt{<out>} are, in this order, $(\texttt{mem},j,\mu_0(j))$ for $j=0,\dots,n$ and then the writes of steps $1,\dots,k-1$. By \Cref{def:ramexec,def:ramcomp}, $\gamma_{k-1}$ arises from the all-zero register and memory state by exactly these writes in this order, since the cells $0,\dots,n$ are the only nonzero cells of $\mu_0$. A sequence of writes leaves at every address the value of the last write to it, or $0$ if there is none.
\end{proof}

We now state precisely what will be constructed, with explicit sizes.

\begin{proposition}\label{prop:predictor}
There is a LEMA transformer $T$ over the vocabulary of \Cref{def:enc} with
\[
L=2w+6,\qquad H=3,\qquad d=9w+62,\qquad d_h=w+5,\qquad \dff\le|P|+21w+23
\]
and integer parameters of magnitude at most $\max(w+1,5)$, such that, for every input $x$ on which $M$ halts, the following holds with $\tau=\tau(x)$ and the forward pass of $T$ evaluated in exact arithmetic:
\begin{enumerate}
\item[(i)] $T(\tau_{1:i})=\tau_{i+1}$ for all $|\enc(x)|\le i<|\tau|$, i.e.\ $T$ continues $\enc(x)$ to $\tau$,
\item[(ii)] every intermediate value, including every partial sum in the evaluation order of \Cref{def:float}, is an integer of magnitude at most $\max(2w+3,11)$, and
\item[(iii)] the keys of $T$ on $\tau$ satisfy $s_T\le3s_M(x)+6w+25$.
\end{enumerate}
\end{proposition}

\begin{proof}[Proof of \Cref{thm:ramsim} from \Cref{prop:predictor}]
Choose the format $\float(p_m,p_e)$ with $p_m=\lceil\log_2\max(2w+3,11)\rceil$ and $p_e=\lceil\log_2(p_m+2)\rceil+1$. Its normal numbers with exponents up to $p_m$ include every integer of magnitude below $2^{p_m+1}$, and $e_{\max}\ge p_m+1$, so by (ii) every parameter and every value of the forward pass on a prefix of $\tau$ is representable. Hence $\rd$ acts as the identity, no operation overflows, and the floating-point forward pass coincides with the exact one, so (i) holds for the floating-point transformer with the chosen format as well. The precision of this transformer is $p=1+p_m+p_e=\O(\log w)$, and its other sizes are the stated $\O$-bounds. By (i), greedy generation from the prompt $\enc(x)$ produces $\tau_{|\enc(x)|+1},\tau_{|\enc(x)|+2},\dots$ up to the final \texttt{<eos>}, which is the first \texttt{<eos>} of $\tau$. In particular, the generation is of the form $(z,\texttt{<out>},\enc(y),\texttt{<eos>})$ with $z$ and $\enc(y)$ containing neither \texttt{<out>} nor \texttt{<eos>}, so $T$ generates $\enc(y)$ from $\enc(x)$ in the sense of \Cref{def:lemagen}, in $|\tau|-|\enc(x)|=\O((t+|y|)w)$ steps by \eqref{eq:transcript-length}, and with $\O(s_M(x)+w)$ distinct keys by (iii).
\end{proof}

\subsubsection{The construction}\label{app:construction}
\noindent\textit{Proof of \Cref{prop:predictor}.}
\paragraph{Overview.}
Consider the \texttt{step} token of some step $k$ in \Cref{fig:transcript}. When the transformer processes this token, it has to decide what comes next: the first token of the record of step $k$, or \texttt{pc} if the step does not write, or \texttt{<out>} if the machine halts. For this it needs the current instruction, whose index $\pc_{k-1}$ is spelled out in the $w$ tokens right before it, the contents of the registers the instruction reads, and for a load the content of a memory cell. All of this is available in the transcript: by \Cref{lem:latestwrite}, the current content of any register or cell is the value in the latest earlier record with that address. The transformer therefore proceeds as in \Cref{tab:layers}. Layers $1$ to $w+1$ parse the transcript, so that afterwards every \texttt{\#} holds the record it closes and every \texttt{step} token its decoded instruction, layers $w+2$ and $w+3$ look up the operands, which is the latest exact match rule with the type and address of a record as key, and layers $w+4$ to $2w+3$ compute the operation one bit per layer, which layer $2w+4$ finalizes. Layer $2w+5$ assembles the write of the step, the next program counter and whether the machine halts, and layer $2w+6$ emits the next token, where every token after a \texttt{step} token copies the assembled result from it.
The output phase works in the same way, with the \texttt{\#} tokens between the output records in place of the \texttt{step} tokens: such a \texttt{\#} looks up the next memory cell and emits its record. We call the \texttt{step} tokens, the \texttt{mem} token directly after \texttt{<out>} and the \texttt{\#} tokens after \texttt{<out>} the \emph{control positions}, since everything that is emitted is decided at them. The constructed transformers lie in a subclass of all LEMA transformers using the residual stream, MLPs and attention heads only in specific limited ways as explained in the next three paragraphs.

\begin{table}[t]
\caption{The layers of the construction.}
\label{tab:layers}
\centering\small
\begin{tabular}{@{}llp{7.6cm}@{}}
\toprule
layers & module & what is known afterwards\\
\midrule
$1$ & structure & token kinds, latest marker, roles of the positions\\
$1,\dots,w$ & window & at every position: the bits of the $w$ tokens before it\\
$w+1$ & records and decoding & at every \texttt{\#} before \texttt{<out>}: type, address and value of its record. At every \texttt{step}: the instruction, the registers to read. At every output \texttt{\#}: the next cell\\
$w+2,w+3$ & reading the state & at every \texttt{step}: the operands, and the loaded cell for a load. At every output control position: the cell to emit\\
$w+4,\dots,2w+4$ & arithmetic & at every \texttt{step}: the result of the operation\\
$2w+5$ & assembly & at every control position: the record to emit, the next program counter, halting\\
$2w+6$ & emission & at every position: the next token\\
\bottomrule
\end{tabular}
\end{table}

\paragraph{The residual stream.}
Every coordinate of the residual stream holds $+1$ or $-1$ at every sublayer boundary. We call such a vector a \emph{binary state}, and we partition its coordinates into named \emph{fields}. These fields for our construction are listed in \Cref{tab:fields}. A field of length one is a \emph{flag} and holds a predicate, $+1$ for true. A field of length $w$ is a \emph{word} and holds a number $a\in[2^w]$ as $\bits_w(a)$, most significant bit first, with $-1$ for the bit $0$. A field is at its \emph{default} if all its coordinates are $-1$, so a word at its default holds $0$. We call the tokens \texttt{mem}, \texttt{reg}, \texttt{:} and \texttt{pc} \emph{markers}, since each is followed by a word of $w$ bit tokens. The embedding of a token sets the fields $\mathtt{one}$, its own coordinate of $\mathtt{tok}$, the flag $\mathtt{eq}$ and the coordinate $\mathtt{next}_{\texttt{0}}$ to $+1$, and $\mathtt{dist}_0$ to $+1$ if the token is a marker, while every other coordinate starts at $-1$. The unembedding of a token $\sigma$ is the indicator vector of the coordinate $\mathtt{next}_\sigma$, so as long as $\mathtt{next}$ is one-hot, the prediction is the token whose coordinate is $+1$.

\begin{table}[t]
\caption{The fields of the residual stream, $d=9w+62$ coordinates in total.}
\label{tab:fields}
\centering\small
\begin{tabular}{@{}p{4.6cm}lp{7.2cm}@{}}
\toprule
field & size & content \\
\midrule
$\mathtt{one}$ & $1$ & the constant $+1$\\
$\mathtt{tok}$ & $10$ & one-hot of the token, with coordinates $\mathtt{tok}_\sigma$\\
$\mathtt{dist}_0,\dots,\mathtt{dist}_w$ & $w+1$ & $\mathtt{dist}_j$: the token $j$ positions earlier is a marker\\
$\mathtt{near}$ & $4$ & one-hot of the latest marker before this position, with coordinates $\mathtt{near}_\sigma$, at its default if there is none\\
$\mathtt{prevout}$, $\mathtt{out}$ & $2$ & the previous token is \texttt{<out>}, and some earlier token is \texttt{<out>}\\
$\mathtt{exec}$, $\mathtt{commit}$, $\mathtt{outsep}$, $\mathtt{outstart}$, $\mathtt{ctrl}$ & $5$ & roles: the token is \texttt{step}, and see the structure module for the rest\\
$\mathtt{win}$ & $w$ & $\mathtt{win}_k$: the token $w+1-k$ positions earlier is \texttt{1}, so the word holds the bits of the $w$ preceding tokens\\
$\mathtt{addr}$, $\mathtt{wtype}$ & $w+1$ & at a \texttt{\#} before \texttt{<out>}: address and type of the record it closes, $\mathtt{wtype}$ for \texttt{reg}\\
$\mathtt{A}$, $\mathtt{B}$ & $2w$ & the operands\\
$\mathtt{X}$, $\mathtt{Y}$, $\mathtt{Z}$, $\mathtt{Q}$ & $4w$ & workspace: the result and then the value to emit, the address to emit, the next program counter, the product\\
$\mathtt{q1}$, $\mathtt{q2}$, $\mathtt{load}$, $\mathtt{store}$, $\mathtt{jnz}$, $\mathtt{cmp}$, $\mathtt{op}$ & $19$ & the decoded instruction: lookup guards, instruction kinds, one-hot of the operation\\
$\mathtt{carry}$, $\mathtt{eq}$, $\mathtt{lt}$, $\mathtt{iszero}$ & $4$ & arithmetic flags\\
$\mathtt{write}$, $\mathtt{wreg}$, $\mathtt{stop}$, $\mathtt{output}$, $\mathtt{last}$ & $5$ & the assembled step: it writes, it writes a register, the machine halts, output phase, last output cell\\
$\mathtt{next}$ & $10$ & one-hot of the predicted token, with coordinates $\mathtt{next}_\sigma$\\
\bottomrule
\end{tabular}
\end{table}

\paragraph{MLPs.}
Every neuron of the construction has incoming weights in $\{-1,0,1\}$, bias $1-k$ where $k$ is the number of nonzero incoming weights, and outgoing weights in $\{-2,0,2\}$. On a binary state the neuron outputs $1$ if the state agrees with the signs of its incoming weights on all $k$ coordinates and $0$ otherwise, since every disagreement lowers the pre-activation by $2$. We say that the neuron \emph{fires} on this condition, and write conditions as conjunctions of flags, negated flags and equations $\mathtt{X}=a$ for words $\mathtt{X}$. A firing neuron adds $\pm2$ to the coordinates it writes. Throughout, $+2$ is only added to coordinates that are $-1$ and $-2$ only to coordinates that are $+1$, and no two neurons of a layer that write the same coordinate fire together, so an MLP flips some coordinates and the state stays binary. To \emph{copy} a word into another word, potentially gated on additional coordinates having specific values, takes $2w$ neurons: for each bit position, one neuron adds $2$ to the target bit if the source bit is $1$ and the target bit is $-1$, and one subtracts $2$ if the source bit is $-1$ and the target bit is $1$. To \emph{clear} a word, i.e.\ set all its entries to $-1$, takes $w$ neurons, one per bit that subtracts $2$ when it is $+1$, with the same optional gating. To \emph{write} a constant into a word at its default takes one neuron that sets its $1$-bits.

\paragraph{Attention heads.}
Every head of the construction copies entries of the state from an earlier position, and its projections only select coordinates. The query and key of a position are lists of entries $\pm x_c$ of its state, given by rows $\pm e_c$ of $W_Q$ and $W_K$, so they are already $\pm1$ and the binarization does nothing. The value is a list of entries $x_c+1\in\{0,2\}$, given by rows $e_c+e_{\mathtt{one}}$ of $W_V$, and $W_O$ adds every value entry to one destination coordinate. We arrange that the destination coordinates are at their default $-1$ whenever the head can attend. So the head sets them to the copied entries of the latest earlier position whose key equals the current query, and leaves the state unchanged if there is none. The number of distinct keys of a head is the number of distinct values its key takes over the positions. We use three kinds of heads. If query and key both consist of $\mathtt{one}$, all keys agree and the head copies from the previous position. A flag $\mathtt{f}$ in the key against $\mathtt{one}$ in the query restricts the copy to the latest position with $\mathtt{f}$, since only keys with $\mathtt{f}=+1$ match. A flag $\mathtt{g}$ in the query against $\mathtt{one}$ in the key restricts the head to positions with $\mathtt{g}$, since a query with $\mathtt{g}=-1$ matches no key. A word in the query against a word in the key gives a dictionary lookup, which copies from the latest position whose key word equals the query word.

\Cref{def:lematransformer} fixes one head dimension and one number of heads for all layers, so shorter queries and keys are padded with $\mathtt{one}$, shorter values with zero rows, layers with fewer heads are completed by heads whose query and key consist of $\mathtt{one}$ and whose value is zero, which contribute one key each, and unused neurons have zero weights.

\paragraph{Structure (layer 1).}
The first layer sets flags describing each token's role and surroundings, for example distinguishing a bit token inside an address from one inside a value, and positions before and after \texttt{<out>}. Three heads act at every position: one copies $\mathtt{dist}_0$, $\mathtt{tok}_{\texttt{1}}$ and $\mathtt{tok}_{\texttt{<out>}}$ from the previous position into $\mathtt{dist}_1$, $\mathtt{win}_w$ and $\mathtt{prevout}$, one copies the four coordinates of $\mathtt{tok}$ for \texttt{mem}, \texttt{reg}, \texttt{:} and \texttt{pc} from the latest position with $\mathtt{dist}_0$ into $\mathtt{near}$, and one copies $\mathtt{one}$ from the latest position with $\mathtt{tok}_{\texttt{<out>}}$ into $\mathtt{out}$. Where there is no such position, the destinations keep their default, which is the correct value. Four neurons then set the roles:
\[
\begin{array}{ll}
 \mathtt{exec}:\ \mathtt{tok}_{\texttt{step}}, &
 \mathtt{commit}:\ \mathtt{tok}_{\texttt{\#}}\wedge\neg\mathtt{out},\\
 \mathtt{outsep}:\ \mathtt{tok}_{\texttt{\#}}\wedge\mathtt{out}, &
 \mathtt{outstart}:\ \mathtt{tok}_{\texttt{mem}}\wedge\mathtt{prevout},\\
 \mathtt{ctrl}:\ \mathtt{exec}\vee\mathtt{outsep}\vee\mathtt{outstart}, &
\end{array}
\]
where the neurons for $\mathtt{exec}$, $\mathtt{outsep}$ and $\mathtt{outstart}$ also set $\mathtt{ctrl}$. Their firing conditions are mutually exclusive, so this realizes the displayed disjunction without additional neurons. The execution flag $\mathtt{exec}$ is simply a copy of $\mathtt{tok}_{\texttt{step}}$; $\mathtt{commit}$ marks the \texttt{\#} tokens before \texttt{<out>}, which close the input and execution records, $\mathtt{outsep}$ those after it, $\mathtt{outstart}$ the \texttt{mem} after \texttt{<out>}, and $\mathtt{ctrl}$ the control positions.

\paragraph{Window (layers 1 to $w$).}
This module collects the preceding word at each \texttt{:}, \texttt{\#} and \texttt{step} token and records each bit token's distance from its marker. In layer $j$, for $j=1,\dots,w$, one head copies $\mathtt{dist}_{j-1}$ and $\mathtt{win}_{w+2-j}$ from the previous position into $\mathtt{dist}_j$ and $\mathtt{win}_{w+1-j}$, where $\mathtt{win}_{w+1}$ stands for $\mathtt{tok}_{\texttt{1}}$. For $j=1$ this is the first head of the structure module. By induction on $j$, after layer $j$ the flag $\mathtt{dist}_j$ of a position says whether the token $j$ positions earlier is a marker, and $\mathtt{win}_{w+1-j}$ whether it is \texttt{1}. Each of these coordinates is written in exactly one layer, so it is at its default before. Hence after layer $w$ every position holds in $\mathtt{win}$ the bits of the $w$ tokens before it, most significant first, where every token other than \texttt{1} counts as $0$.

On a transcript this means the following. Every marker is followed by exactly $w$ bit tokens, its word. So a bit token at distance $j$ from its marker has $\mathtt{dist}_j$ set and no other $\mathtt{dist}_{j'}$ with $1\le j'\le w$, and $\mathtt{near}$ names its marker, while every other token has none of $\mathtt{dist}_1,\dots,\mathtt{dist}_w$ set. Moreover, a \texttt{:} token holds in $\mathtt{win}$ the address of its record and in $\mathtt{near}$ the marker of its record, a \texttt{\#} holds in $\mathtt{win}$ the value of its record and has $\mathtt{near}=\texttt{:}$, and a \texttt{step} token holds in $\mathtt{win}$ the program counter of the block before it and has $\mathtt{near}=\texttt{pc}$. Later modules use $\mathtt{win}$ only at these three kinds of tokens.

\paragraph{Records and decoding (layer $w+1$).}
Two heads collect the records. At positions with $\mathtt{commit}$, one copies $\mathtt{win}$ and $\mathtt{near}_{\texttt{reg}}$ from the latest position with $\mathtt{tok}_{\texttt{:}}$ into $\mathtt{addr}$ and $\mathtt{wtype}$. The latest \texttt{:} before a \texttt{\#} is the \texttt{:} of its record, so afterwards every \texttt{\#} before \texttt{<out>} holds in $(\mathtt{wtype},\mathtt{addr},\mathtt{win})$ the type, address and value of its record, with $\mathtt{wtype}=+1$ for a register. At positions with $\mathtt{outsep}$, the other head copies $\mathtt{win}$ from the latest position with $\mathtt{tok}_{\texttt{:}}$ into $\mathtt{X}$, so a \texttt{\#} closing the output record of cell $a-1$ holds $a-1$ in $\mathtt{X}$. The fields $\mathtt{addr}$ and $\mathtt{wtype}$ are written nowhere else, and $\mathtt{win}$ is never written again.

The MLP decodes the instruction with one neuron per instruction $g\in[|P|]$: it fires on $\mathtt{tok}_{\texttt{step}}\wedge\mathtt{win}=g$, sets the flags of \Cref{tab:decode} and writes the indices of the registers to be read into $\mathtt{X}$ and $\mathtt{Y}$. Exactly one of these neurons fires at a \texttt{step} token and none elsewhere. The same MLP prepares the output phase. At a \texttt{\#} with $\mathtt{outsep}$ it sets $\mathtt{q1}$ and $\mathtt{q2}$ and increments $\mathtt{X}$, with one neuron per bit $s$ that fires if bit $s$ of $\mathtt{X}$ is $0$ and all lower bits are $1$, and then sets bit $s$ and clears the lower bits. Exactly one of them fires, since $a-1<m<2^w$ is not all ones, so afterwards $\mathtt{X}=a$ is the address of the next cell to emit. At the \texttt{mem} with $\mathtt{outstart}$ it sets $\mathtt{q1}$ only, so $\mathtt{X}=\mathtt{Y}=0$ there.

\begin{table}[t]
\caption{Decoding: the flags and the register indices written at a \texttt{step} token by the form of its instruction. Words not listed stay $0$.}
\label{tab:decode}
\centering\small
\begin{tabular}{@{}lll@{}}
\toprule
instruction & flags set & $\mathtt{X}$, $\mathtt{Y}$\\
\midrule
$R_i\leftarrow c$ & $\mathtt{write}$, $\mathtt{wreg}$ & \\
$R_i\leftarrow R_j\odot R_{j'}$ & $\mathtt{write}$, $\mathtt{wreg}$, $\mathtt{q1}$, $\mathtt{q2}$, $\mathtt{op}_\odot$, and $\mathtt{cmp}$ if $\odot\in\{<,\le,=,\neq\}$ & $j$, $j'$\\
$R_i\leftarrow\operatorname{msb}(R_j)$ & $\mathtt{write}$, $\mathtt{wreg}$, $\mathtt{q1}$, $\mathtt{op}_{\operatorname{msb}}$ & $j$\\
$R_i\leftarrow\mem[R_j]$ & $\mathtt{write}$, $\mathtt{wreg}$, $\mathtt{q1}$, $\mathtt{load}$ & $j$\\
$\mem[R_i]\leftarrow R_j$ & $\mathtt{write}$, $\mathtt{q1}$, $\mathtt{q2}$, $\mathtt{store}$ & $i$, $j$\\
\texttt{if }$R_i\neq0$\texttt{ goto }$R_j$ & $\mathtt{q1}$, $\mathtt{q2}$, $\mathtt{jnz}$ & $i$, $j$\\
\texttt{halt} & $\mathtt{stop}$ & \\
\bottomrule
\end{tabular}
\end{table}

\paragraph{Reading the state (layers $w+2$ and $w+3$).}
This module performs the memory reads. Each \texttt{\#} with $\mathtt{commit}$ offers its record as a dictionary entry with the key $(-\mathtt{wtype},\mathtt{addr})$, whose first coordinate is $+1$ for a memory cell and $-1$ for a register, and the value $\mathtt{win}$. Layer $w+2$ has two heads. At positions with $\mathtt{q1}$, the first copies $\mathtt{win}$ into $\mathtt{A}$ from the latest position with $\mathtt{commit}$ whose $(-\mathtt{wtype},\mathtt{addr})$ equals $(\mathtt{out},\mathtt{X})$, and at positions with $\mathtt{q2}$, the second copies $\mathtt{win}$ into $\mathtt{B}$ from the latest position with $\mathtt{commit}$ whose $(-\mathtt{wtype},\mathtt{addr})$ equals $(\mathtt{out},\mathtt{Y})$. Since $\mathtt{addr}$ and $\mathtt{wtype}$ are written only at positions with $\mathtt{commit}$, all other positions have the same key. At a \texttt{step} token, $\mathtt{out}$ is $-1$ and $\mathtt{X}$ holds a register index, so the first head finds the latest record of a write to this register, and by \Cref{lem:latestwrite} copies its current content into $\mathtt{A}$, where a miss leaves $\mathtt{A}$ at $0$, which is the content in that case. The same holds for $\mathtt{B}$. At an output control position, $\mathtt{out}$ is $+1$ and $\mathtt{X}$ holds the address $a$ of the cell to emit, so $\mathtt{A}$ becomes $\mu_t(a)$, and at a \texttt{\#} with $\mathtt{outsep}$, where $\mathtt{Y}=0$, also $\mathtt{B}$ becomes $\mu_t(0)=m$, the length of the output. The MLP of layer $w+2$ sets $\mathtt{iszero}$ with a neuron firing on $\mathtt{jnz}\wedge \mathtt{A}=0$ and clears $\mathtt{X}$ and $\mathtt{Y}$ at \texttt{step} tokens, where the register indices are no longer needed.

Layer $w+3$ has one head for loads: at positions with $\mathtt{load}$, it copies $\mathtt{win}$ into $\mathtt{X}$ from the latest position with $\mathtt{commit}$ whose $(-\mathtt{wtype},\mathtt{addr})$ equals $(\mathtt{load},\mathtt{A})$. At a load, $\mathtt{A}$ holds the address of the cell to load and $\mathtt{X}$ was just cleared, so $\mathtt{X}$ becomes $\mu_{k-1}(\mathtt{A})$ by \Cref{lem:latestwrite}. After layer $w+3$, the \texttt{step} token of step $k$ holds the operands of its instruction in $\mathtt{A}$ and $\mathtt{B}$, the loaded value in $\mathtt{X}$ for a load and $0$ otherwise, $\mathtt{Y}=\mathtt{Z}=\mathtt{Q}=0$, and $\mathtt{iszero}=\ind{\mathtt{A}=0}$ for a conditional jump. The words $\mathtt{A}$ and $\mathtt{B}$ are never written again. Every position that is not a control position has $\mathtt{X}$, $\mathtt{Y}$, $\mathtt{Z}$ and $\mathtt{Q}$ at their defaults, and this remains true until the emission layer, since every neuron from here on is guarded by a flag that is set only at control positions.

\paragraph{Arithmetic (layers $w+4$ to $2w+4$).}
This module computes $\mathtt{A}\odot \mathtt{B}$ at the \texttt{step} tokens of arithmetic instructions, with the operations of \Cref{def:wordram}, and it also performs the comparisons needed for jumps and for the output. It consists of $w$ \emph{stages}, stage $s$ in layer $w+4+s$, and a finalization layer. Every neuron of a stage is guarded by a coordinate of $\mathtt{op}$ or by $\mathtt{cmp}$, $\mathtt{jnz}$, $\mathtt{outsep}$ or $\mathtt{outstart}$, and $\mathtt{A}$ and $\mathtt{B}$ are only read. In this paragraph we index bits by their weight, so bit $s$ of a word has weight $2^s$ and is coordinate $w-s$ of the field. For every operation we state what holds after stage $s$, which follows by induction on $s$ from the neurons of the stage.

\emph{Addition and subtraction.} After stage $s$, bits $0$ to $s$ of $\mathtt{X}$ are the corresponding bits of $\mathtt{A}\pm \mathtt{B}$ and $\mathtt{carry}$ holds the carry, or the borrow, into bit $s+1$. Stage $s$ is a full adder: eight neurons per operation, one per assignment of $\mathtt{A}_s$, $\mathtt{B}_s$ and $\mathtt{carry}$, each guarded by $\mathtt{op}_+$ or $\mathtt{op}_-$, set $\mathtt{X}_s$ and update $\mathtt{carry}$. Before stage $0$, $\mathtt{carry}$ is at its default, which is the correct incoming carry $0$.

\emph{Bitwise operations.} Stage $s$ sets $\mathtt{X}_s=\mathtt{A}_s\odot \mathtt{B}_s$ by the neurons with conditions $\mathtt{op}_\odot$, $\mathtt{A}_s$, $\mathtt{B}_s$ whose result bit is $1$.

\emph{Shifts.} Stage $s$ sets $\mathtt{X}_s$ to bit $s-b$ of $\mathtt{A}$ for $\ll$ and to bit $s+b$ of $\mathtt{A}$ for $\gg$, where $b$ is the shift amount held by $\mathtt{B}$, by one neuron per amount $b\in[w]$ with the conditions $\mathtt{op}$, $\mathtt{B}=b$ and the source bit. If the source bit does not exist, or $b\ge w$, no neuron fires and $\mathtt{X}_s$ stays $0$, as the definition of the shifts requires.

\emph{Comparisons.} These scan from the most significant bit: stage $s$ inspects bit $w-1-s$. After stage $s$, $\mathtt{eq}$ holds that the inspected bits of $\mathtt{A}$ and $\mathtt{B}$ agree so far, and $\mathtt{lt}$ holds that they have decided $\mathtt{A}<\mathtt{B}$. The flag $\mathtt{eq}$ starts at $+1$ from the embedding. Stage $s$ has two neurons with the conditions $\mathtt{cmp}$, $\mathtt{eq}$ and differing inspected bits: if the bit of $\mathtt{A}$ is $0$ and the bit of $\mathtt{B}$ is $1$ it clears $\mathtt{eq}$ and sets $\mathtt{lt}$, otherwise it only clears $\mathtt{eq}$. After stage $w-1$, $\mathtt{eq}=\ind{\mathtt{A}=\mathtt{B}}$ and $\mathtt{lt}=\ind{\mathtt{A}<\mathtt{B}}$.

\emph{Most significant bit.} The same scan uses $\mathtt{eq}$ as the flag that no $1$ has been seen yet. Stage $s$ has one neuron with the conditions $\mathtt{op}_{\operatorname{msb}}$, $\mathtt{eq}$ and bit $w-1-s$ of $\mathtt{A}$, which clears $\mathtt{eq}$ and sets the $1$-bits of the number $w-1-s$ in $\mathtt{X}$. Afterwards $\mathtt{X}$ holds the index of the leading $1$ of $\mathtt{A}$, or $0$ if $\mathtt{A}=0$.

\emph{Multiplication.} Let $a$ and $b$ be the numbers held by $\mathtt{A}$ and $\mathtt{B}$, which do not change, and let $b_s$ be bit $s$ of $b$. Before stage $s$, the workspace satisfies
\begin{equation}\label{eq:mul-invariant}
 \mathtt{X}+\mathtt{Y}=\left\lfloor\frac{a\,(b\bmod2^s)}{2^s}\right\rfloor
 \qquad\text{and}\qquad
 \mathtt{Q}=ab\bmod2^s,
\end{equation}
where $\mathtt{X}$, $\mathtt{Y}$ and $\mathtt{Q}$ stand for the numbers held. This holds before stage $0$ since all three words are $0$. Stage $s$ applies one row of full adders: for every bit $j$ let $u_j$ and $c_j$ be the sum and carry bit of $\mathtt{X}_j+\mathtt{Y}_j+\mathtt{A}_jb_s$, and set $\mathtt{X}_j:=u_{j+1}$ with $u_w:=0$, $\mathtt{Y}_j:=c_j$ and $\mathtt{Q}_s:=u_0$. Every new bit depends on at most five old ones, including the old bit it replaces, so the row costs $\O(w)$ neurons. Writing $U$ and $C$ for the numbers with bits $u_j$ and $c_j$, the full adders give $U+2C=\mathtt{X}+\mathtt{Y}+ab_s$, and the new workspace holds $\mathtt{X}+\mathtt{Y}=(U-u_0)/2+C=\lfloor(\mathtt{X}+\mathtt{Y}+ab_s)/2\rfloor$. Now $a(b\bmod2^{s+1})=2^s(\mathtt{X}+\mathtt{Y}+ab_s)+\varrho$ for some $0\le\varrho<2^s$ by \eqref{eq:mul-invariant}, so this equals $\lfloor a(b\bmod2^{s+1})/2^{s+1}\rfloor$, and $u_0=(\mathtt{X}+\mathtt{Y}+ab_s)\bmod2$ is bit $s$ of $ab$, because the higher bits of $b$ contribute multiples of $2^{s+1}$ to the product. No carry is lost, since $\mathtt{X}+\mathtt{Y}\le a<2^w$ throughout. After stage $w-1$, $\mathtt{Q}$ holds $ab\bmod2^w$.

\emph{Jump targets and output.} The comparison scan is reused twice. At a \texttt{step} token with $\mathtt{jnz}$ it compares $\mathtt{B}$ with the constant $|P|$, one neuron per stage with the conditions $\mathtt{jnz}$, $\mathtt{eq}$ and the inspected bit of $\mathtt{B}$ differing from that of $|P|$, so that afterwards $\mathtt{lt}=\ind{\mathtt{B}<|P|}$. If $|P|=2^w$, every target is in range and this scan is omitted. At output control positions it compares $\mathtt{X}$, which holds the address $a$ of the cell to emit, with $\mathtt{B}=m$ under $\mathtt{outsep}$, and $\mathtt{A}=m$ with the constant $0$ under $\mathtt{outstart}$, so that afterwards $\mathtt{eq}=\ind{a=m}$ tells whether this is the last cell.

\emph{Finalization.} Layer $2w+4$ brings every result into $\mathtt{X}$. Under $\mathtt{op}_\times$ it copies $\mathtt{Q}$ into $\mathtt{X}$ and clears $\mathtt{Y}$, and for a comparison it sets bit $0$ of $\mathtt{X}$ under the conditions $\mathtt{lt}$ for $<$, $\mathtt{lt}$ or $\mathtt{eq}$ for $\le$, $\mathtt{eq}$ for $=$ and $\neg\mathtt{eq}$ for $\neq$, where the two conditions for $\le$ exclude each other. After this layer, the \texttt{step} token of every instruction of the form $R_i\leftarrow R_j\odot R_{j'}$ or $R_i\leftarrow\operatorname{msb}(R_j)$ holds the result of the operation in $\mathtt{X}$, a load holds the loaded value in $\mathtt{X}$, and every other \texttt{step} token has $\mathtt{X}=0$.

\paragraph{Assembly (layer $2w+5$).}
This module turns the results into what the step emits: the value in $\mathtt{X}$, the address in $\mathtt{Y}$, the next program counter in $\mathtt{Z}$, and the flags $\mathtt{write}$, $\mathtt{wreg}$ and $\mathtt{stop}$. For a store, the MLP copies $\mathtt{A}$ into $\mathtt{Y}$ and $\mathtt{B}$ into $\mathtt{X}$, which moves the address and the value into place. One neuron per instruction $g$ other than \texttt{halt}, firing on $\mathtt{tok}_{\texttt{step}}\wedge\mathtt{win}=g$, and for a conditional jump additionally on $\mathtt{iszero}$, writes the destination register $i$ into $\mathtt{Y}$ for the register-writing forms, writes $c$ into $\mathtt{X}$ for $R_i\leftarrow c$, and writes $g+1$ into $\mathtt{Z}$ if $g+1<|P|$ and sets $\mathtt{stop}$ otherwise. Here $\mathtt{Y}$ and $\mathtt{Z}$ are at their defaults, and so is $\mathtt{X}$ for $R_i\leftarrow c$. For a taken jump, that is under $\mathtt{jnz}\wedge\neg\mathtt{iszero}$, $w$ neurons write the target $\mathtt{B}$ into $\mathtt{Z}$ by setting $\mathtt{Z}_j$ on $\mathtt{B}_j$, and, if $|P|<2^w$, a neuron with the further condition $\neg\mathtt{lt}$ sets $\mathtt{stop}$ if the target is out of range. For \texttt{halt}, $\mathtt{stop}$ was set by the decoder. Comparing with \Cref{def:ramexec}, the \texttt{step} token of step $k$ now holds the write of the step, if any, as $(\mathtt{wreg},\mathtt{Y},\mathtt{X})$ with $\mathtt{write}$ set, the next program counter $\pc_k$ in $\mathtt{Z}$ if $k<t$, and $\mathtt{stop}=\ind{k=t}$, since the machine halts exactly when the program counter becomes $\bot$.

At output control positions, the MLP copies $\mathtt{X}$ into $\mathtt{Y}$ under $\mathtt{outsep}$, which moves the address $a$ of the cell to emit into $\mathtt{Y}$, where it is already $0$ under $\mathtt{outstart}$, and copies $\mathtt{A}$ into $\mathtt{X}$ under both, which moves its content $\mu_t(a)$ into $\mathtt{X}$. Both set $\mathtt{output}$, and set $\mathtt{last}$ if $\mathtt{eq}$ holds, that is if $a=m$.

\paragraph{Emission (layer $2w+6$).}
This module predicts the next token. Three heads, at positions with $\neg\mathtt{ctrl}$, copy $\mathtt{X}$ together with the five flags $\mathtt{write}$, $\mathtt{wreg}$, $\mathtt{stop}$, $\mathtt{output}$, $\mathtt{last}$, the word $\mathtt{Y}$, and the word $\mathtt{Z}$ from the latest position with $\mathtt{ctrl}$. The destinations are at their defaults there, as noted above. After the heads, every position holds in these fields the values of the latest control position at or before it, or the defaults if there is none, and we call this its \emph{broadcast state}. The MLP then writes $\mathtt{next}$ by the rules of \Cref{tab:emission}, one neuron per rule and, for the rules involving a bit position, per value of $j$ and per bit token, where each neuron clears $\mathtt{next}_{\texttt{0}}$ and sets $\mathtt{next}_\sigma$ for its token $\sigma$, or does nothing if $\sigma=\texttt{0}$. Rules for different tokens exclude each other since $\mathtt{tok}$ is one-hot, the rules for a bit token exclude each other since exactly one $\mathtt{dist}_j$ is set, and the alternatives within a rule exclude each other by their flags, so $\mathtt{next}$ stays one-hot.

\begin{table}[t]
\caption{The emission rules. The distance of a bit token to its marker is given by $\mathtt{dist}_j$ and the marker by $\mathtt{near}$. A ``bit $\mathtt{Y}_{j+1}$'' means the token \texttt{1} if this coordinate of the broadcast state is $+1$ and the token \texttt{0} otherwise.}
\label{tab:emission}
\centering\small
\begin{tabular}{@{}ll@{}}
\toprule
current token and broadcast state & next token\\
\midrule
\texttt{step} with $\mathtt{write}$ & \texttt{reg} if $\mathtt{wreg}$, else \texttt{mem}\\
\texttt{step} without $\mathtt{write}$, or \texttt{\#} with $\neg\mathtt{out}$ & \texttt{<out>} if $\mathtt{stop}$, else \texttt{pc}\\
\texttt{\#} with $\mathtt{out}$, or \texttt{<out>} & \texttt{mem}\\
\texttt{mem} or \texttt{reg} & bit $\mathtt{Y}_1$\\
\texttt{:} & bit $\mathtt{X}_1$\\
\texttt{pc} & bit $\mathtt{Z}_1$\\
bit at distance $j<w$ from \texttt{mem} or \texttt{reg} & bit $\mathtt{Y}_{j+1}$\\
bit at distance $w$ from \texttt{mem} or \texttt{reg} & \texttt{:}\\
bit at distance $j<w$ from \texttt{:} & bit $\mathtt{X}_{j+1}$\\
bit at distance $w$ from \texttt{:} & \texttt{<eos>} if $\mathtt{output}$ and $\mathtt{last}$, else \texttt{\#}\\
bit at distance $j<w$ from \texttt{pc} & bit $\mathtt{Z}_{j+1}$\\
bit at distance $w$ from \texttt{pc} & \texttt{step}\\
\bottomrule
\end{tabular}
\end{table}

It remains to check the rules against the transcript \eqref{eq:ram-transcript}, which proves \Cref{prop:predictor}(i). The last token of $\enc(x)$ is a bit at distance $w$ from \texttt{:} with the default broadcast state, so it predicts \texttt{\#}. This generated \texttt{\#} has $\neg\mathtt{out}$ and the default $\mathtt{stop}$, so it predicts \texttt{pc}, and the \texttt{pc} and its bits predict the bits of the default $\mathtt{Z}=0=\pc_0$ and then \texttt{step}. Now consider the \texttt{step} token of step $k$, which holds the assembled state of the step. If the step writes, it predicts $D_k$, the marker predicts the first bit of $a_k$ from $\mathtt{Y}$, the address bits predict the following bits and then \texttt{:}, the \texttt{:} and the value bits predict the bits of $v_k$ from $\mathtt{X}$ and then \texttt{\#}, all from the broadcast state of the \texttt{step} token, and this \texttt{\#} predicts \texttt{<out>} if $\mathtt{stop}$ and \texttt{pc} otherwise. If the step does not write, the \texttt{step} token itself predicts \texttt{<out>} or \texttt{pc} by the same flag. For $k<t$ the \texttt{pc} and its bits predict $\bits_w(\pc_k)$ from $\mathtt{Z}$ and then \texttt{step}, the \texttt{step} token of step $k+1$. For $k=t$ the token \texttt{<out>} predicts \texttt{mem}, and this \texttt{mem} is the control position for cell $0$, holding $\mathtt{Y}=0$, $\mathtt{X}=\mu_t(0)=m$ and $\mathtt{last}=\ind{m=0}$. It predicts the first bit of the address $0$, the following tokens predict the rest of the record of cell $0$ from its broadcast state, and its last bit predicts \texttt{<eos>} if $m=0$ and \texttt{\#} otherwise. That \texttt{\#} is the control position for cell $1$, and so on: the control position for cell $a$ holds $\mathtt{Y}=a$, $\mathtt{X}=\mu_t(a)$ and $\mathtt{last}=\ind{a=m}$, so the records of the cells $0,\dots,m$ are predicted in turn and the last one is followed by \texttt{<eos>}. This is $\enc(y)\ \texttt{<eos>}$, and every position from the last token of $\enc(x)$ on has predicted its successor.

\paragraph{Sizes, magnitudes and keys.}
The remaining parts of \Cref{prop:predictor} are read off \Cref{tab:layers,tab:fields}. There are $2w+6$ layers, $d=9w+62$ coordinates, and at most three heads per layer. The dictionary heads have $w+3$ query and key coordinates, the broadcast head of $\mathtt{X}$ with the five flags has $w+5$ value coordinates, and all other heads have fewer, so $d_h=w+5$. The decoder has one neuron per instruction and $w+2$ further neurons, every arithmetic stage has at most $21w+23$ neurons, the assembly at most $|P|+11w+5$, and every other layer uses at most $8w+16$ neurons, so $\dff\le|P|+21w+23$. The parameters are the embeddings in $\{-1,1\}$, the unembeddings in $\{0,1\}$, query and key weights in $\{-1,0,1\}$, value weights in $\{0,1,2\}$, output weights in $\{0,1\}$, hidden weights in $\{-2,\dots,2\}$, and biases of magnitude at most $\max(w+1,5)$, since no neuron reads more than $\max(w+2,6)$ coordinates, where $w+2$ is attained by the shifts and $6$ by the full adders of the multiplication.

For (ii), every residual coordinate is $\pm1$ at sublayer boundaries by the invariants above. Queries and keys have coordinates $\pm1$ and values have coordinates $0$ or $2$. Attention heads write to disjoint destination coordinates, and every nonzero update changes a default $-1$ to $+1$, so the residual stream remains binary after attention. A hidden pre-activation is a sum of at most $\max(w+2,6)$ terms $\pm1$ followed by a bias of magnitude at most $\max(w+1,5)$, so every partial sum has magnitude at most $\max(2w+3,11)$. The hidden activations are $0$ or $1$, the outgoing weights are $\pm2$, and at most one firing neuron writes a coordinate, so the MLP outputs have magnitude at most $2$. The logits are the coordinates of $\mathtt{next}$.

For (iii), consider the transcript up to the position before \texttt{<eos>}. The key of a dictionary head is $(\mathtt{one},\mathtt{commit},-\mathtt{wtype},\mathtt{addr})$. Since $\mathtt{addr}$ and $\mathtt{wtype}$ are written only at positions with $\mathtt{commit}$, all other positions share one key, and the positions with $\mathtt{commit}$ contribute one key per distinct pair of type and address of a record before \texttt{<out>}. These pairs are the memory cells $0,\dots,n$, the written registers and the stored cells, all of which are counted by $s_M(x)$ in \Cref{def:ramcomp}, so the three dictionary heads have at most $3s_M(x)+3$ keys together. The $w$ previous-position heads have one key each, the latest-marker head, the latest-\texttt{<out>} head, the two record heads and the three broadcast heads two each, and the $5w+8$ padded heads one each. In total $s_T\le3s_M(x)+3+w+14+5w+8=3s_M(x)+6w+25$.\qed

\subsection{Proof of \texorpdfstring{\Cref{thm:lemasim}}{Theorem 2}}\label{app:lemasim-proof}
Fix a LEMA transformer $T$ with vocabulary $\voc$, $L$ layers, $H$ heads per layer, dimensions $d,d_h,\dff$, precision $p$ and $N$ parameters. We describe a word-RAM program that evaluates $T$ one token at a time. We first explain the dictionaries and floating-point operations it uses, then its memory layout and execution, and finally bound its program length, register count, time and space and determine the required word size.

\subsubsection{Building blocks}

\paragraph{Trie-based dictionaries.}
The program maintains one dictionary for each attention head, mapping a binary key of length $d_h$ to its latest value vector of length $d_h$. A key coordinate $-1$ is represented by the bit $0$ and $+1$ by the bit $1$. Each value coordinate occupies one memory word, using the floating-point encoding described below.

A dictionary is stored as a binary trie of depth $d_h$. It has one fixed memory cell holding its root pointer. An internal node occupies two consecutive cells: if its address is $u$, cells $u$ and $u+1$ hold the pointers to its $0$-child and $1$-child. Nodes at depth $j$ correspond to prefixes of $j$ key bits. At depth $d_h$, a leaf is an array of $d_h$ consecutive cells holding the value of the corresponding key. No key needs to be stored at the leaf, since the path from the root already identifies it. The depth determines whether a pointer refers to an internal node or a value array. 

The pointer $0$ denotes an absent node. In particular, a zero root pointer represents an empty dictionary. A lookup starts at the root and follows the child selected by each successive key bit. If a pointer is zero, the key is absent and the lookup returns the zero vector. Otherwise, after $d_h$ child pointers it reaches the leaf and copies its value into a buffer. Thus both finding and reading a value take $\O(d_h)$ instructions.

An insertion follows the same path, keeping track of the cell containing the current pointer. Whenever this pointer is zero, the program allocates the missing node and writes its address into that cell. New internal nodes have both child pointers initialized to zero. At depth $d_h$, the program allocates a value array if there is none and writes the new value into it. If the key was already present, it simply overwrites the existing array. A leaf exists exactly when its key has been inserted, even if its stored value is the zero vector.

All dictionaries share an allocator. In our construction, the fixed workspace occupies the highest memory addresses, and dictionary storage grows from there toward lower addresses, leaving the lowest addresses for input and output. The allocator maintains a pointer $\mathtt{hp}$ to the lowest address occupied by the workspace and dictionaries, initially the start of the fixed workspace. To allocate $a$ cells, it reserves the addresses $\mathtt{hp}-a,\dots,\mathtt{hp}-1$ and decreases $\mathtt{hp}$ by $a$. Here $a=2$ for an internal node and $a=d_h$ for a leaf. The program never deletes nodes, and replacing an existing value allocates nothing.

An insertion creates at most $d_h$ internal nodes and one leaf, so it takes $\O(d_h)$ instructions, including allocation and initialization. After $s$ distinct keys have been inserted, there are at most $sd_h$ internal nodes and $s$ leaves. Besides its root-pointer cell, the dictionary therefore occupies at most
\[
 2sd_h+sd_h=3sd_h
\]
memory cells. Both traversals are iterative and need only a depth index and constantly many pointers in addition to their key and value buffers.

\paragraph{Floating-point operations.}
The word-RAM implements the rounded scalar operations of \Cref{def:float}. A float is stored in one word by packing its sign, biased exponent and fraction into $p=1+p_e+p_m$ bits; zero has the all-zero encoding. We show that rounded addition, multiplication and comparison each require only constantly many instructions of \Cref{def:wordram}, provided $w\ge C_0p$ for a universal constant $C_0$.

\emph{Unpacking and exact arithmetic.}
Shifts and masks extract the three fields. The integer significand is the stored fraction plus $2^{p_m}$ for a normal value, and just the fraction for a subnormal value. Zero operands are handled separately. For a nonzero operand, the instruction $\operatorname{msb}$ locates the leading significand bit, so shifts normalize the value to
\[
 (-1)^s a2^E,\qquad 2^{p_m}\le a<2^{p_m+1},
\]
also for subnormal inputs. The exponent $E$ can be negative; a fixed format-dependent offset represents it and the other signed exponent quantities by unsigned integers. All these integers have $\O(p)$ bits.

Multiplication forms the integer product of the two significands, adds their exponents and xors their signs. The exact product significand has at most $2p_m+2$ bits. It remains to round this integer times its power of two to the target format, as described below.

For addition, the routine orders the nonzero operands by magnitude as $(-1)^{s_a}a2^{E_a}$ and $(-1)^{s_b}b2^{E_b}$, with $E_a\ge E_b$, and sets $D=E_a-E_b$. If $D\le p_m+2$, the integer
\[
 Z=(a\ll D)+b\quad\text{or}\quad Z=(a\ll D)-b,
\]
according to whether the signs agree, gives the exact sum $(-1)^{s_a}Z2^{E_b}$. The magnitude $Z$ is nonnegative and has at most $2p_m+4$ bits. Cancellation gives $Z=0$ and returns the canonical zero. If instead $D>p_m+2$, then
\[
 b2^{E_b}<2^{E_a-2}.
\]
Every other representable value is at distance at least $2^{E_a-1}$ from the larger operand. Hence neither adding nor subtracting the smaller operand changes its rounded value, and the routine returns the larger operand directly. In particular, it never needs an alignment shift exceeding $p_m+2$.

\emph{Rounding and packing.}
Both operations reduce to rounding an exact magnitude $Z2^E$ with a separately retained sign. For $Z>0$, the routine sets $e_{\min}=1-e_{\max}$ and computes
\[
 \ell=\operatorname{msb}(Z),\qquad
 \lambda=\max(E+\ell-p_m,\ e_{\min}-p_m),\qquad
 \kappa=\lambda-E.
\]
The spacing $2^\lambda$ gives $p_m+1$ significant bits in the normal range and the fixed spacing in the subnormal range. If $\kappa\le0$, the left shift $q=Z\ll(-\kappa)$ gives the exact result in these units, with $-\kappa\le p_m$. If $\kappa>\ell+1$, the magnitude is below half a unit and rounds to zero. In the remaining case $1\le\kappa\le\ell+1$, a shift and a mask give the quotient and remainder
\[
 q=Z\gg\kappa,\qquad \delta=Z\bmod2^\kappa.
\]
The routine increments $q$ exactly when
\[
 \delta>2^{\kappa-1}\quad\text{or}\quad
 \bigl(\delta=2^{\kappa-1}\ \text{and}\ q\text{ is odd}\bigr),
\]
which implements round-to-nearest ties-to-even. It then restores the sign and packs the result, accounting for a significand carry and the normal/subnormal boundary. A zero result receives the canonical encoding. The case of large $\kappa$ is handled before forming a mask, so every integer formed in rounding has $\O(p)$ bits. The simulation only needs these routines when the floating-point operation does not overflow, as required for a defined forward pass.

\emph{Comparison, binarization and ReLU.}
For nonnegative finite values, the packed encodings increase with numerical value. Inspecting the signs and reversing the magnitude comparison for two negative operands therefore implements comparison. Binarization returns $-1$ on a negative operand and $+1$ otherwise, including zero; the corresponding key bit is $0$ or $1$. ReLU returns zero on a negative operand and leaves a nonnegative operand unchanged. These decisions use only the sign field and the zero encoding.

Every routine consists of a constant number of arithmetic operations, shifts, masks, comparisons and calls to $\operatorname{msb}$, with $\O(p)$-bit constants and intermediates. Consequently a sufficiently large universal $C_0$ makes all of them exact word-RAM computations at every $w\ge C_0p$. Their instruction sequences depend on the format but not on $w$, and they use only constantly many registers and temporary cells.

\subsubsection{The simulation}

\paragraph{Initialization and memory layout.}
The vocabulary order of \Cref{def:lemaforward} assigns the tokens the indices $0,\dots,|\voc|-1$. These indices are the input and output words of the simulating RAM. On a nonempty prompt $x=(x_1,\dots,x_n)$, its initial memory therefore holds $n$ in cell $0$ and the index of $x_i$ in cell $i$ for $1\le i\le n$, as in \Cref{def:ramcomp}. The program reads this input in place and does not overwrite it during prompt processing.

All parameters of $T$ are constants in the program, not an array in memory. The fixed workspace consists of the following blocks, reused for every token:
\begin{center}
\small
\begin{tabular}{@{}llp{8cm}@{}}
\toprule
block & cells & purpose\\
\midrule
$\mathtt{X}$ & $d$ & the current residual representation\\
$\mathtt{S}$ & $d$ & the sum of head outputs, or the MLP output\\
$\mathtt{U}$ & $d$ & one head's projected output\\
$\mathtt{Q},\mathtt{K}$ & $2d_h$ & the current head's query and key bits\\
$\mathtt{V},\mathtt{O}$ & $2d_h$ & its new value and retrieved value\\
$\mathtt{R}$ & $LH$ & one root pointer for each head's dictionary\\
scalar workspace & $c$ & input and output cursors, current token, generation phase, allocator pointer, arithmetic temporaries and other scalar working values\\
\bottomrule
\end{tabular}
\end{center}
Here $c$ is a fixed universal number of cells sufficient for the scalar routines and control logic; none of this scalar storage grows with a model dimension or sequence length. The total number of workspace cells is
\[
 S_0=3d+4d_h+LH+c.
\]
Registers hold only a constant number of operands and pointers at a time; vectors and persistent working values reside in the listed memory blocks.

The program places these blocks consecutively, in the listed order, in the $S_0$ cells with the highest addresses. Their base address is
\[
 \mathtt{base}=0-S_0\pmod{2^w}=2^w-S_0.
\]
Every buffer and root-pointer cell consequently has a fixed offset from $\mathtt{base}$ determined by $T$. The program computes these addresses from the base; it does not contain $2^w$ as an instruction constant---otherwise, $P$ would depend on $w$, which we avoid for generality. The allocator starts with $\mathtt{hp}=\mathtt{base}$ and places all trie nodes and leaves at lower addresses than the fixed workspace. Cell $0$ remains reserved for the input or output length, so no allocated node or leaf has address $0$.

By \Cref{def:ramcomp}, memory outside the input region is initially zero. The capacity bound below ensures that the workspace and dictionary storage stay disjoint from that region, so every root pointer starts at zero and every dictionary is empty. Initialization sets the base and allocator pointers, the input cursor to $1$, and the input length to $n$; the output length and output-mode flag are already zero. No trie nodes are allocated yet.

\paragraph{Evaluating one token.}
The program processes each token through the embedding and all $L$ layers, reusing the same buffers. Its matrix-vector multiplications are compiled row by row: a scalar accumulator starts at zero, and for each coordinate in increasing order the code loads the vector coordinate from memory and the matrix entry as an immediate constant, performs a rounded multiplication, and adds the rounded product to the accumulator. In other words, the update for a row of a matrix $W$ and a vector $u$ is
\[
 a\leftarrow a\oplus(W_{ij}\otimes u_j),\qquad j=1,2,\dots.
\]
The routines above implement each update in $\O(1)$ instructions. A bias, when present, is added after the dot product. The code for the matrices and biases is fixed by $T$ and is reused on every processed token.

\emph{Embedding.}
The current token is held as its vocabulary index. The program branches on this index to a block of instructions that writes the corresponding embedding row into $\mathtt{X}$. This supplies the initial $d$-dimensional representation.

\emph{Attention.}
At the start of a layer, $\mathtt{X}$ holds its input representation and the program clears $\mathtt{S}$. It processes the heads in increasing order. For head $h$ of layer $\ell$, the code for $W_Q^{\ell,h}$ and $W_K^{\ell,h}$ computes the projections of $\mathtt{X}$, binarizes each coordinate and stores the resulting bits in $\mathtt{Q}$ and $\mathtt{K}$. The code for $W_V^{\ell,h}$ similarly writes the floating-point value vector into $\mathtt{V}$.

The program next looks up $\mathtt{Q}$ in this head's dictionary, whose root pointer is in its designated cell of $\mathtt{R}$, and copies the retrieved vector into $\mathtt{O}$, or fills $\mathtt{O}$ with zeros on a miss. Only after this copy does it insert the pair $(\mathtt{K},\mathtt{V})$. This order makes the lookup strictly causal, including when query and key agree: replacing a leaf must not replace the value used for the current attention output.

The code for $W_O^{\ell,h}$ projects $\mathtt{O}$ into $\mathtt{U}$, and the program adds $\mathtt{U}$ coordinatewise to $\mathtt{S}$ using rounded addition. The head buffers can then be reused for the next head. Throughout these computations, $\mathtt{X}$ is unchanged, so every head reads the same layer input. After all heads, the program adds $\mathtt{S}$ to $\mathtt{X}$ coordinatewise. Thus each head's projection is computed separately, the projected outputs are summed in head order, and the residual addition comes last, exactly as in \Cref{def:lemaforward}.

\emph{MLP.}
The program clears $\mathtt{S}$ again and processes the hidden neurons in increasing order. For neuron $j$, it computes row $j$ of $W_1^\ell$ against $\mathtt{X}$, adds $b_j^\ell$ and applies ReLU. The resulting scalar $z_j$ is retained while the program updates all $d$ output coordinates by
\[
 \mathtt{S}_i\leftarrow\mathtt{S}_i\oplus\bigl((W_2^\ell)_{ij}\otimes z_j\bigr),
 \qquad i=1,\dots,d.
\]
It then discards $z_j$ and continues with neuron $j+1$. Each coordinate of $W_2^\ell z$ is therefore accumulated in the prescribed order, without keeping the $\dff$ hidden activations in memory. The input $\mathtt{X}$ stays unchanged until all neurons have been processed, when the program adds $\mathtt{S}$ to it. The same buffers now hold the input for the next layer.

\emph{Unembedding.}
After the final layer, the program computes the dot product of $\mathtt{X}$ with each unembedding row in vocabulary order. It retains only the largest logit seen so far and its token index. The first logit initializes this pair, and each later logit replaces it only on a strict improvement, so ties are resolved by the fixed vocabulary order. The winning index is the next predicted token. No vector of logits is stored.

\paragraph{Generation and output.}
The program first processes the prompt tokens from cells $1,\dots,n$ in order. It evaluates their layers and updates the dictionaries, but skips unembedding except at the final prompt token, since only that prediction starts generation. The input cursor and input length determine when the prompt has been consumed.

Afterwards, each predicted token is retained as the current token for the next pass through the embedding and layers. Earlier generated tokens are not stored as a transcript; their effects remain in the dictionaries. Before the predicted token \texttt{<out>}, generation writes nothing to the output region. When \texttt{<out>} is predicted, the program sets the output-mode flag and then processes this token like any other token to obtain the next prediction. In output mode, each predicted payload token increments the output length and is written to the corresponding cell $1,2,\dots$, as well as being retained for the next forward pass. Since the entire prompt has already been consumed, it being overwritten by the output is not a problem.

When \texttt{<eos>} is predicted, the program writes the output length to cell $0$ and halts. It does not process \texttt{<eos>} through the transformer. At this point cells $1,\dots,|y|$ contain the indices of $y$, which is exactly the output convention of \Cref{def:ramcomp}. The program stores input and output cursors but needs no counter for the total number of generated tokens.

\paragraph{Correctness.}
Let $\tau$ be the prompt followed by the tokens generated by $T$, excluding the final \texttt{<eos>}. For a position $i$, let $k_i^{\ell,h}$ and $v_i^{\ell,h}$ be the key and value of head $h$ in layer $\ell$ in the forward pass of \Cref{def:lemaforward}. Immediately before the program processes position $i$, the dictionary for this head contains exactly the keys from positions $1,\dots,i-1$. For each such key it stores $v_j^{\ell,h}$ at the largest index $j<i$ having that key.

This invariant holds initially because all dictionaries are empty. Assuming it for position $i$, induction over the layers shows that the computed residuals, queries, keys and values agree with the transformer: each lookup returns the latest strictly earlier matching value, and every scalar operation and every reduction uses the specified floating-point order. In particular, processing MLP neurons one by one interleaves independent output sums but does not reorder any of them. Inserting the current key and value then establishes the dictionary invariant for position $i+1$. Strict causality ensures that representations at earlier positions do not change when a new token is appended.

Unembedding consequently predicts the same next token, including ties. Induction over the processed positions proves that the program follows $T$'s generation and halts with the required output whenever $T$ generates $y$ from $x$, provided the words and memory are large enough for the operations just described. We verify these requirements next.

\subsubsection{Resource bounds}

\paragraph{Program length and registers.}
The program contains the matrix entries and biases as immediate constants. The code for a dot-product term or a scalar update has constant length, including the floating-point routines, and every parameter, including those of the unembedding matrix, occurs in only constantly many such blocks. The embedding dispatch uses $\O(|\voc|d)$ instructions. The remaining code for dictionary traversals, buffer operations and control has length $\O(LH(d+d_h)+L\dff+|\voc|)$, which is also $\O(N)$. Thus
\[
 |P|=\O(N).
\]
The loops over model dimensions can be unrolled, leaving a single outer loop over tokens and the input and output control logic. The scalar routines, memory accesses and pointer updates use a constant number of registers, while the buffers and saved scalar state reside in memory. Reusing the same register pool throughout gives $r=\O(1)$; in particular, neither a vector dimension nor the number of layers contributes to $r$.

\paragraph{Time.}
For one token, embedding dispatch takes $\O(|\voc|+d)$ instructions. Each head takes $\O(dd_h)$ instructions for its three projections and output projection, $\O(d_h)$ for dictionary lookup and insertion, and $\O(d)$ to add its projected output to the head sum. Each MLP takes $\O(d\dff)$ instructions for its two matrices and $\O(\dff+d)$ for bias additions, ReLU and the residual addition. Unembedding and selecting its maximum take $\O(|\voc|d)$ instructions. Buffer clearing and input/output bookkeeping fit in these bounds.

The parameter count is
\[
 N=2|\voc|d+L\bigl(4Hdd_h+2d\dff+\dff\bigr).
\]
Since all dimensions are positive, the total cost per processed token is therefore $\O(N)$. Initialization takes $\O(1)$ instructions. If $T$ generates $y$ from $x$ in $t=t_T(x)$ steps, the program processes the $n=|x|$ prompt tokens and the $t-1$ generated tokens before \texttt{<eos>}, so
\[
 t_M(x)=\O\bigl((|x|+t_T(x))N\bigr).
\]

\paragraph{Space and capacity.}
For each head, the number of allocated leaves is the number of distinct keys seen so far. Summed over all heads, this is at most $s_T(x)$ throughout the execution. The dictionary bound above therefore gives at most $3s_T(x)d_h$ cells for all nodes and value arrays, in addition to the $LH$ root-pointer cells already counted in $S_0$. The low-memory input and output occupy the union of cells $0,\dots,n$ and $0,\dots,|y|$, since their storage is reused. Hence the entire memory layout fits in
\[
 \max\{|x|,|y|\}+1+{\underbrace{3d+4d_h+LH+c}_{S_0}}+3s_T(x)d_h
\]
cells. This counts all buffers, roots, scalar working cells, trie nodes, stored values and the input/output length cell; the parameters remain in the program.

The prompt is nonempty, so every head inserts at least one key and $s_T(x)\ge LH\ge1$. Also $d,d_h\ge1$. The number of cells above is thus at most
\[
 \max\{|x|,|y|\}+(c+4)d+8s_T(x)d_h.
\]
Choosing the universal constant $C\ge\max\{c+4,8\}$ in the theorem makes its capacity condition sufficient for this allocation to fit below $2^w$. Since the fixed workspace and heap occupy one region growing downward from the top of memory, while input and output occupy the low region, they remain disjoint at every stage. The allocator needs neither $s_T(x)$ nor $|y|$ in advance. Including the constant number of registers gives
\[
 s_M(x)=\O\bigl(|x|+|y|+d+s_T(x)d_h\bigr).
\]

\paragraph{Word size and uniformity.}
The remaining requirement is that words hold the numerical quantities and instruction constants used by the program. The floating-point routines require $w\ge C_0p$. Packed parameters have $p$ bits, and instruction indices, vocabulary indices, model dimensions and offsets within the $S_0=\O(N)$ workspace have $\O(\log N)$ bits. The additional format-dependent constants of the scalar routines have $\O(p)$ bits. Consequently a sufficiently large universal $C_1$ gives a threshold
\[
 w_0=\left\lceil C_1\bigl(p+\log_2 N\bigr)\right\rceil
       =\O(p+\log N)
\]
at which all these constants and scalar arithmetic intermediates fit, and $P$ and the fixed register count define a valid word-RAM.

Neither the instructions nor their constants depend on the eventual word size $w$. The only dependence of the memory layout on $w$ comes from computing $\mathtt{base}=0-S_0$ by modular subtraction; increasing $w$ simply moves the workspace and heap to the top of the larger memory. All allocated addresses and input/output cursors fit whenever the capacity condition holds. Thus the same program $P$ and register count $r$ work for every $w\ge w_0$ and every input satisfying that condition, with the time and space bounds proved above.\qed

\subsection{Round trip}\label{app:roundtrip}
Fix a word-RAM $M=(P,r,w)$. We apply \Cref{thm:lemasim} to the transformer $T$ of \Cref{thm:ramsim}, choosing a word size that works for every halting input. This transformer has
\[
 N=\O(w^2(w+|P|)),\qquad d,d_h=\O(w),\qquad p=\O(\log w).
\]
For any input $x$ on which $M$ halts with output $y$, write $t=t_M(x)$ and $s=s_M(x)$. Then $T$ generates $\enc(y)$ from $\enc(x)$ with $t_T(\enc(x))=\O((t+|y|)w)$ and $s_T(\enc(x))=\O(s+w)$.

By \Cref{def:enc}, $|\enc(a)|=(|a|+1)(2w+3)-1$. Since $|x|,|y|<s\le r+2^w\le2^{w+1}$, the capacity expression of \Cref{thm:lemasim} satisfies, for a universal constant $A$,
\[
 \max\{|\enc(x)|,|\enc(y)|\}
   +C\bigl(d+s_T(\enc(x))d_h\bigr)
 \le A(s+w)w\le 3Aw2^w.
\]
Also, $|P|\le2^w$ gives $w_0=\O(p+\log N)=\O(w)$. Thus a sufficiently large universal constant $b$ makes $w'=\lceil bw\rceil$ satisfy both $w'\ge w_0$ and $3Aw2^w<2^{w'}$, uniformly over all halting inputs.

Applying \Cref{thm:lemasim} at this word size gives a single word-RAM $M'$ mapping $\enc(x)$ to $\enc(y)$, with tokens identified with their vocabulary indices. Substituting the bounds above gives
\begin{align*}
 t_{M'}(\enc(x))&=\O\bigl((|x|+t+|y|)(w+|P|)w^3\bigr),\\
 s_{M'}(\enc(x))&=\O\bigl((s+w)w\bigr).
\end{align*}
Since $t,s,|P|\ge1$, these imply the time and space bounds in the main text.

\begin{remark}[Executable constructions]\label{rem:validation}
The released code implements both compilers in Python by defining classes for LEMA transformers and word-RAMs reflecting \Cref{def:lematransformer,def:wordram} and implementing two functions to convert between the two, corresponding to the constructions in \Cref{thm:ramsim,thm:lemasim}. Instructions are in \texttt{verification/README.md}. For the forward construction it checks complete predicted transcripts against RAM executions, the binary-state and arithmetic invariants, and the architecture, token, magnitude, and key bounds. For the reverse construction it compares predictions, residual states, and stored dictionaries with direct transformer evaluation and checks program length, word size, and occupied space. The floating-point tests compare the integer algorithms and emitted RAM routines with exact rounding, exhaustively on small formats and on boundary cases for IEEE binary16/32/64. These implementations are mainly provided to analyze fine details of the constructions and automatically check their correctness on selected examples. Of course, they do not replace the proofs above in any way.
\end{remark}

\section{Additional material for \texorpdfstring{\Cref{sec:training}}{Section 4}}\label{app:training}

\subsection{Proof of \texorpdfstring{\Cref{lem:surrogate}}{Lemma 1}}\label{app:surrogate}

\begin{proof}[Proof of \Cref{lem:surrogate}]
Fix $i$, write $\sigma_m=\sigma\bigl(\alpha(\ip{q_i}{k_m}-d_h+1)\bigr)$ and $V=\max_{j<i}\norm{v_j}$. Since $\ip{q_i}{k_m}$ equals $d_h$ minus twice the number of disagreeing coordinates, $\sigma_m=\sigma(\alpha)$ if $k_m=q_i$ and $\sigma_m\le\sigma(-\alpha)\le e^{-\alpha}$ otherwise.

If no $j<i$ satisfies $k_j=q_i$, the LEMA output is $0$ and every weight satisfies $w_{ij}\le\sigma_j\le e^{-\alpha}$, so the difference is at most $(i-1)e^{-\alpha}V\le ne^{-\alpha}V$.

Otherwise let $\ell=\ell_i$ be the latest match, so the LEMA output is $v_\ell$ and every $m$ with $\ell<m<i$ is a mismatch. Since $\sigma(\alpha)\ge1-e^{-\alpha}$,
\[
    w_{i\ell}=\sigma(\alpha)\prod_{m=\ell+1}^{i-1}(1-\sigma_m)
    \ge(1-e^{-\alpha})^{i-\ell}
    \ge1-(i-\ell)e^{-\alpha}
    \ge1-ne^{-\alpha}.
\]
Moreover, the telescoping identity $\sum_{j<i}w_{ij}=1-\prod_{m<i}(1-\sigma_m)$ gives $\sum_{j\neq\ell}w_{ij}\le1-w_{i\ell}$. The difference is therefore
\[
    \Bigl\|(1-w_{i\ell})v_\ell-\sum_{j\neq\ell}w_{ij}v_j\Bigr\|
    \le(1-w_{i\ell})V+(1-w_{i\ell})V
    \le2ne^{-\alpha}V. \qedhere
\]
\end{proof}

\subsection{Model and training details shared by all experiments}\label{app:model-details}
All models are pre-norm decoder-only transformers with RMSNorm~\citep{zhang2019rmsnorm}, SwiGLU MLPs~\citep{shazeer2020glu} of width $8d/3$ rounded up to a multiple of $128$, untied input and output embeddings, and neither biases nor dropout. Weights are initialized from $\mathcal{N}(0,1/\mathrm{fan\text{-}in})$, embeddings from $\mathcal{N}(0,1/d)$, and the output projections of the attention and MLP blocks are additionally scaled by $1/\sqrt{2L}$ for depth $L$. Softmax transformers use RoPE of base $10^4$~\citep{su2024roformer}. LEMA and GDN models use no positional encoding, and GDN layers are those of \citet{yang2025gateddelta} in the flash-linear-attention implementation~\citep{yang2024fla}. All models are trained with AdamW~\citep{loshchilov2019decoupled} ($\beta_1=0.9$, $\beta_2=0.999$, $\epsilon=10^{-8}$), weight decay $0.1$ on all matrices including the embeddings, gradient clipping at norm $1$, a linear learning rate warmup and bf16 mixed precision. All LEMA models have the head dimension $d_h=64$ of the softmax transformers, except in \Cref{app:head-sizes}. They use $\beta=4$ in the binarization of \Cref{sec:training}, the initial value $\alpha=1/\sqrt{d_h}$, the final value $\alpha=10$ and the backward cap of $2$. The timing of the $c$ and $\alpha$ ramps is given per experiment. The surrogate runs on the Triton kernel of \citet{tan2025stickbreaking}, modified to take the threshold $c$ as a runtime logit bias, to accumulate the log-space sums in fp32 and to use lock-free atomics in the backward pass.

\subsection{Hardening schedule}\label{app:hardening}
To understand the hardening schedule better, we run some ablation studies specifically for the $77$M-parameter LEMA language models of \Cref{sec:lm}.
\paragraph{Trajectory of a run.}
First, we rerun the $77$M model to record its hardening trajectory. We plot $c$ and the $\alpha$ used in the forward and backward pass of the model together with the soft and hard loss in \Cref{fig:hardening}. Here, \emph{soft} loss refers to the loss on the probe set, $128$ held-out windows of $2048$ tokens evaluated every $200$ steps, using the current forward $\alpha$, i.e.\ using the forward pass that training uses at that step, while \emph{hard} loss refers to the loss when using the actual LEMA operation with the current model parameters. These diverge heavily in the beginning, where the model learns parameters that work well with the low-$\alpha$ soft forward pass but not when substituting in the LEMA operation. Then, as $\alpha$ grows, the stick-breaking surrogate approaches LEMA (\Cref{lem:surrogate}) and thus the soft and hard losses become nearly identical late in training.

\begin{figure}[htbp]
\centering
\includegraphics[width=\linewidth]{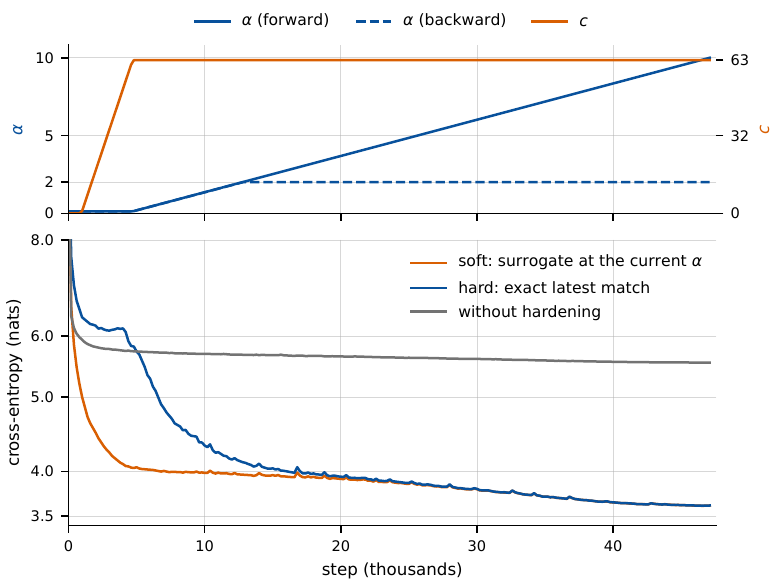}
\caption{Hardening trajectory of a $77$M LEMA model. \emph{Top:} the gate threshold $c$ and the forward and backward $\alpha$ over training. \emph{Bottom:} cross-entropy with the surrogate at the current forward $\alpha$ (soft) and with the exact latest-match operation (hard), and the cross-entropy of the run without hardening.}
\label{fig:hardening}
\end{figure}

\paragraph{Training without hardening.}
To understand why we use a hardening schedule instead of using LEMA in the forward pass throughout training, we train the same model in that exact way, with the exact latest-match operation in the forward pass from step $0$ and the surrogate at $\alpha = 2$ in the backward pass. No query ever finds a matching key, every attention output stays zero and the model trains without attention (\Cref{fig:hardening}, bottom), ending $1.95$ nats above the hardened run in validation cross-entropy ($5.497$ against $3.542$).

\paragraph{Uncapped backward $\alpha$.}
Finally, to understand why we cap $\alpha$ in the backward pass, we train the model again without capping the backward $\alpha$, i.e.\ the stick-breaking surrogate uses the same $\alpha$ in the forward and backward passes throughout training. \Cref{fig:no-cap} shows the (soft) loss and the gradient magnitudes of the query and key projections $W_Q, W_K$ of the model during training with and without the backward cap on $\alpha$. The gradients of $\sbop_{\alpha, d_h-1}$ with respect to the queries and keys for queries and keys in $\{-1,1\}^{d_h}$ go to zero as $\alpha \to \infty$, explaining why the gradients for the query and key projections die out. This hinders further adaptation of these matrices and leads to a final validation cross-entropy $0.13$ nats worse ($3.67$ against $3.542$).

\begin{figure}[htbp]
\centering
\includegraphics[width=\linewidth]{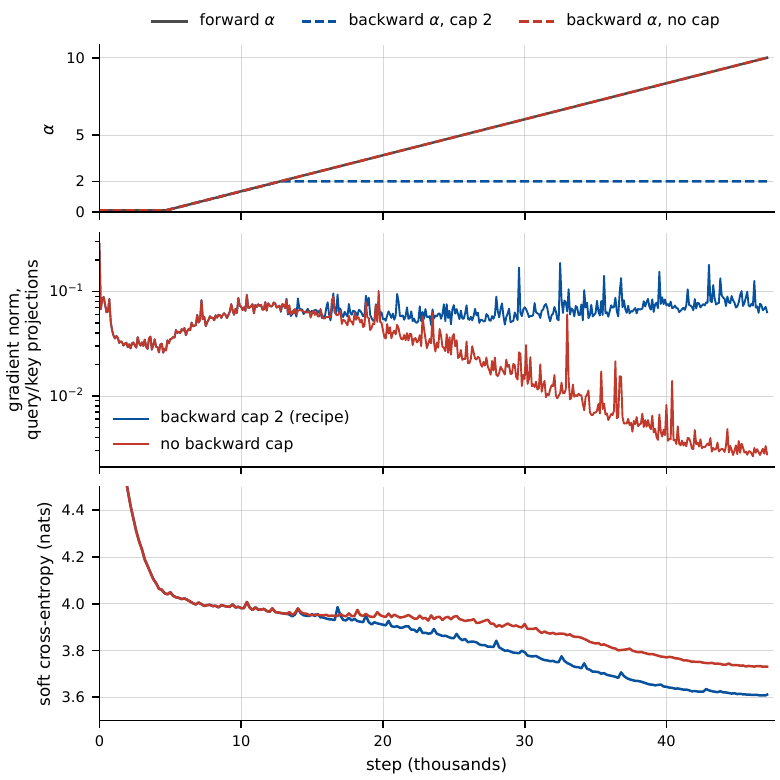}
\caption{Backward $\alpha$ capped at $2$ (the recipe) against uncapped, $77$M model. \emph{Top:} forward and backward $\alpha$. \emph{Middle:} pre-clipping gradient norm of the query and key projections every $100$ steps. \emph{Bottom:} soft cross-entropy on the probe set, which becomes nearly identical to the hard loss late in training in both runs.}
\label{fig:no-cap}
\end{figure}

\subsection{Associative recall details}\label{app:recall-details}
\paragraph{Task in prior work.}
Several prior variants of this task~\citep{arora2023zoology, arora2024based, okpekpe2025revisiting} use at most $64$ or $256$ pairs and lengthen sequences with filler tokens between the positions where recall is tested. We instead chose to use larger numbers of associations, as this actually requires a larger state to solve.
\paragraph{Training.}
Settings shared with the language modeling experiment are listed in \Cref{app:model-details}. The softmax transformers and GDN models are trained at each $n$ of the curriculum, i.e.\ $2, 4, \dots, 4096$, for at most $50\,000$ steps with a batch size of $\lfloor 2^{15}/n \rfloor$ sequences, learning rate $10^{-3}$ with a linear warmup over the first $1000$ steps and cosine decay to $0$. If the accuracy on a held-out batch, measured every $250$ steps, hits $1.000$, the stage ends and the curriculum advances to the next larger $n$ from that checkpoint with a fresh optimizer; a stage that never hits $1.000$ runs for the full $50\,000$ steps and the curriculum advances from its checkpoint with the highest held-out accuracy. Successful stages typically finish within a few thousand steps. For GDN models, increasing the head dimension increases their state size and hence memory capacity and would likely increase the $n$ up to which the models learn the task. The GDN models use the published value expansion factor of $2$, i.e.\ the state of the single GDN head is a matrix of size $64 \times 128$. Lowering the peak learning rate to $3\cdot10^{-4}$ or raising it to $3\cdot10^{-3}$ does not improve GDN's long-range recall (\Cref{fig:recall-gdn-lr}). Lastly, the curriculum is particularly important for softmax transformers to aid discovery of an algorithm. In fact, in our setup they do not learn the task within $100\,000$ steps when training directly at $n=4$, as shown in \Cref{fig:recall-discovery}. Some prior works avoid this problem e.g.\ by training on a mixture of pair counts instead of a curriculum~\citep{arora2024based}.

For LEMA models, the stick-breaking surrogate discovers a solution directly at $n=8$. We train for $50\,000$ steps with a batch size of $4096$ sequences and learning rate $3\cdot10^{-4}$ after a linear warmup over the first $1000$ steps. The threshold $c$ rises linearly from $0$ to $d_h-1=63$ between steps $1000$ and $15\,000$, and $\alpha$ rises linearly from $1/\sqrt{d_h}$ at step $15\,000$ to $10$ at the last step. With cosine decay from the same peak learning rate to zero, hardening is unstable for some seeds (\Cref{fig:recall-lr-schedule}). We therefore lower the learning rate to $5\cdot10^{-5}$ throughout the $\alpha$ ramp and evaluate the final checkpoint. At $n=4096$, final exact-match accuracies with cosine decay are $97.34\%$, $0.05\%$ and $99.33\%$, compared with $99.58\%$--$99.73\%$ with our schedule.

\paragraph{Evaluation.}
Every model is evaluated at $n \in \{4, 8, \dots, 4096\}$ on $4$ held-out batches of $\lfloor 8192/n \rfloor$ sequences, i.e.\ $32\,768$ predictions per $n$. LEMA models are evaluated with the exact latest-match operation from the final checkpoint of their single run, the baselines from the checkpoint of the stage trained at that $n$ that the curriculum advanced from.

\begin{figure}[htbp]
\centering
\includegraphics[width=0.6\linewidth]{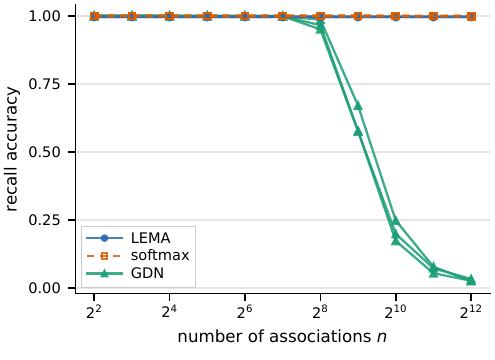}
\caption{Same as \Cref{fig:recall} (left), but showing each of the three seeds separately.}
\label{fig:recall-seeds}
\end{figure}

\begin{figure}[htbp]
\centering
\includegraphics[width=0.5\linewidth]{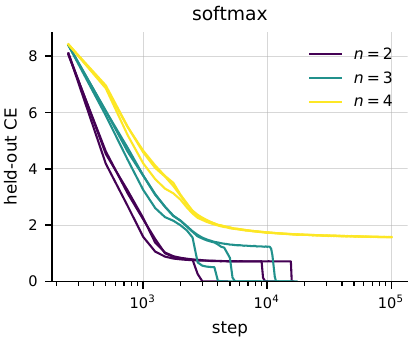}
\caption{Loss development of softmax transformers on the synthetic recall task for different numbers of pairs $n$, three seeds each. These runs use learning rate $10^{-3}$ after $1000$ warmup steps, without cosine decay. Models converge to a suboptimal solution, then discover the perfect one, while at $n=4$ this discovery does not happen.}
\label{fig:recall-discovery}
\end{figure}

\begin{figure}[htbp]
\centering
\includegraphics[width=\linewidth]{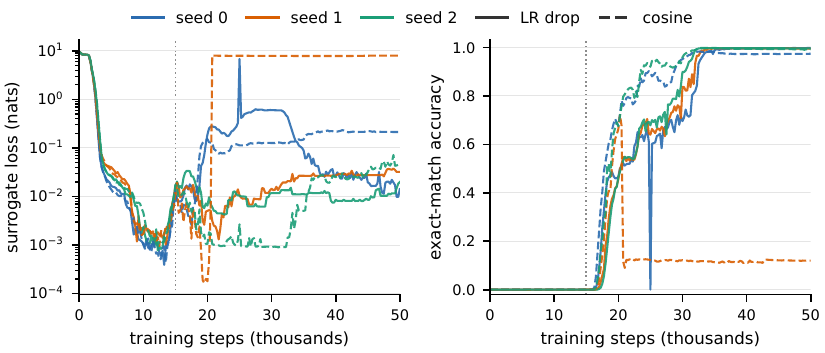}
\caption{LEMA hardening on synthetic recall at $n=8$, three seeds under each learning-rate schedule. \emph{Left:} held-out loss with the training surrogate, on a logarithmic scale. \emph{Right:} held-out accuracy with exact latest-match attention. Solid lines use the learning-rate drop to $5\cdot10^{-5}$, dashed lines cosine decay from the same peak of $3\cdot10^{-4}$. The vertical dotted line marks the start of the $\alpha$ ramp at step $15\,000$.}
\label{fig:recall-lr-schedule}
\end{figure}

\begin{figure}[htbp]
\centering
\includegraphics[width=0.6\linewidth]{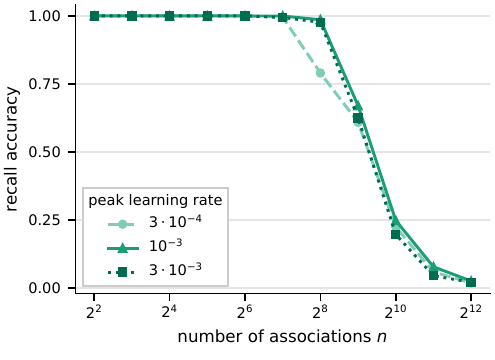}
\caption{GDN recall accuracy against the number of associations at different peak learning rates, using seed $0$.}
\label{fig:recall-gdn-lr}
\end{figure}

\subsection{Mechanistic analysis of the associative recall models}\label{app:recall-mechanism}
\paragraph{LEMA models.}
\Cref{fig:attention-seeds} (left) shows the attention graphs of all three LEMA seeds on the same sample with $n=4$. The first one is the picture of \Cref{fig:recall} (right). All three seeds share the circuit described in the main text: in the first layer, each $b_i$ attends to the $a_i$ directly before it, and in the second layer, each $a_i$ after $\mathrm{SEP}$ attends to its paired $b_i$.

To check that the same mechanism is used at large $n$, we record the codes on two samples with $n=4096$, giving $8192$ positions after $\mathrm{SEP}$. Every possible $a_i$ appears in exactly one pair per sample. In the first layer, all tokens $a_i$ before $\mathrm{SEP}$ emit the same key code within each seed, and $99.93\%$, $99.95\%$ and $99.56\%$ of the tokens $b_i$ emit exactly this code as their query code, hence attending to the preceding $a_i$. The first layer thus copies almost every $a_i$ into the residual stream at $b_i$. In the second layer, the key code emitted at $b_i$ depends almost only on its preceding $a_i$: for $99.85\%$, $99.88\%$ and $99.12\%$ of the $4096$ possible tokens $a_i$, this code agrees across the two evaluated samples. Across all pairs in both samples, the tokens $b_i$ emit $4100$, $4099$ and $4107$ distinct key codes. After $\mathrm{SEP}$, the head at $a_i$ retrieves its paired $b_i$ in $99.77\%$, $99.56\%$ and $99.44\%$ of cases, matching the accuracies of $0.9977$, $0.9957$ and $0.9946$ on these samples.

\paragraph{Softmax transformers.}
\Cref{fig:attention-seeds} (right) shows the corresponding circuit in all three softmax transformers, with attention also spread over other positions. At $n=4096$, on the same two samples as above, the tokens $b_i$ place on average $0.51$, $0.59$ and $1.00$ of their first-layer attention mass on the preceding $a_i$, for seeds $0$, $1$ and $2$, respectively. The tokens $a_i$ after $\mathrm{SEP}$ place $0.93$, $0.84$ and $0.69$ of their second-layer mass on the corresponding $b_i$. All three seeds solve the task near-perfectly.
\paragraph{Competing associations.}
We present $64$ distinct tokens $a_i$, each paired four times with different following tokens, with all pairs in random order. After $\mathrm{SEP}$, each $a_i$ appears once more. No model saw such competing associations in training. We use the $n=64$ curriculum checkpoints and the final LEMA checkpoints, evaluating $8192$ predictions per seed. \Cref{tab:repeated-keys} lists the mean probability mass on the token paired with $a_i$ at each occurrence. LEMA models put almost all their mass on the most recently paired token, as the earlier associations have already been overwritten in the dictionary. Meanwhile, softmax transformers and GDN show some recency bias, assigning the most mass to the most recently paired token on average.
\begin{table}[htbp]
\centering
\caption{Mean probability mass on the tokens paired with $a_i$ at its first through fourth occurrences, and on all other tokens, with $64$ distinct $a_i$ each appearing in four pairs. Averaged over three seeds.}
\label{tab:repeated-keys}
\begin{tabular}{lccccc}
\toprule
model & 1st & 2nd & 3rd & 4th & elsewhere \\
\midrule
LEMA & 0.000 & 0.000 & 0.002 & 0.996 & 0.002 \\
softmax & 0.005 & 0.037 & 0.171 & 0.528 & 0.258 \\
GDN & 0.093 & 0.138 & 0.159 & 0.233 & 0.377 \\
\bottomrule
\end{tabular}
\end{table}
\begin{figure}[htbp]
\centering
\includegraphics[width=\linewidth]{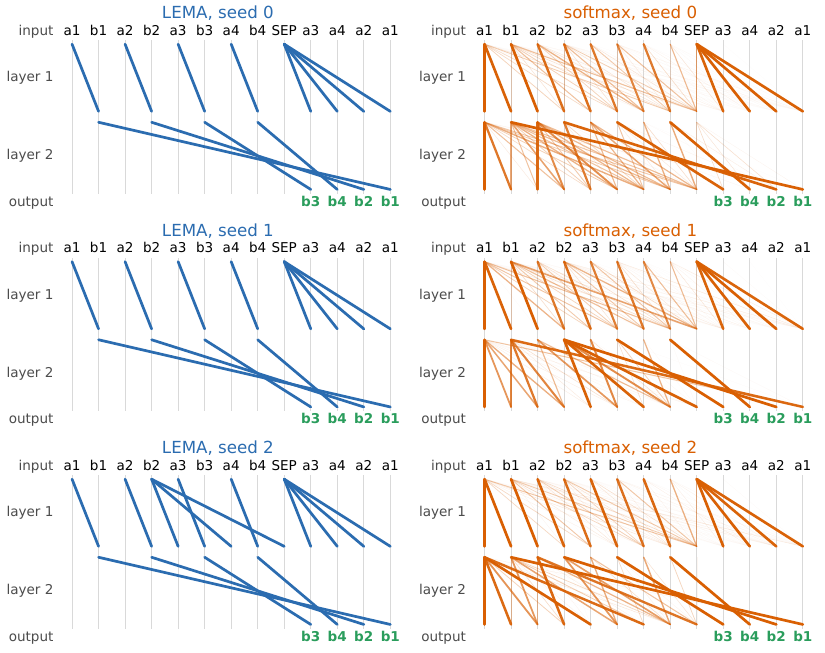}
\caption{Attention graphs of the three LEMA (left) and softmax (right) recall models on the same sample with $n=4$, one row per seed. The top row repeats \Cref{fig:recall} (right). Position $i$ attending to position $j$ is drawn as a line from $j$ at the top of the layer's band to $i$ at its bottom, with width and opacity following the attention weight. LEMA weights are $0$ or $1$. The output row of each panel shows the token predicted at each position after $\mathrm{SEP}$, green when correct.}
\label{fig:attention-seeds}
\end{figure}

\subsection{Stick-breaking attention on the recall task}\label{app:recall-sb}
For comparison with LEMA, we train plain stick-breaking models with real-valued queries and keys, $c=0$ and $\alpha=1/\sqrt{d_h}$, using the same model dimensions, $n=8$, $50\,000$ steps and initial learning rate, without the learning rate drop, for three seeds. \Cref{fig:recall-sb} (left) shows their accuracy at every $n$. The solution found at $n=8$ extrapolates to a few hundred pairs and then fails, with large differences between seeds: at $n=512$ the three seeds reach $0.06$, $0.98$ and $0.94$, and at $n=4096$ $0.00$, $0.23$ and $0.09$. A candidate explanation is that the stick is consumed by the many small gate values $\sigma(\alpha\ip{q_i}{k_j})$ of non-matching keys before the matching key is reached, an effect that grows with $n$. \Cref{fig:recall-sb} (right) tests this by zeroing every gate below a threshold $\tau$ at evaluation time. With $\tau=0.1$, mean accuracy stays between $0.93$ and $0.97$ across $n$, with all three seeds above $0.90$ at $n=4096$. Increasing the threshold to $\tau=0.3$ substantially reduces accuracy. This recovery shows that leakage through small gates limits length extrapolation. LEMA's discrete rule eliminates such leakage.
\begin{figure}[htbp]
\centering
\includegraphics[width=\linewidth]{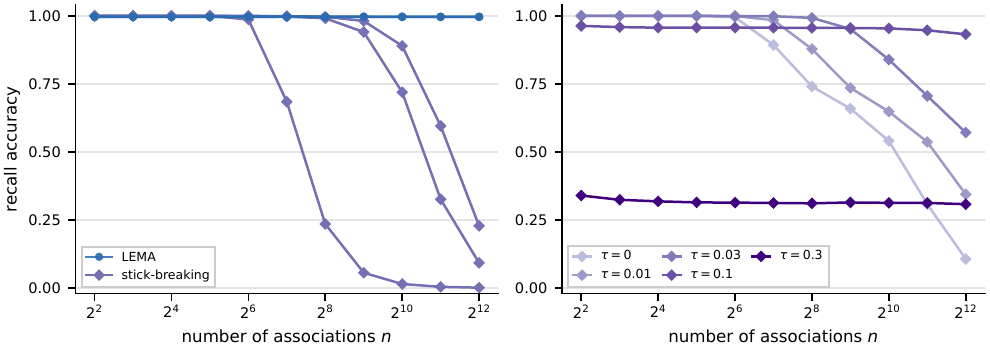}
\caption{Plain stick-breaking attention trained at $n=8$ on the recall task, three seeds. \emph{Left:} accuracy against the number of pairs, with the three LEMA seeds of \Cref{fig:recall-seeds}. \emph{Right:} the same models evaluated with every gate value below $\tau$ set to zero, averaged over the seeds.}
\label{fig:recall-sb}
\end{figure}

\subsection{Language modeling details}\label{app:lm-details}
\paragraph{Training.}
Settings shared with the recall experiment are listed in \Cref{app:model-details}. We use the \texttt{sample-100BT} subset of FineWeb-Edu~\citep{penedo2024fineweb}, tokenized with the GPT-2 tokenizer~\citep{radford2019language}. The last of its $140$ shards is held out, with its first half of documents used for validation. A training batch consists of $16$ windows of $2048$ tokens drawn uniformly at random from a subset of the training shards holding at least $31$B tokens, i.e.\ $2^{15}$ tokens per step. At each model dimension $d$ and depth $L$, all architectures receive the same training-token budget, approximately $20N$, where $N$ is the softmax/LEMA parameter count including the embedding and output matrices~\citep{hoffmann2022training}. \Cref{tab:lm-models} lists the sizes. Softmax transformers and LEMA models use $d/64$ heads of dimension $64$. GDN uses $d/128$ heads of dimension $128$ without value expansion and has slightly more parameters at the same model dimension and depth. The learning rate follows a cosine schedule from a peak of $10^{-3}\cdot1024/d$ to $1\%$ of the peak over the full run, with a linear warmup during the first $2\%$ of the steps. For LEMA models, $c$ increases linearly from $0$ to $d_h-1$ between $2\%$ and $10\%$ of the steps, and $\alpha$ increases linearly from $1/\sqrt{d_h}$ at $10\%$ of the steps to $10$ at the last step. The main results use one seed per configuration; three-seed results at $d=256$ and $d=512$ are shown in \Cref{fig:lm-recall-seeds}.
\paragraph{Evaluation.}
Validation cross-entropy is measured on the first $2^{15}$ non-overlapping windows of $2048$ tokens of the validation set, $2^{26}$ tokens in total, averaged over all predicted tokens and identical for every model. Every validation cross-entropy in this paper is measured on these $2^{26}$ tokens, which form $2^{12}$ windows at context length $16\,384$. LEMA models are evaluated with the exact latest-match operation.
\begin{table}[htbp]
\centering
\caption{Model dimension $d$, depth $L$, parameter counts and shared training-token budgets of the language models of \Cref{sec:lm}.}
\label{tab:lm-models}
\begin{tabular}{rrrrr}
\toprule
$d$ & $L$ & softmax/LEMA params & GDN params & tokens \\
\midrule
$256$  & $4$  & $29$M  & $29$M  & $0.6$B  \\
$512$  & $8$  & $77$M  & $79$M  & $1.5$B  \\
$768$  & $12$ & $162$M & $170$M & $3.2$B  \\
$1024$ & $16$ & $309$M & $326$M & $6.2$B  \\
$1280$ & $20$ & $525$M & $559$M & $10.5$B \\
$1536$ & $24$ & $834$M & $892$M & $16.7$B \\
\bottomrule
\end{tabular}
\end{table}
\paragraph{Validation losses.}
\Cref{tab:lm-ce} lists the validation cross-entropies behind \Cref{fig:lm} (left). Interpolating the softmax row log-linearly in the parameter count between neighboring sizes, the $77$M to $834$M LEMA models match softmax transformers of $44$M, $90$M, $174$M, $292$M and $456$M parameters.
\begin{table}[htbp]
\centering
\caption{Validation cross-entropy at context length $2048$ of the language models of \Cref{sec:lm}.}
\label{tab:lm-ce}
\begin{tabular}{rccc}
\toprule
$d$ & softmax & GDN & LEMA \\
\midrule
$256$ & 3.751 & 3.664 & 4.345 \\
$512$ & 3.256 & 3.200 & 3.545 \\
$768$ & 2.968 & 2.924 & 3.194 \\
$1024$ & 2.759 & 2.735 & 2.945 \\
$1280$ & 2.620 & 2.600 & 2.777 \\
$1536$ & 2.506 & 2.492 & 2.657 \\
\bottomrule
\end{tabular}
\end{table}

\paragraph{Variation across seeds.}
We repeat the experiments at $d=256$ and $d=512$ with two additional seeds, including context extension and the repeated-bigram evaluation. Validation losses vary little across seeds (\Cref{tab:lm-seeds}), and the qualitative recall patterns persist (\Cref{fig:lm-recall-seeds}).

\begin{table}[htbp]
\centering
\caption{Validation cross-entropy at context length $2048$ for three seeds at the two smallest sizes. Seed $0$ is the run of \Cref{tab:lm-ce}.}
\label{tab:lm-seeds}
\begin{tabular}{lrrrrrr}
\toprule
 & \multicolumn{3}{c}{$d=256$} & \multicolumn{3}{c}{$d=512$} \\
\cmidrule(lr){2-4} \cmidrule(lr){5-7}
 & seed $0$ & seed $1$ & seed $2$ & seed $0$ & seed $1$ & seed $2$ \\
\midrule
softmax & $3.751$ & $3.763$ & $3.761$ & $3.256$ & $3.257$ & $3.258$ \\
GDN & $3.664$ & $3.666$ & $3.686$ & $3.200$ & $3.198$ & $3.209$ \\
LEMA & $4.345$ & $4.334$ & $4.313$ & $3.545$ & $3.554$ & $3.556$ \\
\bottomrule
\end{tabular}
\end{table}
\begin{figure}[htbp]
\centering
\includegraphics[width=\linewidth]{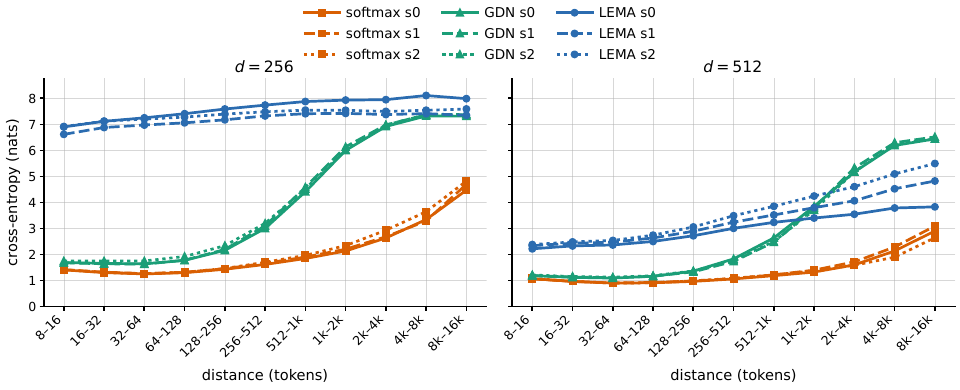}
\caption{Repeated-bigram recall for three seeds at $d=256$ and $d=512$, with softmax and GDN after context extension to $16$k and LEMA as trained.}
\label{fig:lm-recall-seeds}
\end{figure}

\subsection{Language modeling ablations}\label{app:lm-ablations}
\paragraph{Learning rate ablation.}
\Cref{tab:lm-lr} shows the validation cross-entropy of the three smallest sizes trained at half and at double the peak learning rate $10^{-3}\cdot1024/d$. Our rule is optimal or close to optimal for each architecture and size and is hence used for the larger models as well.
\begin{table}[htbp]
\centering
\caption{Validation cross-entropy at context length $2048$ with the peak learning rate of \Cref{app:lm-details} (rule), halved and doubled. The softmax run at $162$M and double learning rate diverged.}
\label{tab:lm-lr}
\begin{tabular}{lrrccc}
\toprule
model & params & $d$ & half & rule & double \\
\midrule
softmax & $29$M & $256$ & 3.780 & 3.751 & 3.739 \\
softmax & $77$M & $512$ & 3.276 & 3.256 & 3.257 \\
softmax & $162$M & $768$ & 2.973 & 2.968 & 4.325 \\
GDN & $29$M & $256$ & 3.713 & 3.664 & 3.661 \\
GDN & $79$M & $512$ & 3.227 & 3.200 & 3.198 \\
GDN & $170$M & $768$ & 2.940 & 2.924 & 2.926 \\
LEMA & $29$M & $256$ & 4.339 & 4.345 & 4.354 \\
LEMA & $77$M & $512$ & 3.551 & 3.545 & 3.597 \\
LEMA & $162$M & $768$ & 3.193 & 3.194 & 3.218 \\
\bottomrule
\end{tabular}
\end{table}

\paragraph{Stick-breaking ablation.}
For comparison, we train a $77$M model with ordinary stick-breaking attention~\citep{tan2025stickbreaking}, using real-valued queries and keys, no positional encoding and head dimension $64$, following the recipe of \Cref{app:lm-details}. Its validation cross-entropy is $3.189$ at context length $2048$ against $3.256$ for the softmax transformer, and on the recall analysis of \Cref{sec:lm} it roughly matches the context-extended softmax transformer (\Cref{fig:sb-recall}). Stick-breaking attention keeps the growing kv-cache and the quadratic cost of softmax attention, so it offers none of the inference gains of \Cref{sec:inference}.
\begin{figure}[htbp]
\centering
\includegraphics[width=0.6\linewidth]{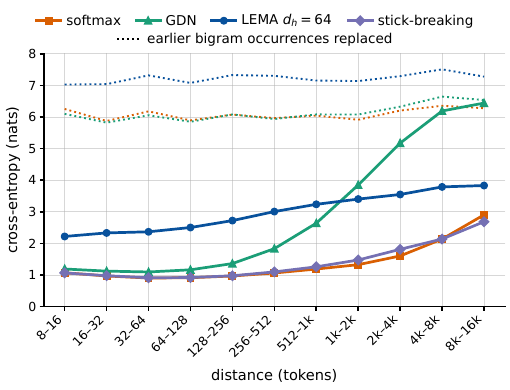}
\caption{As \Cref{fig:lm-recall}, for the models at $d=512$, including the model trained with stick-breaking attention throughout.}
\label{fig:sb-recall}
\end{figure}

\subsection{Head dimensions}\label{app:head-sizes}
Up to $309$M parameters, we also train LEMA models on FineWeb-Edu with head dimensions $d_h \in \{8, 16, 32\}$, with $d/d_h$ heads per layer. \Cref{tab:lm-ce-heads} lists their validation cross-entropies. At the small model sizes, smaller head dimensions reach a slightly lower cross-entropy, $d_h=16$ is best up to $162$M parameters and $d_h=32$ at $309$M, but the spread between head dimensions shrinks from $0.22$ nats at $29$M to $0.02$ nats at $309$M parameters.
\begin{table}[htbp]
\centering
\caption{Validation cross-entropy at context length $2048$ of the LEMA models at every head dimension, best per size in bold.}
\label{tab:lm-ce-heads}
\begin{tabular}{rrcccc}
\toprule
$d$ & params & $d_h=8$ & $d_h=16$ & $d_h=32$ & $d_h=64$ \\
\midrule
$256$ & $29$M & 4.141 & \textbf{4.126} & 4.176 & 4.345 \\
$512$ & $77$M & 3.486 & \textbf{3.481} & 3.502 & 3.545 \\
$768$ & $162$M & 3.157 & \textbf{3.143} & 3.162 & 3.194 \\
$1024$ & $309$M & 2.949 & 2.950 & \textbf{2.926} & 2.945 \\
\bottomrule
\end{tabular}
\end{table}

\Cref{tab:state-heads} shows that heads of smaller dimension hold fewer entries at every context length, recorded as described in \Cref{app:lm-analysis}.
\begin{table}[htbp]
\centering
\caption{Mean number of entries per head after $t$ tokens for the LEMA models with $309$M parameters, averaged over heads and over $32$ held-out windows.}
\label{tab:state-heads}
\begin{tabular}{lrrrrrrrr}
\toprule
$d_h$ & 2k & 4k & 8k & 16k & 32k & 64k & 128k & 256k \\
\midrule
$d_h=8$ & 34 & 39 & 45 & 51 & 56 & 61 & 66 & 71 \\
$d_h=16$ & 98 & 137 & 188 & 256 & 335 & 431 & 539 & 660 \\
$d_h=32$ & 178 & 279 & 432 & 669 & 1\,005 & 1\,514 & 2\,232 & 3\,261 \\
$d_h=64$ & 210 & 340 & 542 & 872 & 1\,358 & 2\,146 & 3\,347 & 5\,249 \\
\bottomrule
\end{tabular}
\end{table}

\Cref{fig:lm-recall-heads} shows the repeated-bigram analysis of \Cref{sec:lm} for the $309$M LEMA models at every head dimension and \Cref{fig:sniah1-heads} their accuracy on S-NIAH-1. With $d_h=8$, recall declines worst with distance, while $d_h=32$ and $64$ decline more gently than GDN. On S-NIAH-1, every head dimension is flat with the context length, at $0.14$ to $0.17$ for $d_h=8$, $0.42$ to $0.47$ for $d_h=16$, $0.67$ to $0.73$ for $d_h=32$ and $0.54$ to $0.57$ for $d_h=64$. As the loss penalty of $d_h=64$ shrinks with model size and fewer heads per layer make training and inference faster, we use $d_h=64$ for our main models, accepting its lower S-NIAH-1 accuracy at $309$M.
\begin{figure}[htbp]
\centering
\includegraphics[width=0.6\linewidth]{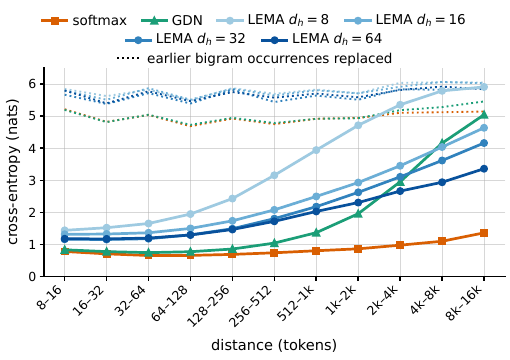}
\caption{As \Cref{fig:lm-recall}, for the models at $d=1024$ and every LEMA head dimension.}
\label{fig:lm-recall-heads}
\end{figure}

\begin{figure}[htbp]
\centering
\includegraphics[width=0.6\linewidth]{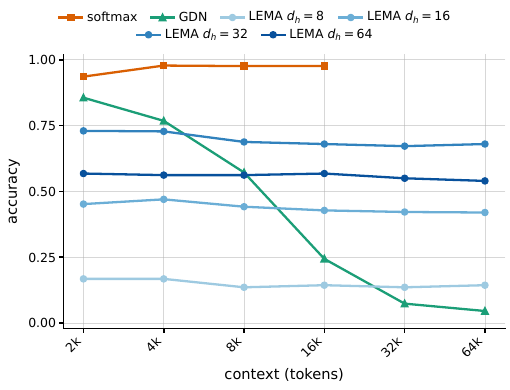}
\caption{Accuracy on S-NIAH-1 (\Cref{app:sniah1}) against the context length for the models at $d=1024$ and every LEMA head dimension.}
\label{fig:sniah1-heads}
\end{figure}

\subsection{State growth and attention distances}\label{app:lm-analysis}
\paragraph{State growth.}
For a trained LEMA model, we feed $32$ held-out windows of $2^{18}$ tokens and record for every head, after every prefix length $t$, the number of distinct key codes among its first $t$ keys, i.e.\ the number of entries the head's dictionary holds after $t$ tokens. \Cref{tab:state-sizes} lists the means over heads and windows at selected context lengths for every model size. Heads differ widely in how many distinct keys they use, and at every context length, larger models use more entries per head.
\begin{table}[htbp]
\centering
\caption{Mean number of entries per head after $t$ tokens for the LEMA models of every size, averaged over heads and over $32$ held-out windows.}
\label{tab:state-sizes}
\begin{tabular}{lrrrrrrrr}
\toprule
params & 2k & 4k & 8k & 16k & 32k & 64k & 128k & 256k \\
\midrule
$29$M & 58 & 84 & 122 & 176 & 251 & 357 & 499 & 700 \\
$77$M & 144 & 228 & 356 & 561 & 853 & 1\,306 & 1\,946 & 2\,869 \\
$162$M & 190 & 306 & 482 & 767 & 1\,181 & 1\,833 & 2\,785 & 4\,209 \\
$309$M & 210 & 340 & 542 & 872 & 1\,358 & 2\,146 & 3\,347 & 5\,249 \\
$525$M & 258 & 435 & 735 & 1\,257 & 2\,128 & 3\,672 & 6\,368 & 11\,192 \\
$834$M & 319 & 558 & 982 & 1\,748 & 3\,099 & 5\,567 & 10\,029 & 18\,189 \\
\bottomrule
\end{tabular}
\end{table}

\paragraph{Attention distance and hit rate.}
For the largest model ($834$M parameters), we record on $8$ held-out windows of $16\,384$ tokens where every head reads: the fraction of its queries that find a matching key, and for those the distance $i-\ell_i$ to the position read. \Cref{fig:distances} shows the distribution of the distance for every head. Of the $576$ heads, $56$ never find a key. These are dead and indicate that the training procedure could be improved. Most heads retrieve mostly from quite short distances with more long-range heads in the later layers. This is qualitatively somewhat similar to softmax transformers, where lower layers attend locally, attention distance tends to grow with depth and only few heads attend far back~\citep{sukhbaatar2019adaptive, vig2019analyzing, wu2024retrievalheads}.
\begin{figure}[htbp]
\centering
\includegraphics[width=\linewidth]{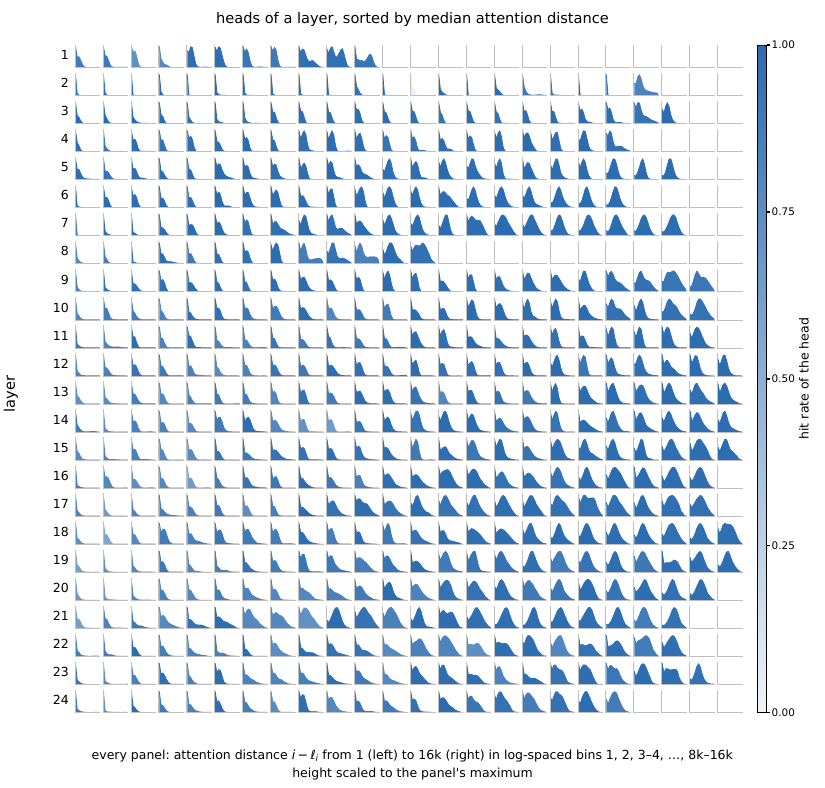}
\caption{Attention distance of every head of the largest LEMA model ($834$M parameters) on $8$ held-out windows of $16\,384$ tokens. One panel per head, layers as rows, heads of a layer sorted by their median distance from left to right. Each panel shows the distribution of the distance $i-\ell_i$ between a query and the position it attends to, over all positions with a match, in log-spaced bins from $1$ to $16\,384$ and scaled to the panel's maximum. The color encodes the hit rate of the head and fades towards white as fewer of its queries find a key. A head that never finds one is empty.}
\label{fig:distances}
\end{figure}

\subsection{Context extension}\label{app:lm-extension}
For context extension, models are retrained from their finished runs for $1$B tokens at context length $16\,384$ (results for the largest models in \Cref{tab:lm-extension}) with a fresh AdamW using the betas, weight decay and gradient clipping of pretraining, a peak learning rate of a tenth of the pretraining peak, $2\%$ linear warmup and cosine decay to the final learning rate of pretraining, within the range of \citet{chen2023extending} and \citet{xiong2023effective}. Softmax transformers raise the RoPE base from $10^4$ to $10^6$~\citep{roziere2023code}, GDN is unchanged. For LEMA, we fix $c=63$, forward $\alpha=10$ and backward $\alpha=2$. The batch size remains $2^{15}$ tokens for all models.
\begin{table}[htbp]
\centering
\caption{Validation cross-entropy of the largest models at context lengths $2048$ and $16\,384$, as trained and after context extension.}
\label{tab:lm-extension}
\begin{tabular}{lrrrr}
\toprule
 & \multicolumn{2}{c}{as trained} & \multicolumn{2}{c}{after extension} \\
\cmidrule(lr){2-3}\cmidrule(lr){4-5}
model & $2048$ & $16\,384$ & $2048$ & $16\,384$ \\
\midrule
softmax & 2.506 & 6.747 & 2.521 & 2.476 \\
GDN & 2.492 & 2.459 & 2.507 & 2.470 \\
LEMA & 2.657 & 2.635 & 2.668 & 2.642 \\
\bottomrule
\end{tabular}
\end{table}
\Cref{fig:lm-extension} shows the recall curves before and after extension. As is well known, softmax transformers with RoPE do not handle contexts far exceeding the ones seen in training without adapting positional encodings and/or retraining. GDN handles these contexts, but nonetheless its recall ability improves with the context extension. For the LEMA model, extension leaves recall at short distances unchanged and worsens it increasingly with distance, in the farthest bucket from $2.80$ to $3.19$ nats.

\begin{figure}[htbp]
\centering
\includegraphics[width=\linewidth]{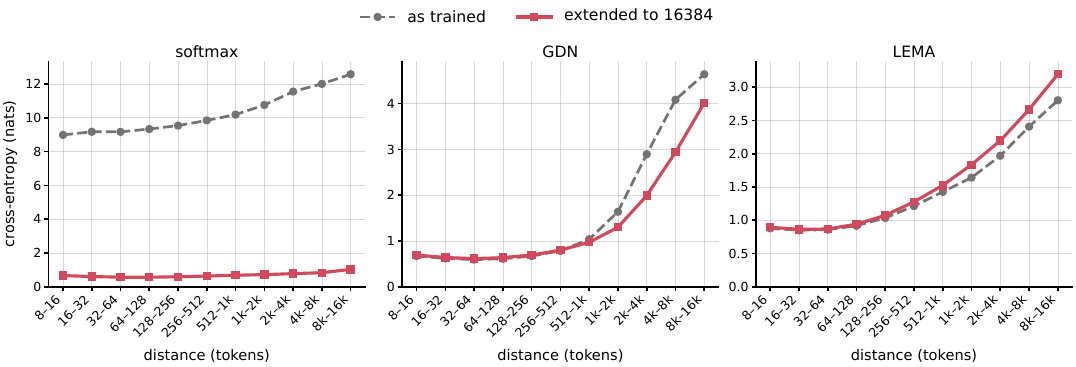}
\caption{Cross-entropy on the second token of repeated rare bigrams against the distance to the earlier occurrence for the largest models, as trained and after context extension to $16$k.}
\label{fig:lm-extension}
\end{figure}

\subsection{Repeated rare bigrams}\label{app:lm-recall}
Targets are taken from windows of $16\,384$ tokens of the validation split, within which documents are delimited by end-of-text tokens. Rarity counts use the first $10$B training tokens. A repeated occurrence of a rare bigram $xy$ is a target if $x$ does not recur between it and the previous occurrence of $xy$. We score $y$ and bucket targets by the distance between these occurrences; target counts range from $117\,588$ in the $32$ to $64$ bucket down to $740$ in the $8$k to $16$k bucket. \Cref{fig:lm-recall} shows the loss of every model size against this distance. The dotted baselines are measured on $1024$ targets sampled per bucket, all $740$ in the farthest one, each in its own copy of its window in which every earlier occurrence of the bigram is replaced by tokens drawn from the unigram distribution of the validation split, so that corruptions never interfere and every model sees the same windows. The difference to this baseline shows whether a model uses the earlier occurrence, unlike e.g.\ taking the loss at the first occurrence as baseline, which would be confounded by the available context. \Cref{fig:lm-recall-examples} (top) shows ten random targets of the farthest bucket, each one a rare phrase to be copied from context.
\begin{figure}[htbp]
\centering
\includegraphics[width=\linewidth]{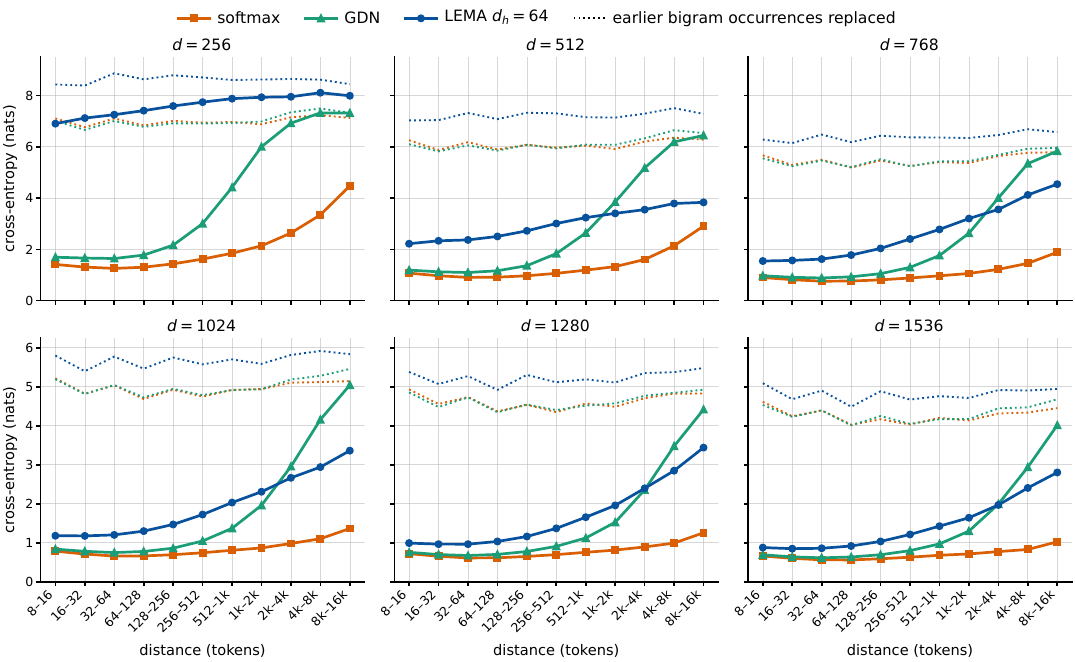}
\caption{Cross-entropy on the second token of repeated bigrams occurring rarely in training against the distance to the earlier occurrence of the bigram, one panel per model dimension, for softmax transformers and GDN after context extension to $16$k and for LEMA models as trained. Dotted lines estimate the cross-entropy of the same model after replacing all earlier occurrences of the bigram by random tokens.}
\label{fig:lm-recall}
\end{figure}

Our restriction makes $y$ the latest observed continuation of $x$, so the copying rule above suffices. With intervening competing continuations (\Cref{fig:lm-recall-examples}, bottom), that rule instead predicts a different token: even perfect retention of $xy$ leaves the model to choose among continuations. \citet{olsson2022induction} illustrate how copying previous continuations can hurt prediction in conventional transformers. Allowing such intervening occurrences, as in \citet{arora2023zoology}, adds targets that constitute $16\%$ of the broader set below $128$ tokens but $72\%$ in the farthest bucket. \Cref{fig:lm-recall-union} shows the largest models on this broader target set. All models decline more steeply with distance there, softmax transformers included. While LEMA only barely outperforms GDN in the farthest bucket in terms of the loss on this slice, it is still further below its baseline and hence uses the earlier occurrence of the bigram more to improve its prediction than GDN does ($0.9$ against $0.3$ nats difference to baseline).
\begin{figure}[htbp]
\centering
\includegraphics[width=\linewidth]{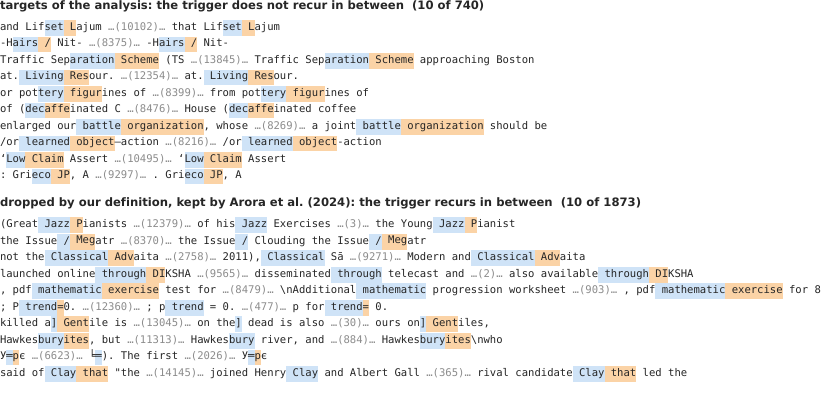}
\caption{\emph{Top:} ten random targets of the distance bucket $8$k to $16$k, with two tokens of context on either side of each occurrence and the number of tokens hidden in between. The trigger is marked blue, the token after it, which is scored at the later occurrence, orange. \emph{Bottom:} ten random targets of the definition of \citet{arora2023zoology} that ours drops, with the last occurrence of the trigger before the scored one as the middle span.}
\label{fig:lm-recall-examples}
\end{figure}
\begin{figure}[htbp]
\centering
\includegraphics[width=0.6\linewidth]{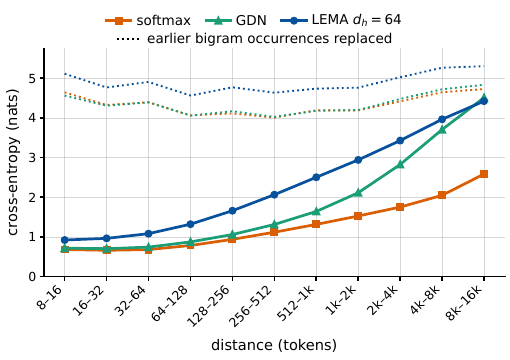}
\caption{The largest models on all targets of the definition of \citet{arora2023zoology}, as in \Cref{fig:lm-recall}.}
\label{fig:lm-recall-union}
\end{figure}

\subsection{Single-needle retrieval}\label{app:sniah1}
We use the task S-NIAH-1 of RULER~\citep{hsieh2024ruler} as released. The haystack repeats the sentence ``The grass is green. The sky is blue. The sun is yellow. Here we go. There and back again.'', a needle ``One of the special magic numbers for \{key\} is: \{number\}.'' with a random adjective-noun key and a random seven-digit number is inserted at a random depth, and the prompt ends with ``What is the special magic number for \{key\} mentioned in the provided text? The special magic number for \{key\} mentioned in the provided text is''. The model generates $128$ tokens greedily, and a prompt counts as solved if the number occurs in the generation. We use $500$ prompts per context length, the context-extended softmax transformers and GDN models (\Cref{app:lm-extension}) and the LEMA models as trained. \Cref{tab:sniah1} lists the accuracies, \Cref{fig:sniah1-depth} their dependence on the depth of the needle at $16$k.
\begin{table}[htbp]
\centering
\caption{Accuracy on S-NIAH-1 against the context length, $500$ prompts per cell. Softmax transformers are tested only up to their extended context length.}
\label{tab:sniah1}
\begin{tabular}{llrrrrrr}
\toprule
$d$ & model & $2$k & $4$k & $8$k & $16$k & $32$k & $64$k \\
\midrule
$1024$ & softmax, extended & 0.94 & 0.98 & 0.98 & 0.98 & -- & -- \\
 & GDN, extended & 0.86 & 0.77 & 0.57 & 0.24 & 0.07 & 0.05 \\
 & LEMA & 0.57 & 0.56 & 0.56 & 0.57 & 0.55 & 0.54 \\
$1536$ & softmax, extended & 0.94 & 0.97 & 0.99 & 1.00 & -- & -- \\
 & GDN, extended & 0.99 & 0.87 & 0.61 & 0.39 & 0.22 & 0.10 \\
 & LEMA & 0.93 & 0.92 & 0.91 & 0.91 & 0.90 & 0.89 \\
\bottomrule
\end{tabular}
\end{table}
\begin{figure}[htbp]
\centering
\includegraphics[width=0.6\linewidth]{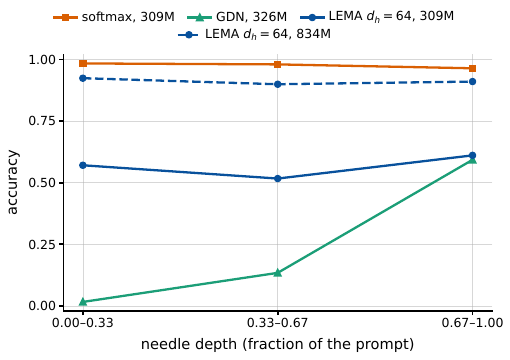}
\caption{Accuracy on S-NIAH-1 at context length $16$k against the depth of the needle.}
\label{fig:sniah1-depth}
\end{figure}

Since we use LEMA transformers without positional encoding, their hidden states at a position depend only on the tokens and on the dictionaries. For a repeated sentence, the dictionaries of an $L$-layer model therefore stop changing after at most $L$ repetitions: the first layer sees the same input in every repetition, so its dictionary is fixed after the first one and its outputs are the same from the second on, and by induction the inputs of layer $l$ are the same from repetition $l$ on, so its dictionary is fixed after repetition $l$. We feed the filler one repetition at a time and read out every dictionary after each ingestion. The dictionaries of the $16$-layer $309$M model stop changing after the $14$th repetition and those of the $24$-layer $834$M model after the $21$st. The needle then changes the dictionaries, and the filler after it reaches a fixed point again after $13$ and $21$ repetitions, with $7\,591$ and $27\,557$ entries in total. Hence the state at the question, and with it the answer, is the same for any number of filler sentences before and after the needle beyond these counts.

\subsection{Further RULER tasks}\label{app:ruler-more}
\Cref{fig:ruler-more} shows the largest models on further RULER tasks with the protocol of \Cref{app:sniah1}. S-NIAH-2 and S-NIAH-3 place a number or a UUID in the Paul Graham essays that RULER ships, the multi-key task MK-NIAH-1 adds three distractor needles to the essays, and in MK-NIAH-2 the haystack consists of key-value lines, so that the number of stored pairs grows with the context. As GDN and LEMA mostly answer with a blank instead of a number on these tasks, we also report the cross-entropy of the correct answer, the negative log-probability of its tokens given the prompt and the colon that every model emits first, summed over the tokens of the answer and averaged over prompts. On S-NIAH-2, S-NIAH-3 and MK-NIAH-1, GDN generates the answer more often than LEMA and has the lower cross-entropy up to $4$k, while LEMA has the lower cross-entropy at $16$k. On MK-NIAH-2, GDN stays ahead at every length.
\begin{figure}[htbp]
\centering
\includegraphics[width=\linewidth]{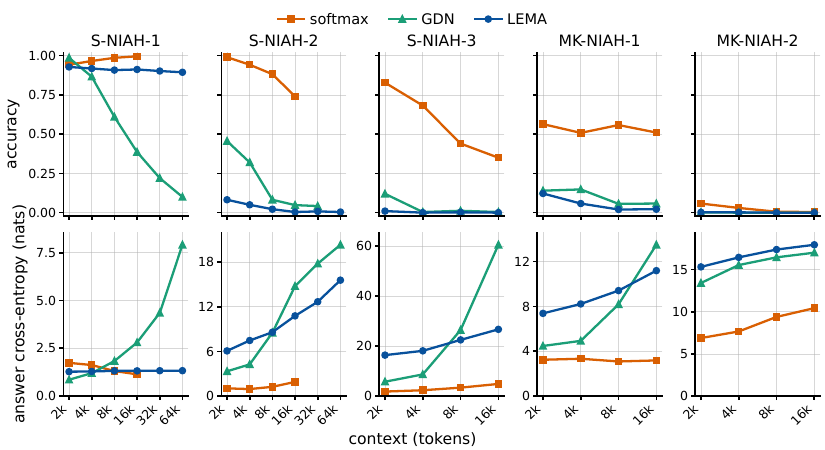}
\caption{Accuracy (top) and cross-entropy of the correct answer (bottom) of the largest models on RULER tasks against the context length, softmax transformers and GDN after context extension to $16$k.}
\label{fig:ruler-more}
\end{figure}

\subsection{Longer training and recall}\label{app:lm-tokens}
We train all three architectures at $d=512$ on five times the tokens of \Cref{app:lm-details}, $7.7$B instead of $1.5$B, with the recipe otherwise unchanged, so that every schedule stretches with the run. \Cref{tab:lm-tokens} lists the validation cross-entropies: softmax transformers and GDN improve by $0.16$ nats at context length $2048$ and LEMA by $0.09$ nats, so the gap widens. More importantly, on the recall of repeated rare bigrams (\Cref{fig:lm-tokens}), the two baselines, again after context extension, improve at every distance, while the LEMA model worsens at every distance.
\begin{table}[htbp]
\centering
\caption{Validation cross-entropy of the models at $d=512$ trained on $1.5$B and on $7.7$B tokens, at context length $2048$ as trained and at $16\,384$ after context extension for softmax transformers and GDN and as trained for LEMA.}
\label{tab:lm-tokens}
\begin{tabular}{lcccc}
\toprule
 & \multicolumn{2}{c}{$1.5$B tokens} & \multicolumn{2}{c}{$7.7$B tokens} \\
\cmidrule(lr){2-3}\cmidrule(lr){4-5}
model & $2048$ & $16\,384$ & $2048$ & $16\,384$ \\
\midrule
softmax & 3.256 & 3.170 & 3.101 & 3.060 \\
GDN & 3.200 & 3.130 & 3.043 & 3.017 \\
LEMA & 3.545 & 3.538 & 3.452 & 3.442 \\
\bottomrule
\end{tabular}
\end{table}
\begin{figure}[htbp]
\centering
\includegraphics[width=\linewidth]{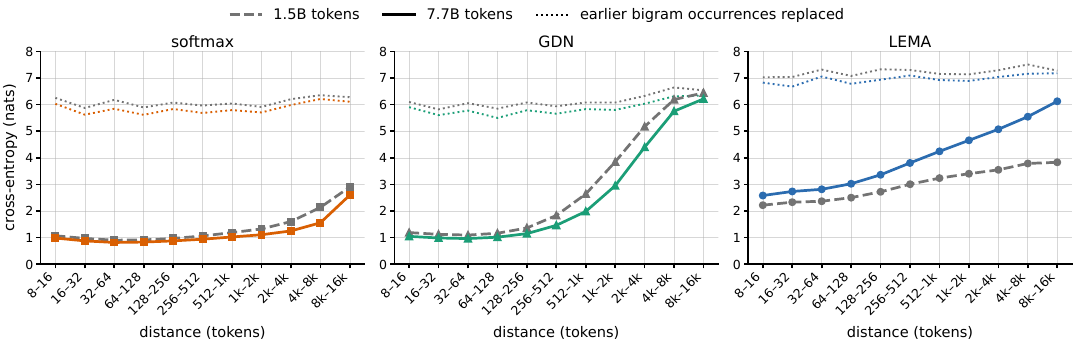}
\caption{Cross-entropy on the second token of repeated rare bigrams against the distance to the earlier occurrence for the models at $d=512$ trained on $1.5$B and on $7.7$B tokens, softmax transformers and GDN after context extension to $16$k and LEMA as trained. Dotted lines estimate the cross-entropy of the same model after replacing all earlier occurrences of the bigram by random tokens.}
\label{fig:lm-tokens}
\end{figure}

\section{Additional material for \texorpdfstring{\Cref{sec:inference}}{Section 5}}\label{app:inference}
\subsection{Softmax and GDN implementation}\label{app:inf-softmax}
The softmax transformers are run with vLLM~0.17~\citep{kwon2023efficient} on Llama models of the geometries of \Cref{tab:inference-models} with random weights, with prefix caching disabled and $95\%$ of the VRAM for weights and cache, and with the faster of its attention kernels in each scenario: FlashAttention-2~\citep{dao2024flashattention2} for prefill and for every grouped-query or batched generation, its Triton kernel for generation at batch size $1$ with multi-head attention.

An idealized bandwidth lower bound assumes one read of each required weight and cached key and value per decode step. With $W$ bytes of weights accessed per step, excluding unused input-embedding rows, and $4LHd_h$ bytes of cache per token for $L$ layers, $H$ heads and head dimension $d_h$ in bf16, the bound at context length $n$ is
\[
    \frac{W + 4LHd_h\,n}{B}
\]
for a memory bandwidth $B$, a constant plus a term linear in $n$. Offloading the kv-cache can extend the context, but incurs additional data transfers or CPU attention computation~\citep{sheng2023flexgen, lee2024infinigen}. At the peak bandwidth of the RTX~3090, vLLM is slower than this bound by a factor between $1.29$ and $1.56$ for all three models and all context lengths.

The GDN models are Qwen3-Next models~\citep{qwen2025next} run in vLLM with every layer set to GDN and dense MLPs, in the geometries of \Cref{tab:inference-models} with the GDN heads of \Cref{app:lm-details}, half as many of dimension $128$, and random weights. vLLM's bookkeeping of experts is bypassed, as the models have none. The output gate of the GDN block adds $7\%$ of parameters over the softmax model of the same geometry. The state of a sequence is $2LHd_h^2$ bytes for the recurrent part, with GDN's head count and head dimension, plus the short convolution, $10$\,MB for the $0.8$B model, so its state size does not limit context length.

\subsection{LEMA implementation}\label{app:inf-lema}
\paragraph{Codes.}
The binarized query or key of a head with $d_h \le 64$ is packed into one $64$-bit integer whose $i$-th bit is set when the $i$-th pre-activation is nonnegative, so binarization and packing are one operation on the GPU and no $\pm1$ vectors are formed. Two codes match exactly when the integers are equal. Larger head dimensions would need more than one word per code, which changes the constants below but nothing else.

\paragraph{The hash table.}
All dictionaries of a model reside in one hash table in main memory with open addressing and linear probing~\citep{knuth1998sorting}, shared by every layer, every head and every sequence of a batch. A slot holds a $64$-bit code, a $64$-bit identifier of the layer, sequence and head the entry belongs to, and the value, $16 + 2d_h$ bytes in total. The key of an entry is the pair of identifier and code, and its home slot is a hash of this pair. A lookup walks from the home slot over consecutive slots until it finds the key or an empty slot, a miss, which returns the zero vector. An insert walks the same way and either overwrites the value of the key it finds, so that the latest value wins as in \Cref{def:lema}, or claims the first empty slot. With one table for all heads, a head using many distinct codes takes more slots, the table fills at the total number of distinct codes over all heads instead of at the number of the busiest head times the number of heads, and the load of the table, the fraction of occupied slots that determines the length of the walks, is one number for the whole model. The table is allocated once with a fixed number of slots, every page of it is touched at allocation, and it is never resized. The available memory is known in advance, as for the preallocated kv-cache of the baseline.

\paragraph{Decoding.}
For each generated token and each layer, the GPU computes the codes and values of every head and copies them into pinned buffers in main memory, one call into a C++ routine looks up the query code and then inserts the key code with its value for each of the $bH$ dictionaries of the layer at batch size $b$, and the retrieved values, or zeros, are copied back to the GPU, where the rest of the layer follows. At batch size $1$, a token thus costs $LH$ lookups and $LH$ inserts, independent of the context length. The GPU work of a step is captured into $L+1$ CUDA graphs, where graph $i$ finishes layer $i-1$ and starts layer $i$, and the C++ call runs between two replays. The routine prefetches the home slots of upcoming items, so that the memory accesses of different dictionaries overlap, and processes different dictionaries in parallel during prefill and sufficiently large batches.

\paragraph{Prefill.}
A prompt is processed in chunks of $2048$ tokens, so that only the cache grows with the prompt and not the memory used for other activations. For a LEMA model, the projections and MLPs of a chunk are computed for all its positions at once on the GPU, and the C++ routine then processes the chunk's $n_c \le 2048$ positions in every dictionary in order, so that each query sees the keys of all earlier positions, of earlier chunks and of the same chunk, and not its own. No matching takes place on the GPU, the total work equals that of processing the whole prompt at once, and a decode step is the case $n_c = 1$ of the same call. 

\paragraph{Random codes.}
Random codes are close to the worst case for occupancy and locality due to practically no collisions at $d_h=64$. Generation and prefill times of softmax transformers do not depend on the weights. Those of LEMA models depend on them only through the codes, which determine which slots are accessed and how quickly the table fills. All measurements of \Cref{fig:inference} use random weights, and for the LEMA models the codes of every layer, queries and keys alike, are replaced by fresh random codes in every step. Codes then never repeat, so every lookup misses, every insert claims a new slot, the table gains $LH$ entries per token, the fastest it can fill, and no access pattern is left for the caches of the CPU to exploit. Trained models repeat codes (\Cref{tab:state-sizes}) and fill the table far more slowly. \Cref{fig:trained-vs-worst} compares the $834$M LEMA model of \Cref{sec:lm} to a model of the same geometry with random codes, in generation, where the model samples its own text, and in prefill, where it processes held-out text. After $573$k generated tokens, the trained model occupies $6\%$ of the table, whereas the random-code model occupies $95\%$, and its time per token stays at $3.1$\,ms. In prefill, the trained model processes $571$k tokens in $28$\,s and the random-code model in $40$\,s.
\begin{figure}[htbp]
\centering
\includegraphics[width=\linewidth]{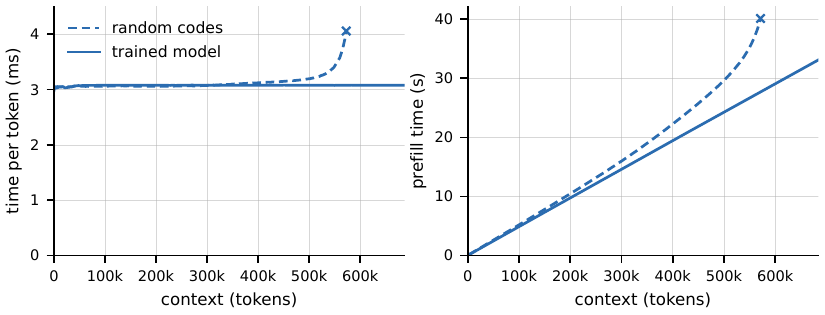}
\caption{Time per generated token (left) and prefill time (right) against context length for the $834$M LEMA model of \Cref{sec:lm}, generating its own text and processing held-out text, and for a model of the same geometry with random codes, both with the $50$\,GB hash table of \Cref{fig:inference}. Crosses mark $95\%$ table occupancy.}
\label{fig:trained-vs-worst}
\end{figure}

\subsection{Benchmark details}\label{app:inf-benchmark}
\paragraph{Hardware and software.}
All measurements are taken on one machine with an NVIDIA RTX~3090 with $24$\,GB of VRAM and a peak memory bandwidth of $936$\,GB/s, an Intel Core i9-12900K and $64$\,GB of main memory. The LEMA implementation runs on PyTorch~2.8 with CUDA~12.8, and vLLM~0.17 on PyTorch~2.10 with CUDA~13.0. Weights, the kv-cache and the values in the hash table are stored in bf16.

\paragraph{Models.}
\Cref{tab:inference-models} lists the three geometries. They follow \Cref{app:model-details} with head dimension $64$ throughout, one key and value per query head and the MLP widths of the table, that of Llama~3 for the $8$B model. Softmax models use RoPE with base $10^4$ and LEMA models no positional encoding. The $0.8$B geometry is that of the largest trained model of \Cref{sec:lm}.
\begin{table}[htbp]
\centering
\caption{The three model geometries of \Cref{fig:inference}: depth $L$, model dimension $d$, heads $H$ of dimension $64$, MLP width $\dff$, parameters and their size in bf16, and the parameters of the GDN model of the same geometry.}
\label{tab:inference-models}
\begin{tabular}{lrrrrrrr}
\toprule
model & $L$ & $d$ & $H$ & $\dff$ & params & weights & GDN params \\
\midrule
$0.8$B & $24$ & $1536$ & $24$ & $4096$ & $0.83$B & $1.7$\,GB & $0.89$B \\
$3$B & $28$ & $3072$ & $48$ & $8192$ & $3.48$B & $7.0$\,GB & $3.75$B \\
$8$B & $32$ & $4096$ & $64$ & $14\,336$ & $8.20$B & $16.4$\,GB & $8.74$B \\
\bottomrule
\end{tabular}
\end{table}

\paragraph{Provisioning.}
Both caches are allocated once before a run and are never reallocated. vLLM provisions its kv-cache from $95\%$ of the VRAM minus the weights and its workspace, which is $22$, $17$ and $7$\,GB for the three sizes. The hash table of a LEMA model takes $50$\,GB of main memory for every size, $347$ million slots of $144$ bytes.

\paragraph{Generation.}
One run gives one curve of \Cref{fig:inference} (top). Starting from the end-of-text token and an empty cache, the model generates one token at a time, sampled from its own prediction at temperature $1$. Softmax and LEMA runs stop at the VRAM limit and $95\%$ table occupancy, respectively; GDN runs stop at prescribed context lengths. The wall-clock time of every step is recorded and averaged over $400$ bins of consecutive tokens per curve. \Cref{tab:inference-ends} lists initial decode times, softmax capacity limits, and LEMA measurements at $80\%$ and $95\%$ occupancy. GDN generates at $2.9$, $11.3$ and $24.6$\,ms per token at the three sizes, constant to within $1.1$\,ms over the whole run.
\begin{table}[htbp]
\centering
\caption{Time per generated token in ms over the first bin of the LEMA curve of \Cref{fig:inference} (top), for softmax over the same tokens, at the end of the softmax curves, where the kv-cache is full, and at loads $0.8$ and $0.95$ of the hash table for LEMA, with the context length in tokens at these points.}
\label{tab:inference-ends}
\begin{tabular}{lrrrrrrrr}
\toprule
 & \multicolumn{3}{c}{softmax} & \multicolumn{5}{c}{LEMA} \\
\cmidrule(lr){2-4} \cmidrule(lr){5-9}
model & first & full & tokens & first & load $0.8$ & tokens & load $0.95$ & tokens \\
\midrule
$0.8$B & $2.66$ & $35.9$ & $150\,497$ & $3.04$ & $3.17$ & $481\,152$ & $4.06$ & $572\,675$ \\
$3$B & $10.53$ & $34.8$ & $48\,913$ & $10.61$ & $10.95$ & $206\,304$ & $12.86$ & $245\,432$ \\
$8$B & $22.82$ & $33.3$ & $13\,905$ & $22.53$ & $23.00$ & $135\,408$ & $25.86$ & $161\,064$ \\
\bottomrule
\end{tabular}
\end{table}

\paragraph{Prefill.}
One run gives one curve of \Cref{fig:inference} (bottom). For a LEMA model, a prompt of uniformly random tokens is fed into a freshly provisioned cache in chunks of $2048$ tokens, and the elapsed time after every chunk is the time to process a prompt of that length. One chunk is processed beforehand through a separate small cache, so that compilation is not part of any recorded time. For vLLM, a request with a random prompt of every multiple of $1024$ tokens is submitted to the running engine, after warm-up requests, and timed until it returns its first token. The curves end as in generation. At $2048$ tokens, prefill takes $0.07$, $0.25$ and $0.55$\,s for the softmax transformers and $0.11$, $0.32$ and $0.64$\,s for the LEMA models. GDN takes $0.07$, $0.27$ and $0.59$\,s, and its prefill stays linear at $28$k, $7.6$k and $3.5$k tokens per second. The LEMA curve falls below the softmax curve from $16$k and $14$k tokens on for the $0.8$B and $3$B models, and at $12$k tokens, where the kv-cache of the $8$B model is full, the two curves meet.

\subsection{State per token}\label{app:inf-state}
The architectures differ in how much state they keep per processed token. GDN keeps a state of fixed size. A dense softmax transformer stores the keys and values of every head at every position, $144$\,KiB per token in bf16 for our $834$M model, and a LEMA transformer stores at most one entry of $144$ bytes per head and token, $81$\,KiB per token, for which the hash table reserves $85$\,KiB at its load factor of $0.95$, while the trained $834$M model stores only about $5.6$\,KiB per token over its first $256$k tokens due to overwriting (\Cref{tab:state-sizes}). The state per token determines the slope of the time per generated token at batch size $1$ when the state resides in VRAM, as for the softmax curves of \Cref{fig:inference}, whereas LEMA's decode time remains nearly constant until its table approaches capacity. For both, it determines where the context runs out, VRAM for softmax transformers and main memory for LEMA transformers.

Modern language models therefore almost always reduce the state per token by a constant factor, for example through grouped-query attention (GQA)~\citep{shazeer2019fast, ainslie2023gqa}, where $G$ query heads share one kv-cache and thus divide the state per token by $G$, so that the kv-cache of the $0.8$B model with $G=8$ holds $1.2$M tokens on our GPU, and through hybrid architectures that replace most attention layers by fixed-state layers~\citep{minimax2025minimax01, kimi2025linear}. Both can be applied to LEMA transformers just as well: $G$ query heads can share one dictionary, and LEMA layers can take the place of the attention layers of a hybrid model. Fewer stored entries alone do not imply more efficient memory use, since overwriting can discard information needed for recall. LEMA's memory advantages lie elsewhere: the state resides in main memory instead of VRAM, and its growth adapts to the content, which we demonstrated on the synthetic task and on S-NIAH-1 (\Cref{app:sniah1}), while its benefit for language modeling remains unclear.

\subsection{Batched generation}\label{app:inf-batched}
When generating a batch of $b$ sequences at once, the weights are streamed once for all $b$ tokens, but the kv-cache of every sequence is read and the dictionaries of every sequence are updated, and the hash table is shared by the batch. \Cref{fig:batched} shows the throughput at different batch sizes of the trained $834$M LEMA model of \Cref{sec:lm} and of a softmax transformer of the same geometry but using $8$-fold GQA to extend curves further (see \Cref{app:inf-state}). The throughput of the LEMA model remains nearly constant until the table approaches capacity, at $118$k tokens for batch size $64$ and $22$k tokens for batch size $256$, while the softmax transformer slows down from the first thousand tokens on and runs out of VRAM at $19$k and $5$k tokens. GDN, run as in \Cref{app:inf-softmax}, achieves similar throughput to LEMA away from the latter's capacity limit. The decline at batch size $256$ appears to be due to vLLM overhead; GDN's recurrent state and computation per token remain fixed.
\begin{figure}[htbp]
\centering
\includegraphics[width=\linewidth]{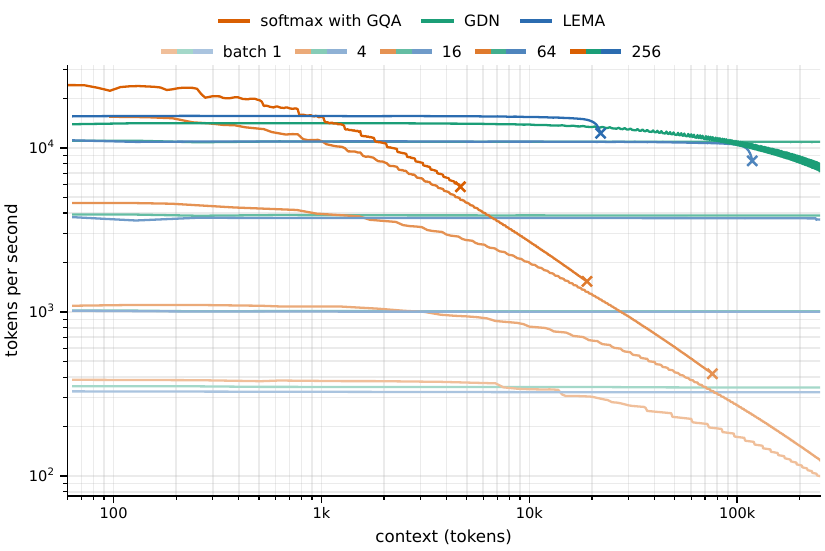}
\caption{Tokens generated per second against context length for the trained $834$M LEMA model, a softmax transformer of the same geometry with $8$-fold GQA and GDN, at batch sizes $1$ to $256$ on the machine of \Cref{app:inf-benchmark}. Crosses mark the VRAM limit for softmax and $95\%$ occupancy of LEMA's hash table; other curves are shown up to $250$k tokens.}
\label{fig:batched}
\end{figure}

\end{document}